\documentclass[11pt,letterpaper]{article}

\usepackage[margin=1in]{geometry}
\usepackage[utf8]{inputenc}
\usepackage[T1]{fontenc}
\usepackage[hyphens]{url}
\usepackage{graphicx}
\usepackage{natbib}
\usepackage{caption}
\usepackage{amsmath,amssymb,amsthm}
\usepackage{booktabs}
\usepackage{algorithm}
\usepackage{algorithmic}
\usepackage{float}
\usepackage{placeins}

\newtheorem{assumption}{Assumption}
\newtheorem{lemma}{Lemma}
\newtheorem{theorem}{Theorem}
\newtheorem{proposition}{Proposition}
\newtheorem{corollary}{Corollary}
\newtheorem{definition}{Definition}
\newtheorem{remark}{Remark}

\newcommand{\E}{\mathbb{E}}
\newcommand{\Prb}{\mathbb{P}}
\newcommand{\R}{\mathbb{R}}
\newcommand{\1}{\mathbf{1}}

\newcommand{\Gbar}{\bar{G}}
\newcommand{\Ghat}{\hat{G}}
\newcommand{\mubar}{\bar{\mu}}
\newcommand{\Stloc}{S_{\mathrm{loc}}}
\newcommand{\Stdec}{S_{\mathrm{dec}}}
\newcommand{\Stdecavg}{S_{\mathrm{dec}}^{\mathrm{avg}}}
\newcommand{\DG}{D_G}
\newcommand{\Lreg}{L}
\newcommand{\Ccert}{\mathcal{C}}
\newcommand{\Bcert}{\mathcal{B}}
\newcommand{\Hscan}{H_{\mathrm{scan}}}

\newcommand{\circnum}[1]{\textcircled{\raisebox{0.1pt}{\scriptsize #1}}}
\DeclareUnicodeCharacter{2460}{\circnum{1}}
\DeclareUnicodeCharacter{2461}{\circnum{2}}
\DeclareUnicodeCharacter{2462}{\circnum{3}}

\title{Adapting to Decision-Relevant Non-Stationarity in Decentralized Heterogeneous Bandits}
\author{
  Zhaojun Peng \\
  \texttt{Hantsukipzj@gmail.com}
}

\date{}

\begin{document}
\maketitle

\begin{abstract}
Decentralized bandit systems often contain heterogeneous agents: rewards can change at individual agents even when the best action for the network stays the same. These local changes may cancel when rewards are averaged across agents, so the number of local changes $\Stloc$ can be much larger than the number of changes in the best common arm $\Stdec$. We introduce Decision-Relevant Fresh Comparison (DRFC), which uses new, balanced samples from all agents to compare arms at the network level and switches only when fresh global evidence indicates that the common best arm has changed. We prove a high-probability dynamic regret bound with no adaptation term depending on $\Stloc$, and show that every algorithm must still pay for identifying genuine decision switches and propagating them through the communication graph. Under a distinct time-average benchmark, an anytime-valid sliding-window extension handles gradual drift; experiments on synthetic, semi-real, and MovieLens-1M replays show that DRFC ignores decision-irrelevant local changes while the extension avoids false switches.
\end{abstract}

\section{Introduction}

Many online learning systems are both distributed and heterogeneous. Examples include federated recommendation, edge decision systems, and regional pricing. Different agents face different local populations, so the same action can yield different reward distributions at different agents. At the same time, raw observations cannot be centralized, so agents must communicate over sparse, time-varying networks. In these systems, non-stationarity is often local before it is global: one population, device, or region may drift while the action preferred by the network as a whole stays unchanged.

Combining decentralized and non-stationary bandits naturally suggests detecting local shifts and refreshing estimates. This works in homogeneous systems, where a local change informs the common decision, but fails under heterogeneity: local movement need not be decision-relevant.

The common decision depends on the agent-average reward, so local changes can cancel. One population may move toward arm~1 while another moves toward arm~2, leaving the average gap and best common arm unchanged. Resetting on each local shift then spends samples and communication on changes with no decision consequence.

This separation is the point of the paper. The right adaptation unit is not a local distribution change, but a change in the network-level best action. We distinguish the number of local distribution changes, $\Stloc$, from the number of decision switches, $\Stdec$, where a decision switch is a change in the best common arm under the agent-average reward. In heterogeneous systems $\Stloc$ can be much larger than $\Stdec$, so $\Stloc$ counts environmental motion that the learner should ignore.

\paragraph{A two-agent illustration.}
Two agents and two arms already show the gap. If one agent alternates toward arm~1 while the other moves in the opposite phase, the local means change every $L$ rounds but the agent-average best arm can remain fixed. Then $\Stloc=\Theta(T/L)$ while $\Stdec=0$: local-change-reactive methods keep resetting, whereas a decision-level method should certify the common arm once and keep it.

We propose Decision-Relevant Fresh Comparison (DRFC) to implement this decision-level view. DRFC does not try to detect every local shift. It periodically performs fresh equal-agent comparisons of global arm gaps: each active arm is sampled under the same time distribution, each agent contributes equally to the aggregate estimate, and source-traceable gossip makes the comparison common knowledge over the communication graph. The algorithm switches only when a fresh global-gap certificate says that the best common arm has changed.

The resulting regret bound has no adaptation term in $\Stloc$. Its switch-dependent terms scale with $\Stdec$, plus the sampling and communication needed to certify fresh global gaps. DRFC-Probe uses cheap probes to trigger full confirmation only when the cached decision appears stale. We also prove that this dependence is not a proof artifact. Protocols that react to local changes pay $\Omega(\Stloc)$ on cancellation instances, and any rate-consistent algorithm must pay a switch-induced certification and communication-delay cost when the best common arm really changes.

\paragraph{Contributions.}
\textbf{(1) Decision-relevant non-stationarity.}
We distinguish local changes $\Stloc$ from decision switches $\Stdec$ and identify $\Stdec\ll\Stloc$ as the regime where local-change adaptation is wasteful.
\par\noindent
\textbf{(2) Fresh global-gap certification.}
We give DRFC, a fresh equal-agent comparison protocol that certifies the best common arm over the graph and attains high-probability dynamic consensus regret with no $\Stloc$ adaptation term (Theorem~\ref{thm:main}).
\textbf{(3) Switch-axis necessity.}
We prove unavoidable switch-induced certification and communication-delay costs for rate-consistent algorithms when $\Stdec>0$ (Theorem~\ref{thm:info}), and show that local-change-reactive protocols pay $\Omega(\Stloc)$ even when $\Stdec=0$ (Theorem~\ref{thm:separation}).
\textbf{(4) Drift-robust decision monitoring.}
Under a distinct time-average decision benchmark, we extend the principle to within-regime drift via DRFC-Seq, an anytime-valid sliding-window monitor that avoids decision-irrelevant false switches while adapting under a window-average margin condition (Theorem~\ref{thm:drfc-seq}).

After formalizing decision-relevant non-stationarity, we develop DRFC and DRFC-Probe and establish switch-axis upper and lower bounds together with a local-reactive separation. We then introduce DRFC-Seq for the distinct time-average drift benchmark before evaluating both settings.

\section{Related Work}

\paragraph{Non-stationary bandits.}
Passive methods forget old observations through sliding windows, discounting, or variation budgets~\citep{auer2002finite,garivier2008upper,besbes2014stochastic,cheung2019learning}. Active methods restart after detected changes~\citep{liu2018change,cao2019nearly,besson2022efficient,komiyama2024adr}, and adaptive reductions remove prior knowledge of the amount of non-stationarity~\citep{auer2019adaptively,wei2021nonstationary}. Recent work refines these ideas for smooth drift, constrained feedback, heavy tails, and path-length variation~\citep{suk2024adaptive,li2025constrained,genalti2025catoni,hu2026dynamic}. These results measure reward-model variation or best-arm switches in a centralized stream. Our setting separates local reward changes from changes in the decentralized common decision.

\paragraph{Decentralized heterogeneous bandits.}
Cooperative bandits quantify the cost of communication through network, spectral, or flooding-time parameters~\citep{landgren2016distributed,kolla2018collaborative,martinezrubio2019decentralized,xu2023decentralized,liu2025distributed,cheng2023distributed}. Heterogeneous and federated bandits study local biases, client sampling, gossip, asynchronous communication, robustness, heavy-tailed graph and reward models, and common-arm objectives~\citep{yang2022distributed,chawla2020gossiping,wang2020distributed,shi2021federated,chawla2023collaborative,xu2025heterogeneous,mirfakhar2025heterogeneous,wang2025asynchronous,hu2025robust,wang2026multiobjective,wang2025heavytailed}. These works are primarily stationary. Recent multi-agent work also studies combinatorial allocation with evolving rewards~\citep{adams2025finite}, but it does not distinguish $\Stloc$ from $\Stdec$. DRFC targets this gap: agents certify fresh equal-agent global gaps over the graph and adapt to decision switches rather than local distribution changes.

\paragraph{Sequential and distributed change detection.}
Classical quickest-detection theory characterizes optimal single-stream procedures under false-alarm constraints~\citep{moustakides1986optimal,pollak1985optimal,lai1995sequential}. Distributed variants combine local sensor tests through fusion rules or consensus~\citep{veeravalli2001decentralized,tartakovsky2008asymptotically,braca2011consensus}. Such detectors are designed to alarm when a monitored stream changes. Under heterogeneous rewards, that event can be decision-irrelevant because local changes can cancel in the agent average. We use confidence-sequence tools~\citep{howard2021timeuniform,kaufmann2021mixture,waudbysmith2024betting} for a different purpose: certifying fresh global comparisons and, later, the drift extension.

\section{Problem Formulation}

\paragraph{Reward and communication model.}
There are $N$ agents and $K$ arms. At round $t$, agent $i$ pulls arm $A_{i,t}$ and observes bounded reward $X_{i,A_{i,t},t}\in[0,1]$ with
\[
\E[X_{i,a,t}\mid\mathcal{F}_{t-1}]=\mu_{i,a,t}.
\]
Local means are heterogeneous and time-varying. Agents communicate over a random graph $G_t$; the only graph property used in the main theorem is the flooding time below. A notation table is in the supplement.

\begin{assumption}[Bounded oblivious stochastic rewards]
\label{ass:rewards}
Let $\mathcal{F}_t$ be the $\sigma$-algebra generated by past rewards, actions, graphs, and fresh randomization. The mean path $\{\mu_{i,a,t}\}$ is $\mathcal{F}_0$-measurable. Fresh randomization used by the learner is drawn independently across agents and blocks, independent of $\mathcal{F}_{t-1}$ and of the reward noise process. Conditional on $\mathcal{F}_{t-1}$ and all learner randomization, the centered noise $X_{i,a,t}-\mu_{i,a,t}\in[-1,1]$ has zero mean, and the noises observed by different agents in the same round are mutually independent.
\end{assumption}

The exogeneity and cross-agent clauses are what give equal-agent averaging its $1/\sqrt{N}$ radius in Lemma~\ref{lem:fresh-comparison}. The lower-bound families of Theorem~\ref{thm:info} use independent Bernoulli rewards and lie inside this class, so the hardness they establish applies to every algorithm covered by the upper bounds.

\begin{assumption}[Random graph flooding]
\label{ass:flooding}
There exists $\DG(\delta)$ such that, with probability at least $1-\delta$, every source-traceable record relevant to the regret analysis floods to all agents within $\DG(\delta)$ rounds. For Erd\H{o}s--R\'enyi $G_t$ with $p\geq c\log(NT/\delta)/N$, $\DG(\delta)\leq N-1$.
\end{assumption}

The equal-agent global mean is
\[
\mubar_{a,t}=\frac1N\sum_{i=1}^N\mu_{i,a,t}.
\]

\paragraph{Decision-relevant non-stationarity.}
The instantaneous best common arm is
\[
a_t^\star=\arg\max_a \mubar_{a,t}.
\]
We assume this maximizer is unique at every round. This removes artificial decision switches caused only by tie-breaking. The local-change and decision-switch counts are
\begin{equation}
\label{eq:sloc-sdec}
\begin{aligned}
\Stloc
&=\sum_{t=2}^T
  \1\{\exists i,a:\mu_{i,a,t}\ne\mu_{i,a,t-1}\},\\
\Stdec
&=\sum_{t=2}^T
  \1\{a_t^\star\ne a_{t-1}^\star\}.
\end{aligned}
\end{equation}
The regime of interest is $\Stdec\ll\Stloc$: local rewards may change often, while the best common arm changes rarely.

\paragraph{Decision regimes and regret.}
\begin{definition}[Sign-stable decision regimes]
\label{ass:regimes}
There exist switch times $1=\rho_0<\cdots<\rho_{\Stdec}<\rho_{\Stdec+1}=T+1$ with $a_t^\star=a_r^\star$ for all $t\in[\rho_r,\rho_{r+1})$. Write $\Lreg_r=\rho_{r+1}-\rho_r$ for the length of regime $r$, so $\sum_{r=0}^{\Stdec}\Lreg_r=T$. For every $a\neq a_r^\star$, the instantaneous global gap in regime $r$ is sign-stable:
\begin{gather*}
\Delta_a^{(r)}=\inf_{t\in[\rho_r,\rho_{r+1})}\bigl(\mubar_{a_r^\star,t}-\mubar_{a,t}\bigr)>0,\\
\Gamma_a^{(r)}=\sup_{t\in[\rho_r,\rho_{r+1})}\bigl(\mubar_{a_r^\star,t}-\mubar_{a,t}\bigr)\leq 1.
\end{gather*}
\end{definition}

The main benchmark is dynamic consensus regret against the instantaneous best common arm:
\begin{equation}
\label{eq:regret-def-anchor}
R_T=\sum_{t=1}^T\sum_{i=1}^N\bigl(\mu_{i,a_t^\star,t}-\mu_{i,A_{i,t},t}\bigr).
\end{equation}
Local means may change arbitrarily often inside a regime. The main model only requires the global best arm and the signs of the global gaps to remain stable between decision switches.

\paragraph{Environment assumptions.}
The stochastic reward and graph assumptions above, together with the minimum length condition below, are the only environment assumptions for Theorem~\ref{thm:main}. Scan-budget choices are algorithmic and stated with the certification guarantee.

\begin{assumption}[Minimum decision-regime length]
\label{ass:min-regime}
Let $L_{\min}=\min_r \Lreg_r$. The main guarantee requires $L_{\min}$ to exceed the monitoring, certification, and flooding time needed after a decision switch. The concrete inequality is stated with Theorem~\ref{thm:main}. This spacing is required only between decision switches, not between local changes.
\end{assumption}

\section{Fresh Decision Certification}

The basic primitive is a fresh comparison of two arms at the equal-agent level. It does not test whether any local distribution changed. Instead, it discards stale evidence and estimates the current global gap from a balanced sample in which every agent contributes symmetrically. This primitive is what lets DRFC react to $\Stdec$ rather than $\Stloc$.

\paragraph{Fresh equal-agent comparison.}
For a scan window $\mathcal{W}$, define the equal-agent time-average gap
\[
\Gbar_{a,b}(\mathcal{W})
=\frac{1}{|\mathcal{W}|}\sum_{t\in\mathcal{W}}
  \bigl(\mubar_{a,t}-\mubar_{b,t}\bigr).
\]
Each scan block pulls every active arm the same number of times at every agent, with reward-independent random order. The resulting pairwise estimator is
\[
\Ghat_{a,b}=\frac1N\sum_{i=1}^N\bigl(\hat g_{i,a}-\hat g_{i,b}\bigr),
\]
where each agent has the same number of fresh samples from arms $a$ and $b$.

Two symmetries are essential. Reward-independent random arm orders give every arm the same exposure to each slot of an oblivious mean path, so within-scan drift cannot systematically favor one arm. Equal-agent averaging assigns weight $1/N$ to each agent regardless of local sample counts, while source-keyed deduplication preserves those weights under gossip. The certificate therefore targets the fresh global gap rather than a sample-weighted or stale surrogate.

\begin{lemma}[Fresh equal-agent comparison]
\label{lem:fresh-comparison}
\normalfont
Fix a scan window $\mathcal{W}$ contained in one decision regime. On the high-probability event used in Theorem~\ref{thm:main}, simultaneously for all active pairs,
\[
\begin{aligned}
|\Ghat_{a,b}-\Gbar_{a,b}(\mathcal{W})|
&\leq \beta(\mathcal{W}),\\
\beta(\mathcal{W})
&=\widetilde O\!\left(
  \sqrt{\frac{\log(K^2T^3/\delta)}{NhB}}
  \right),
\end{aligned}
\]
where $B$ is the number of balanced blocks in $\mathcal{W}$. Hence a positive lower-confidence gap certifies a positive equal-agent gap over the same fresh window.
\end{lemma}

\begin{corollary}[Scan-budget calibration]
\label{cor:scan-calibration}
Suppose lower bounds $\underline\Delta_a\leq\Delta_a^{(r)}$ are known for the active challengers and set
\[
B_a=\left\lceil
\frac{16C^2\log(K^2T^3/\delta)}
{Nh\,\underline\Delta_a^2}
\right\rceil.
\]
A fresh active-set scan with budget
\[
\Hscan\geq C_{\mathrm{scan}}h\sum_a B_a+\DG(\delta/3)
\]
certifies the unique best common arm whenever the scan is contained in one decision regime. Constants and the active-set proof are in supplement~A.1.
\end{corollary}

\section{DRFC and Its Regret Guarantee}

\paragraph{Algorithm.}
DRFC is the conservative instantiation of the fresh certificate. It maintains a candidate $\hat{a}$, an epoch id, a scan schedule, and source-traceable records. Between scans, agents exploit $\hat{a}$. Every $M$ rounds, DRFC runs an active-set scan using only fresh balanced records, with no local-change alarm or local reset rule. After each block, arm $a$ is eliminated once
\[
\max_{b\in\mathcal{A}\setminus\{a\}}\bigl(\Ghat_{b,a}-\beta_{b,a}\bigr)>0.
\]
After certification, switch records propagate over the random graph and agents synchronize for $\DG(\delta/3)$ rounds.

\begin{algorithm}[t]
\caption{DRFC active-set scan}
\label{alg:drfc}
\begin{algorithmic}[1]
\STATE \textbf{Input:} block size $h$, monitoring period $M$, scan budget $\Hscan$, confidence $\delta$, flooding $\DG$.
\STATE Initialize $\hat{a}$ by one scan in rounds $[1,\Hscan]$, charged to $\Ccert_0$, then set $e\leftarrow 0$.
\FOR{$t=\Hscan+1$ to $T$}
\STATE Exploit $\hat{a}$ if no phase active.
\IF{$t\equiv 0\bmod M$ and idle}
\STATE Set $\mathcal{A}\leftarrow[K]$.
\WHILE{$|\mathcal{A}|>1$ and budget $\Hscan$ not exhausted}
\STATE Run synchronized balanced active-set block with per-agent independent random orders ($h$ slots per active arm).
\STATE Gossip source-keyed records in background (dedup by source, epoch, active set, arm, block), sampling continues.
\STATE On completed equal-agent records, eliminate identically at all agents each $a$ with $\max_b(\Ghat_{b,a}-\beta_{b,a})>0$.
\ENDWHILE
\IF{$\mathcal{A}=\{w\}$ and $w\neq\hat{a}$}
\STATE Emit switch record, set $\hat{a}\leftarrow w$, set $e\!\mathrel{+}\!=\!1$, sync $\DG(\delta/3)$ rounds.
\ENDIF
\ENDIF
\ENDFOR
\end{algorithmic}
\end{algorithm}

\paragraph{Certification complexity.}
For regime $r$, define
\[
\Ccert_r=\sum_{a\neq a_r^\star}\frac{\Gamma_a^{(r)}}{(\Delta_a^{(r)})^2},
\]
which is the gap-dependent cost of certifying all challengers in regime $r$ and reduces to $\sum_{a}1/\Delta_a^{(r)}$ for stationary gaps.

\begin{theorem}[DRFC dynamic consensus regret]
\label{thm:main}
\normalfont
Under Assumptions~\ref{ass:rewards}, \ref{ass:flooding}, and~\ref{ass:min-regime}, and the decision regimes of Definition~\ref{ass:regimes}, choose $\Hscan$ by Corollary~\ref{cor:scan-calibration}. If $L_{\min}\ge M+2\Hscan+\DG(\delta/3)$, then DRFC satisfies w.p.\ $\geq 1-\delta$
\begin{multline*}
R_T=\widetilde{O}\Bigl(\tfrac{NK(h+\DG(\delta/3))T}{M}+\sum_{r=0}^{\Stdec}(1+\tfrac{\Lreg_r}{M})\Ccert_r \\
{}+ N\Stdec(M+\Hscan+\DG(\delta/3))\Bigr).
\end{multline*}
\end{theorem}

\paragraph{Proof sketch.}
Lemma~\ref{lem:fresh-comparison} gives a uniform radius for every fresh balanced scan. Sign stability makes certified lower-confidence gaps preserve the true equal-agent signs, so each suboptimal arm is removed once $\beta(\mathcal{W})\le\Delta_a^{(r)}/4$, contributing $\Ccert_r$ up to logs. The other terms are periodic scans, monitoring delay, one conservative boundary scan, and graph synchronization. Full proof in supplement~A.

The three terms are scheduled certification, within-regime statistical certification, and switch-only delay/synchronization. No term scales with $\Stloc$, so local changes that cancel in the equal-agent average do not force adaptation. The probe-triggered variant below replaces periodic full scans by probe-triggered confirmation and removes the conservative boundary term $N\Stdec\Hscan$.

\paragraph{Gap adaptivity.}
A geometric scan-budget schedule removes prior knowledge of $\Delta_{\min}$ up to logarithmic factors; the full Gap-Adaptive DRFC statement and proof are in supplement~B.

\section{Probe-Triggered Monitoring and Lower Bound}

\subsection{DRFC-Probe}

DRFC-Probe is a modification of the monitoring mechanism, not a separate learning algorithm. It replaces
\[
\begin{gathered}
\text{periodic full scan}\\[-1mm]
\Downarrow\\[-1mm]
\text{cheap probe $+$ triggered confirming scan}
\end{gathered}
\]
Every $M$ rounds, a balanced all-arm probe tests whether the cached candidate may be stale; only a suspicious probe launches the same active-set scan as Algorithm~\ref{alg:drfc}. Let $d_{\mathrm{probe}}$ be the worst-case delay from a decision boundary to completion of the first suspicious probe (for the fixed schedule in the supplement, $d_{\mathrm{probe}}\leq M+q_pKh_{\mathrm{probe}}$). To keep the confirming scan inside the new regime, we require explicitly
\[
\Lreg_r\geq d_{\mathrm{probe}}+\Hscan+\DG(\delta/3).
\]

\begin{corollary}[Probe-triggered monitoring]
\label{cor:probe-main}
Under the assumptions of Theorem~\ref{thm:main}, suppose $(2+\lambda_p)\beta_p<\Delta_{\min}$, so a correct cached candidate causes no false trigger, and the regime-length condition above holds. Then, w.p.\ $\geq1-\delta$, every confirming scan is contained in one decision regime and
\[
\begin{aligned}
R_T=\widetilde O\!\Bigl(&
\frac{NKq_ph_{\mathrm{probe}}T}{M}
+\sum_{r=0}^{\Stdec}\Ccert_r\\
&+N\Stdec\bigl(d_{\mathrm{probe}}+\DG(\delta/3)\bigr)
\Bigr).
\end{aligned}
\]
Thus probe-triggered monitoring removes the conservative boundary term $N\Stdec\Hscan$. The trigger calibration and proof are in supplement~C.
\end{corollary}

\subsection{Information-Theoretic Lower Bound}

\begin{theorem}[Switch-axis lower bound]
\label{thm:info}
Fix margin $\Delta\in(0,1/8)$ and sufficiently long regimes. There are constants $c_{\mathrm c},c_{\mathrm g}>0$ and two different homogeneous instance families with lower-bound scales
\[
\begin{aligned}
\text{certification:}\quad
A&:=c_{\mathrm c}\Stdec\frac{K-1}{\Delta},\\
\text{graph propagation:}\quad
B&:=c_{\mathrm g}N\Stdec\DG\Delta.
\end{aligned}
\]
The first applies to rate-consistent algorithms by a per-arm change-of-measure argument when the regimes provide $\Omega(K\log T/(N\Delta^2))$ certification room. The second applies to decentralized graph-communication algorithms on a diameter-$\DG$ broom graph under the conditions stated in supplement~G. Because the constructions are different, every algorithm in the intersection faces an instance in their union satisfying
\[
\E[R_T]\geq\max\{A,B\}\geq\tfrac12(A+B).
\]
This is a two-instance minimax statement, not an additive lower bound on one instance. Full conditions and proofs are in supplement~G.
\end{theorem}

\begin{table}[H]
\centering
\small
\begin{tabular}{@{}ll@{}}
\toprule
Upper-bound cost & Corresponding lower-bound scale \\
\midrule
Certification & $\Stdec(K-1)/\Delta$ \\
Graph propagation & $N\Stdec\DG\Delta$ \\
Periodic monitoring & Not constrained by Theorem~\ref{thm:info} \\
\bottomrule
\end{tabular}
\caption{Axis-wise lower-bound comparison after Corollary~\ref{cor:probe-main} removes the boundary term $N\Stdec\Hscan$.}
\label{tab:switch-matching}
\end{table}

On homogeneous equal-gap instances, certification matches in $\Stdec$, $K$, and $1/\Delta$. Propagation matches in $\Stdec$ and $\DG$; the upper bound is worst-case in the per-round gap whereas the lower bound carries $\Delta$, so the table does not claim full gap-wise optimality.

\section{Local-Reactive Protocol Separation}

We complete the instantaneous-benchmark analysis with an illustrative, class-relative separation.

\begin{theorem}[Illustrative separation from local-reactive protocols]
\label{thm:separation}
For the class of protocols required to reset after sufficiently many agents experience local jumps, there is a cancellation instance with $\Stdec=0$ and $\Stloc=\Theta(T/L)$ on which every such protocol pays $\Omega(\Stloc)$ reset regret or control-communication cost, while Theorem~\ref{thm:main} has no $\Stloc$-dependent adaptation term. This is a class-relative illustrative separation, not the algorithm-independent lower bound of Theorem~\ref{thm:info}. The formal reactive-class definition, zero-sum sign construction, SW-UCB and CUSUM membership proofs, and per-reset cost argument are in supplement~E--F.
\end{theorem}

\section{Extension: Within-Regime Drift}

The preceding sections use the instantaneous best common arm and count its switches by $\Stdec$. We now study a distinct comparator: the best fixed common arm after averaging over each drift regime.

\subsection{Time-Average Decision Benchmark}

Let $1=\rho_0<\cdots<\rho_{\Stdecavg}<\rho_{\Stdecavg+1}=T+1$ partition the horizon into maximal average-decision regimes. For regime $r$, let $\Lreg_r=\rho_{r+1}-\rho_r$ and define
\[
a_r^{\mathrm{avg}}
=\arg\max_a\frac{1}{\Lreg_r}
\sum_{t=\rho_r}^{\rho_{r+1}-1}\mubar_{a,t}.
\]
Adjacent regimes have different unique maximizers; hence $\Stdecavg$ counts average-decision switches, distinct from the instantaneous count $\Stdec$.
The corresponding decision regret is
\[
R_T^{\mathrm{dec}}
=\sum_r\sum_{t=\rho_r}^{\rho_{r+1}-1}\sum_{i=1}^N
\bigl(\mu_{i,a_r^{\mathrm{avg}},t}-\mu_{i,A_{i,t},t}\bigr).
\]
This benchmark differs from the instantaneous regret $R_T$ in \eqref{eq:regret-def-anchor}. Because $a_r^{\mathrm{avg}}$ is optimal only after averaging over the regime, an individual summand---and even a short partial sum---may be negative. Thus $R_T^{\mathrm{dec}}$ measures regime-level decision quality, not tracking of an instantaneous oracle.

\subsection{Window-Average Margin}

\begin{assumption}[Window-stable decision margin]
\label{ass:tv-margin}
Under the average-decision regimes above, there are a window length $W$ in balanced probe blocks and a margin $\bar\Delta>0$ such that, for every regime $r$, every length-$W$ window $\mathcal W\subseteq[\rho_r,\rho_{r+1})$, and every $a\neq a_r^{\mathrm{avg}}$,
\[
\Gbar_{a_r^{\mathrm{avg}},a}(\mathcal W)\geq\bar\Delta.
\]
The instantaneous gap may change sign inside $\mathcal W$; only the window-average decision is required to remain stable.
\end{assumption}

\subsection{DRFC-Seq}

For a window starting at probe block $s$, $\Ghat_{b,a}^{(s)}(n)$ and $\Gbar_{b,a}^{(s)}(n)$ denote the empirical and corresponding equal-agent time-average gaps over the same $n$ per-arm observations.

\begin{lemma}[Uniform window correctness]
\label{lem:uniform-window}
Under Assumptions~\ref{ass:rewards} and~\ref{ass:flooding}, with probability at least $1-\alpha$, simultaneously for every ordered arm pair $(b,a)$, every window start $s$, and every per-arm count $n\geq1$,
\[
\begin{aligned}
\left|\Ghat_{b,a}^{(s)}(n)-\Gbar_{b,a}^{(s)}(n)\right|&\leq r_n,\\
u_{n,s}&:=\log\log(e\,n)\\
&\quad+\log(C_0K(K{-}1)s^2/\alpha),\\
r_n&=\sqrt{a_0u_{n,s}/(Nn)}.
\end{aligned}
\]
The guarantee is uniform over window starts, not merely valid for one fixed window. Its block-martingale proof is in supplement~H.
\end{lemma}

DRFC-Seq applies Lemma~\ref{lem:uniform-window} to a continuously sliding window. Every $\delta_{\mathrm p}$ rounds, agents collect one balanced all-arm probe block; a challenger replaces the candidate only when its lower-confidence window-average gap is positive. The monitor acts only after the window contains $W$ fully flooded blocks.

\begin{algorithm}[t]
\caption{DRFC-Seq: anytime-valid drift-robust monitor}
\label{alg:drfc-seq}
\begin{algorithmic}[1]
\STATE \textbf{Input:} block size $h$, probe gap $\delta_{\mathrm{p}}$, window $W$ (blocks), level $\alpha$, flooding $\DG(\alpha)$.
\STATE Initialize $\hat{a}$ arbitrarily to any arm, then set $\mathcal{Q}\leftarrow\emptyset$ (probe recovery adopts the best arm within $W$ blocks).
\FOR{$t=1$ to $T$}
\IF{$t$ is a probe time}
\STATE Run one synchronized balanced block with per-agent independent orders ($h$ pulls/arm/agent) over all arms, then flood records.
\STATE Append the fully-flooded equal-agent reward/pull sums to $\mathcal{Q}$, and if $|\mathcal{Q}|>W$ drop oldest.
\STATE Form windowed means $\hat\mu_a$ and counts $n_a$, then set $b\leftarrow\arg\max_{a\neq\hat a}\hat\mu_a$.
\STATE Compute the anytime-valid radius $r_n$ at the current window start $s$ from the stitched boundary above.
\STATE \textit{// act only on a full window ($|\mathcal{Q}|=W$), aligning with the length-$W$ margin (Assumption~\ref{ass:tv-margin})}
\IF{$|\mathcal{Q}|=W$ \textbf{and} $(\hat\mu_b-\hat\mu_{\hat a})-(r_{n_b}+r_{n_{\hat a}})>0$}
\STATE Emit switch, set $\hat{a}\leftarrow b$, sync $\DG(\alpha)$ rounds. \textit{// relabel only, sliding window (line 5) forgets, no flush}
\ENDIF
\ELSE
\STATE Exploit $\hat{a}$.
\ENDIF
\ENDFOR
\end{algorithmic}
\end{algorithm}

\subsection{DRFC-Seq Regret}

\begin{theorem}[DRFC-Seq decision regret under within-regime drift]
\label{thm:drfc-seq}
\normalfont
Run DRFC-Seq with a sliding window of $W$ balanced blocks, probe gap
$\delta_{\mathrm{p}}$, and radius $r_n$. Let $n=Wh$ be the per-arm sample
count in a full window. Suppose Assumptions~\ref{ass:rewards},
\ref{ass:flooding}, and~\ref{ass:tv-margin} hold. The length-$W$ window must
satisfy the time-average margin condition with margin $\bar\Delta$ inside
each regime; the full-window sample size must satisfy
$Wh\geq n_0(\bar\Delta)=\widetilde O(1/(N\bar\Delta^2))$; and every regime
must contain at least $2W$ probe blocks, equivalently at least
$2W\delta_{\mathrm{p}}+\DG(\alpha)$ rounds. The monitor is evaluated only on a
full window.

\textbf{Correctness.} With probability at least $1-\alpha$, once DRFC-Seq holds
$a_r^{\mathrm{avg}}$, no full in-regime window can induce a false switch away from it,
even when the instantaneous gap is negative. This follows simultaneously over
all window starts from Lemma~\ref{lem:uniform-window}.

\textbf{Adaptivity and regret.} Under $\Stdecavg$ average-decision switches satisfying the
spacing above, the window slides fully into each new regime within $W$ probe
blocks and DRFC-Seq adopts the new best arm. Its decision regret satisfies
\[
\begin{aligned}
R_T^{\mathrm{dec}}
=\widetilde O\Bigl(&
  NKh\,T/\delta_{\mathrm{p}}
  +\sum_{r=0}^{\Stdecavg}\Ccert^{\mathrm{seq}}_r \\
&\quad
  +N(\Stdecavg+1)(W\delta_{\mathrm{p}}+\DG(\alpha))
\Bigr),
\end{aligned}
\]
where $\Ccert^{\mathrm{seq}}_r=\widetilde O(K/\bar\Delta_r^2)$ is keyed to
the time-average margin. This is not the instantaneous $\Ccert_r$ of
Theorem~\ref{thm:main}, which need not exist under drift.
No term depends on $\Stloc$ or on the number of instantaneous sign changes.
\end{theorem}

\subsection{Memory--Adaptation Tradeoff}

For a monitor $\mathcal M$, let $\ell_{\mathrm{eff}}$ be the shortest observation window such that, conditional on the current candidate, the law of its adopt decision is determined by the last $\ell_{\mathrm{eff}}$ rounds.

\begin{proposition}[Memory--adaptation tradeoff under cancellation drift]
\label{prop:tvgap}
There exists a single-regime cancellation-drift instance ($\Stdecavg=0$) of period $P$ whose time-average gap is $\bar\Delta>0$ but whose instantaneous gap is negative during each down-phase. Any monitor with effective memory
\[
\ell_{\mathrm{eff}}<P/2
\]
that adopts a genuine $\bar\Delta$-margin switch within $\ell_{\mathrm{eff}}$ rounds with probability at least $q_0$ uniformly over prehistory false-adopts during a down-phase with probability at least $q_0$. A drift-safe monitor can escape only by using effective memory at least $P/2$ and paying the corresponding adaptation delay. A sample-sufficient DRFC-Seq window spanning a full period operates on this drift-safe side of the tradeoff. The formal detector class and proof are in supplement~H.1.
\end{proposition}

\subsection{Unknown Drift Period}

Doubling-Window DRFC-Seq removes exact prior knowledge of the drift period by running a geometric ladder of window lengths and operating on the shortest window whose certificate agrees with a drift-safe reference window. The guarantee assumes that the ladder's largest window covers the unknown period, is drift-safe, and is sample-sufficient; it may be chosen from a conservative upper range. Its one-time calibration and per-switch latency are given in supplement~H.2. This extension adapts the memory scale $W$ to an unknown period $P$; it is distinct from the gap-adaptive construction above, which adapts the scan budget $\Hscan$ to an unknown gap $\Delta_{\min}$.

\section{Experiments}

We compare DRFC with local-reactive monitors and with decision-level baselines that observe the same equal-agent signal: SW-UCB-Dec, DecCUSUM, DL-GLR-klUCB, and DL-ADR-bandit. The Oracle has centralized communication but no drift-period side information. Full configurations, confidence intervals, and tuning sweeps are in supplement~I.

\begin{figure*}[t]
\centering
\includegraphics[width=\textwidth]{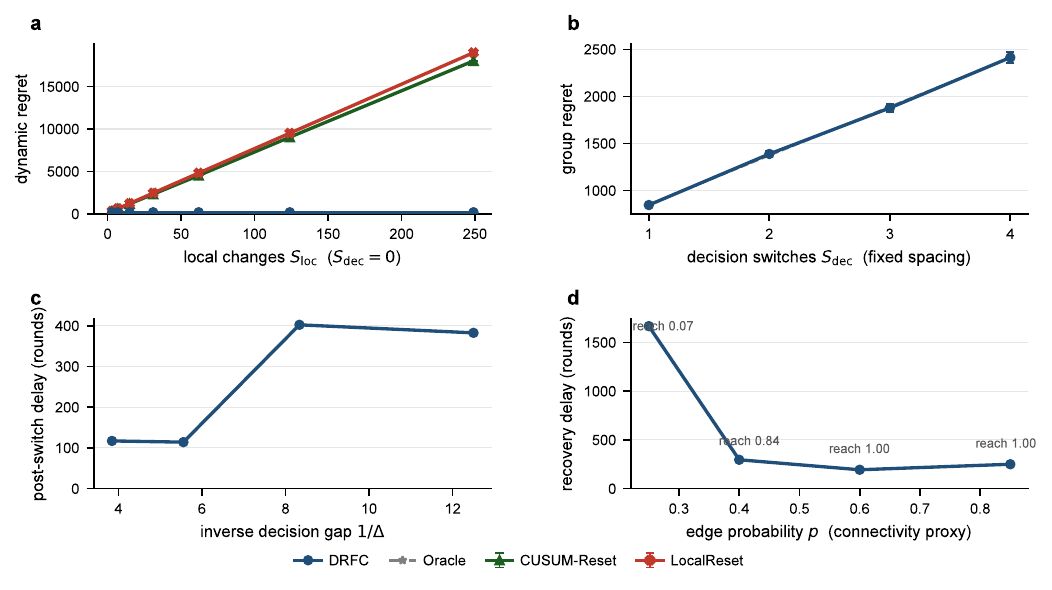}
\caption{Decision-relevant scaling. (a) At $\Stdec=0$, DRFC is insensitive to $\Stloc$, while local-reactive baselines grow. (b) At fixed spacing, regret is approximately linear in $\Stdec$. (c) Narrow gaps increase post-switch delay through harder certification. (d) Denser communication shortens recovery; labels give source-record reach, a flooding proxy rather than direct $\DG$. Curves show means and 95\% confidence intervals where available (100 runs in panel~a; 20 seeds in panels~b--d).}
\label{fig:theory-aligned}
\end{figure*}

\subsection{$\Stloc$--$\Stdec$ Separation}

Figure~\ref{fig:theory-aligned}a fixes $\Stdec=0$ while increasing $\Stloc$. At the largest sweep point ($\Stloc=249$), mean regret is $180$ for DRFC versus $18{,}016$ for CUSUM-Reset and $18{,}992$ for LocalReset, jointly testing Theorem~\ref{thm:main} and the class-relative separation of Theorem~\ref{thm:separation}.

\subsection{Switch-Axis Scaling}

Figures~\ref{fig:theory-aligned}b--d isolate switch count, certification difficulty, and communication delay. Regret is approximately linear in $\Stdec$, narrower gaps increase recovery delay consistently with harder certification, and denser graphs shorten recovery. The connectivity sweep is a proxy for $\DG$, not a direct measurement. These tests correspond to the two lower-bound components in Theorem~\ref{thm:info}; periodic monitoring remains a separate, unconstrained axis. Supplement~I.2 directly compares periodic and probe-triggered monitoring through their regret--source-record trade-off; full axis sweeps are in supplement~I.5.

\subsection{Within-Regime Drift}

At $\Stdecavg=0$ we sweep drift amplitude $b$, so every emitted average-decision switch is false, and report both $R_T^{\mathrm{dec}}$ and the false-switch rate. Once $b>\Delta$, the instantaneous gap changes sign while the window-average gap remains positive. DRFC-Seq alone stays at zero false switches, as predicted by Theorem~\ref{thm:drfc-seq}. Supplement~I reports the full drift-amplitude sweep and additionally varies the window length and drift period, testing Proposition~\ref{prop:tvgap} under both underspecified and conservative memory.

\subsection{Real-Data Replay}

On the raw, unanchored MovieLens-1M~\citep{harper2015movielens} genre replay ($\bar\Delta\approx0.0115$), DRFC-Seq attains the lowest decision regret and zero false switches among adopt-capable monitors (Table~\ref{tab:movielens-main}). Across a gain sweep, its false-switch separation from fast detectors is Holm-significant over 30 seeds for $g\ge2.0$. This validates cancellation and window memory; supplement~I.4--I.8 reports full results and topology, asynchrony, period-change, and $K=32$ checks.

\begin{table}[H]
\centering
\small
\setlength{\tabcolsep}{5pt}
\begin{tabular}{lcc}
\toprule
Method & $R_T^{\mathrm{dec}}$ & False-switch rate \\
\midrule
\textbf{DRFC-Seq} & $\mathbf{2710\pm6}$ & $\mathbf{0.00}$ \\
DecCUSUM & $3053\pm475$ & $1.00$ \\
DL-GLR-klUCB & $3459\pm736$ & $0.90$ \\
DL-ADR-bandit & $3868\pm500$ & $0.90$ \\
\bottomrule
\end{tabular}
\caption{Raw MovieLens drift ($\Stdecavg=0$, 30 seeds; mean $\pm$ s.d.); lower is better. Full results: supplement~I.}
\label{tab:movielens-main}
\end{table}

\FloatBarrier

\section{Discussion}
In heterogeneous networks, adaptation complexity follows common-decision switches $\Stdec$, not local changes $\Stloc$: fresh equal-agent certification ignores changes that cancel globally. DRFC-Probe replaces scheduled scans with cheap probes and triggered confirmation. The lower bounds make certification and graph propagation unavoidable; periodic monitoring remains a design choice. The drift extension changes the comparator: under the time-average benchmark, DRFC-Seq trades adoption speed for robustness to instantaneous sign changes. Thus both models certify only network-level decision changes. Extending them through push-sum averaging to directed or unreliable graphs~\citep{kempe2003gossip,nedic2015distributed} is a natural next step.

\paragraph{Limitations.}
(i)~Obliviousness excludes adaptive drift or participation. (ii)~Rewards are bounded and sub-Gaussian; bounded-variance extensions can use block median-of-means. (iii)~The guarantees require a unique common best arm and a positive regime/window margin; latency becomes vacuous near ties. (iv)~Communication is $\widetilde O(NK(h{+}\DG)T/M)$ source records, within $1.3\times$ of decision-level detectors; sharper edge bounds need finer flooding (supplement~I.2). (v)~Synchronized balanced blocks are demanding asynchronously; margin-dependent degradation is in supplement~I.4--I.7.

\label{contentend}

\clearpage
\section*{Supplementary Material}
This supplement contains:
(A) the full proof of the dynamic regret theorem, including the time-varying concentration lemma, sign preservation, active-set scan certification, dynamic regret decomposition, and the random-graph flooding instantiation;
(B) the Gap-Adaptive DRFC proposition, which removes prior knowledge of $\Delta_{\min}$ through a persistent geometric scan-budget schedule;
(C) the probe-triggered monitoring guarantee with explicit false-trigger budget;
(D) the distinct source-record complexity bound;
(E) the construction and proof of the class-relative separation theorem;
(F) verification that sliding-window UCB and CUSUM-style detectors lie inside the local-change-reactive comparator class on the construction;
(G) the full proof of the information-theoretic lower bound;
(H) the full anytime-valid correctness and adaptivity proof of DRFC-Seq under within-regime drift (the main DRFC-Seq theorem), via Ville's inequality and a law-of-the-iterated-logarithm confidence sequence, including (H.1) the drift-separation proof of the dichotomy proposition (main Proposition~1);
(I) additional experimental figures (component ablations, communication ablations, graph-probability sweep, user-cluster semi-real, monitor-period sensitivity);
(J) the proof of the personalized (per-cluster) regret theorem for DRFC-Pers;
(K) a reproducibility checklist.

\paragraph{Notation.} Table~\ref{tab:notation} collects the symbols used in the main paper and this supplement.

\begin{table}[h]
\centering
\footnotesize
\begin{tabular}{@{}ll@{}}
\toprule
Symbol & Meaning \\
\midrule
$N$, $K$ & number of agents, number of arms \\
$\Delta$ & per-regime decision margin, the equal-agent gap of the best common arm \\
$\Stloc$, $\Stdec$ & local-mean changes, instantaneous decision switches \\
$\Stdecavg$ & switches of the regime-average decision in the drift model \\
$M$ & monitor period, rounds between certification triggers \\
$h$ & block size, per-arm per-agent pulls in one balanced block \\
$\Hscan$ & scan budget, rounds spent certifying the best arm at a trigger \\
$\DG(\delta)$ & flooding time, rounds for a source record to reach all agents w.p.\ $1-\delta$ \\
$W$ & DRFC-Seq sliding-window length, in probe blocks \\
$\delta_{\mathrm{p}}$ & probe gap, rounds between consecutive probe blocks \\
$\alpha$ & DRFC-Seq anytime false-switch level \\
$r_n$ & anytime-valid confidence radius on the windowed equal-agent gap \\
$\Ccert_r$ & certification complexity $\sum_{a\ne a_r^\star}\Gamma_a^{(r)}/(\Delta_a^{(r)})^2$ \\
$\Lreg_r$ & length of decision regime $r$ \\
\bottomrule
\end{tabular}
\caption{Notation used throughout the main paper and supplement.}
\label{tab:notation}
\end{table}

\section{A. Proof Details}
\label{app:proofs}

This supplement expands the proof sketches in the main paper.
The purpose is to make explicit the sample-count convention, the finite good event, and the reason local distribution changes inside a decision regime do not enter the adaptation term.
All constants below are universal and may change from line to line. The controlling constants of the three quantitative results are listed explicitly in the following subsection.
Theorem, lemma, and equation references inside this supplement use independent numbering; cross-references to the main paper use result names whenever possible.

\subsection{Controlling Constants and Hidden Dependencies}
\label{app:constants}

For a theory reader we collect the controlling constants of the three quantitative results, so that no constant named in a theorem statement is left implicit. Every value below is the one used in the corresponding proof.

\paragraph{DRFC radius constant.}
The good-event radius of the equal-agent time-average gap over a window of $|\mathcal{W}|$ balanced blocks is
\[
\beta(\mathcal{W})=C\sqrt{\frac{\log(K^2T^3/\delta)}{Nh\,|\mathcal{W}|}},
\]
where $C$ is the universal constant of the two-term concentration bound of Lemma~\ref{lem:app-concentration}, combining reward noise by Hoeffding--Azuma with schedule-selection noise by sampling-without-replacement (Bardenet--Maillard). It satisfies $C\le 8$, and the scan-budget calibration uses the per-arm block count $B_a^{(r)}=\lceil 16C^2\log(K^2T^3/\delta)/(Nh(\Delta_a^{(r)})^2)\rceil$. The only factor the $\widetilde O$ of the main DRFC theorem suppresses is the single logarithm $\log(K^2T^3/\delta)$.

\paragraph{DRFC-Seq radius constants.}
The anytime-valid radius is
\[
r_n=\sqrt{\frac{a_0\bigl(\log\log(e\,n)+\log(C_0K(K-1)s^2/\alpha)\bigr)}{Nn}},
\]
with $a_0=\Theta(c_1)$, where $c_1$ is the sub-Gaussian variance proxy of one balanced block ($\sigma_\ell^2\le c_1/(Nh)$), and $C_0$ the stitched-boundary constant of \cite{howard2021timeuniform}, Theorem~1. Both are absolute. Here $s$ indexes the window start, and the term $\log(C_0K(K-1)s^2/\alpha)$ carries a summable $6/(\pi^2s^2)$ schedule over starts together with the $K(K-1)$ ordered arm pairs, so the guarantee holds with no predeclared horizon (Section~H), while $\log\log(e\,n)$ is the iterated-logarithm price of the anytime guarantee. These are the only factors the $\widetilde O$ of the main DRFC-Seq theorem suppresses.

\paragraph{Certification length.}
The per-arm sample count at which the windowed estimator separates the time-average-best arm from every challenger is
\[
n_0(\bar\Delta)=\inf\{n\ge 1:r_n\le\bar\Delta/4\}=\widetilde O\bigl(1/(N\bar\Delta^2)\bigr),
\]
so a window of $W$ blocks certifies once $Wh\ge n_0(\bar\Delta)$, giving $\Ccert^{\mathrm{seq}}_r=\widetilde O(K/\bar\Delta_r^2)$ keyed to the regime margin $\bar\Delta_r$. The block budget $h$ is a fixed design constant and enters $r_n$ only through $n=Wh$. Because $r_n$ carries $\log(C_0K(K-1)s^2/\alpha)$, the certification length $n_0$ is start-indexed; the recovery and regret statements use the worst start in the horizon, $n_0(\bar\Delta)=\widetilde O(\log(T)/(N\bar\Delta^2))$ through $s\le T$, the $\log T$ absorbed into $\widetilde O$. The no-false-switch guarantee of part~(a) needs no such cap and holds at every start by the summable schedule, so only the adaptivity rate, not the validity, sees the horizon.

\subsection{Comparison Windows}
\label{app:windows}

Fix an ordered arm pair $(a,b)$ and a tested comparison window $\mathcal{W}$.
The window consists of active-set blocks in which both $a$ and $b$ are active.
Let $B=|\mathcal{W}|$ be the number of such fresh blocks and let
\[
m=hB
\]
be the number of fresh pulls per arm, per agent.
For each block $\ell\in\mathcal{W}$, write
\[
\mathcal{U}_{\ell}=\{t_{\ell,1},\ldots,t_{\ell,q_\ell}\}
\]
for the common global-round macro-slot pool in that active-set block.
In a pure pairwise block, $q_\ell=2h$.
In an active-set block, $q_\ell$ also includes slots assigned to other active arms.
The active-set randomization assigns exactly $h$ slots to $a$ and exactly $h$ slots to $b$ from this block-level common pool, independently of rewards.
Equivalently, for each agent and block it draws disjoint subsets
\[
S_{i,a}^{(\ell)},S_{i,b}^{(\ell)}\subset \mathcal{U}_{\ell},
\qquad
|S_{i,a}^{(\ell)}|=|S_{i,b}^{(\ell)}|=h,
\qquad
S_{i,a}^{(\ell)}\cap S_{i,b}^{(\ell)}=\emptyset,
\]
with uniform marginals over size-$h$ subsets of $\mathcal{U}_{\ell}$.
The local and global estimators are
\[
\hat{g}_{i,a,b}(\mathcal{W})
=
\frac{1}{B}\sum_{\ell\in\mathcal{W}}
\frac{1}{h}\sum_{t\in S_{i,a}^{(\ell)}}X_{i,a,t}
-
\frac{1}{B}\sum_{\ell\in\mathcal{W}}
\frac{1}{h}\sum_{t\in S_{i,b}^{(\ell)}}X_{i,b,t},
\qquad
\Ghat_{a,b}(\mathcal{W})
=
\frac{1}{N}\sum_{i=1}^{N}\hat{g}_{i,a,b}(\mathcal{W}).
\]
The concentration target is the equal-agent time-average pairwise gap
\[
\Gbar_{a,b}(\mathcal{W})
=
\frac{1}{N}\sum_{i=1}^{N}
\frac{1}{B}\sum_{\ell\in\mathcal{W}}
\frac{1}{q_\ell}
\sum_{t\in \mathcal{U}_{\ell}}
\left(\mu_{i,a,t}-\mu_{i,b,t}\right).
\]
This is the key object: it is well defined even when the local means vary within the window.
The pure pairwise comparison in the main text is the special case $q_\ell=2h$ for every block; the active-set scan uses the same concentration statement for pairwise projections of larger block-level macro-slot pools.

Let $\mathfrak{W}$ be the family of windows on which DRFC may evaluate a decision rule.
A window is indexed by an epoch id ($\leq \Stdec+1\leq T$ since each accepted switch increments the epoch), an ordered arm pair ($K(K-1)$ choices), and a contiguous block interval (at most $T^2$ start--end pairs of global rounds).
This gives the conservative count
\[
|\mathfrak{W}|\leq K(K-1)\cdot T\cdot T^2 = K(K-1)T^3.
\]
The confidence radius used below is
\[
\beta(\mathcal{W})
=
C
\sqrt{
\frac{\log(K^2T^3/\delta)}
{Nh|\mathcal{W}|}
}.
\]

\subsection{Time-Varying Concentration}
\label{app:concentration}

\begin{lemma}[Uniform time-varying comparison concentration]
\label{lem:app-concentration}
For any fixed synchronized pairwise or active-set-projection comparison window $\mathcal{W}$ and ordered pair $(a,b)$,
\[
\Prb\left[
\left|\Ghat_{a,b}(\mathcal{W})-\Gbar_{a,b}(\mathcal{W})\right|
>
C\sqrt{\frac{\log(1/\delta)}{Nm}}
\right]\leq \delta .
\]
Consequently, after increasing $C$, the bound
\[
\left|\Ghat_{a,b}(\mathcal{W})-\Gbar_{a,b}(\mathcal{W})\right|
\leq
\beta(\mathcal{W})
\]
holds simultaneously for all tested windows with probability at least $1-\delta/3$.
\end{lemma}

\begin{proof}
Throughout, let $m=hB$ be the per-arm per-agent sample count in the window.
By main Assumption~1, the local mean path is $\mathcal{F}_0$-measurable and each block's fresh randomization $\{S_{i,a}^{(\ell)},S_{i,b}^{(\ell)}\}_{i,\ell}$ is independent of $\mathcal{F}_{t-1}$.
Let $\mathcal{S}$ denote the entire collection of schedules used in the window.
Decompose the estimation error as
\[
\Ghat_{a,b}(\mathcal{W})-\Gbar_{a,b}(\mathcal{W})=A+R,
\]
where
\[
A
=
\Ghat_{a,b}(\mathcal{W})
-
\E\!\left[\Ghat_{a,b}(\mathcal{W})\mid \mathcal{S},\,\mathcal{F}_{0}\right]
\]
is the reward noise conditional on schedules and the mean path, and
\[
R
=
\E\!\left[\Ghat_{a,b}(\mathcal{W})\mid \mathcal{S},\,\mathcal{F}_{0}\right]
-
\Gbar_{a,b}(\mathcal{W})
\]
is the schedule-randomization error in the conditional mean.

\textit{Reward noise term $A$.}
Conditional on $\mathcal{S}$ and $\mathcal{F}_0$,
\[
A
=
\frac{1}{N}\sum_{i=1}^N
\frac{1}{B}\sum_{\ell\in\mathcal{W}}
\frac{1}{h}
\left(
\sum_{t\in S_{i,a}^{(\ell)}}(X_{i,a,t}-\mu_{i,a,t})
-
\sum_{t\in S_{i,b}^{(\ell)}}(X_{i,b,t}-\mu_{i,b,t})
\right).
\]
Order the $2NhB=2Nm$ summands lexicographically, first by round $t$ and within a round by agent index $i$, and let $\mathcal{G}_{t,i}$ be $\mathcal{F}_{t-1}\vee\sigma(\mathcal{S})$ enlarged by the round-$t$ observed noises of the agents preceding $i$ in this order.
By the exogeneity clause of main Assumption~1, the reward noises are independent of the learner randomization $\mathcal{S}$ and each summand has zero mean conditional on $\mathcal{F}_{t-1}$ and all learner randomization, and by the cross-agent clause the round-$t$ noises of different agents are mutually independent under the same conditioning, so conditioning additionally on the preceding same-round noises leaves the mean at zero, $\E[X_{i,a,t}-\mu_{i,a,t}\mid\mathcal{G}_{t,i}]=0$.
Without that clause the $1/\sqrt{N}$ rate below can fail, since a round-level noise common to all agents satisfies the per-round martingale property yet makes $A$ concentrate only at rate $1/\sqrt{m}$.
The ordered summands are therefore a bounded martingale difference sequence in $[-1,1]$, and $A$ combines them with coefficients $1/(NhB)=1/(Nm)$.
Hoeffding-Azuma~\cite{azuma1967weighted} therefore gives
\[
\Prb\!\left(|A|>x\mid \mathcal{S},\mathcal{F}_0\right)
\leq
2\exp\!\left(-\frac{2x^2}{2Nm\cdot(1/(Nm))^2\cdot(2)^2}\right)
=
2\exp\!\left(-\frac{Nm x^2}{4}\right).
\]
Marginalizing over $\mathcal{S}$ preserves the bound.

\textit{Schedule-randomization term $R$.}
For each agent $i$ and block $\ell\in\mathcal{W}$, define the per-block, per-arm population block means
\[
\bar{M}_{i,a}^{(\ell)}=\frac{1}{q_\ell}\sum_{t\in\mathcal{U}_\ell}\mu_{i,a,t},
\qquad
\bar{M}_{i,b}^{(\ell)}=\frac{1}{q_\ell}\sum_{t\in\mathcal{U}_\ell}\mu_{i,b,t},
\]
and their empirical counterparts on the block schedule,
\[
\widehat{M}_{i,a}^{(\ell)}=\frac{1}{h}\sum_{t\in S_{i,a}^{(\ell)}}\mu_{i,a,t},
\qquad
\widehat{M}_{i,b}^{(\ell)}=\frac{1}{h}\sum_{t\in S_{i,b}^{(\ell)}}\mu_{i,b,t}.
\]
The block-level deviations are
\[
Z_{i,a}^{(\ell)}=\widehat{M}_{i,a}^{(\ell)}-\bar{M}_{i,a}^{(\ell)},
\qquad
Z_{i,b}^{(\ell)}=\widehat{M}_{i,b}^{(\ell)}-\bar{M}_{i,b}^{(\ell)},
\]
and the agent-level deviation aggregates to
\[
R_i
=
\frac{1}{B}\sum_{\ell\in\mathcal{W}}
\left(Z_{i,a}^{(\ell)}-Z_{i,b}^{(\ell)}\right),
\qquad
R=\frac{1}{N}\sum_{i=1}^N R_i .
\]

\emph{Block-level concentration.}
Fix $(i,\ell)$. The active-set randomization draws a uniform random ordered partition of $\mathcal{U}_\ell$ into $(S_{i,a}^{(\ell)},S_{i,b}^{(\ell)},\mathcal{U}_\ell\setminus(S_{i,a}^{(\ell)}\cup S_{i,b}^{(\ell)}))$ subject to the size constraints $|S_{i,a}^{(\ell)}|=|S_{i,b}^{(\ell)}|=h$. By symmetry of this partition, the marginal law of $S_{i,a}^{(\ell)}$ is uniform over size-$h$ subsets of $\mathcal{U}_\ell$, and similarly for $S_{i,b}^{(\ell)}$. Each marginal therefore satisfies the Hoeffding--Serfling inequality for sampling without replacement~\cite[Cor.~2.4]{bardenet2015concentration}:
\[
\Prb\!\left(\bigl|Z_{i,a}^{(\ell)}\bigr|>x\right)
\leq
2\exp\!\left(-\frac{2hx^2}{1-(h-1)/q_\ell}\right)
\leq
2\exp\!\left(-2hx^2\right),
\]
and identically for $Z_{i,b}^{(\ell)}$.
Hence both $Z_{i,a}^{(\ell)}$ and $Z_{i,b}^{(\ell)}$ are sub-Gaussian with proxy variance $\sigma_0^2=1/(4h)$, regardless of the joint dependence between them.

\emph{Aggregation across blocks.}
Different blocks $\ell\in\mathcal{W}$ use independently drawn schedules (main Assumption~1), so the pairs $\{(Z_{i,a}^{(\ell)},Z_{i,b}^{(\ell)})\}_{\ell\in\mathcal{W}}$ are independent across $\ell$ for fixed $i$. For each $\ell$, the difference $Z_{i,a}^{(\ell)}-Z_{i,b}^{(\ell)}$ is the sum of two sub-Gaussian random variables and is therefore sub-Gaussian with proxy variance at most $2\cdot 4\sigma_0^2 = 2/h$ (using the elementary bound that the sub-Gaussian proxy variance of a sum of two zero-mean variables is at most twice the sum of their individual proxies). Summing $B$ independent copies and dividing by $B$ gives that $R_i$ is sub-Gaussian with proxy variance $2/(hB)=2/m$, so
\[
\Prb(|R_i|>t)\leq 2\exp\!\left(-\frac{mt^2}{4}\right).
\]

\emph{Aggregation across agents.}
The agent-level schedules are independent across $i$ (main Assumption~1), so the $\{R_i\}_{i=1}^N$ are independent sub-Gaussian random variables with the same proxy variance $2/m$. Therefore $R=N^{-1}\sum_i R_i$ is sub-Gaussian with proxy variance $2/(Nm)$, giving
\[
\Prb(|R|>x)
\leq
2\exp\!\left(-\frac{Nm x^2}{4}\right).
\]

\textit{Combining.}
By the union bound,
\[
\Prb\!\left(|\Ghat_{a,b}(\mathcal{W})-\Gbar_{a,b}(\mathcal{W})|>2x\right)
\leq
\Prb(|A|>x)+\Prb(|R|>x)
\leq
4\exp\!\left(-\frac{Nm x^2}{4}\right).
\]
Setting the right-hand side equal to $\delta$ yields the fixed-window deviation $2x=4\sqrt{\log(4/\delta)/(Nm)}$, which is at most $C\sqrt{\log(1/\delta)/(Nm)}$ for a universal constant $C$ (e.g.\ $C=8$ for $\delta\leq 1/2$, since $\log(4/\delta)/\log(1/\delta)\leq 2$ in that range).

\textit{Simultaneous bound over $\mathfrak{W}$.}
Apply the fixed-window bound with failure probability $\delta'=\delta/(3|\mathfrak{W}|)$ where $|\mathfrak{W}|\leq K(K-1)T^3$.
This replaces $\log(1/\delta)$ by $\log(K(K-1)T^3/\delta)+\log 3\leq 2\log(K^2T^3/\delta)$, which is absorbed into $C$ in the radius $\beta(\mathcal{W})$.
Union bounding over $\mathfrak{W}$ gives the simultaneous bound with probability at least $1-\delta/3$, as claimed.
\end{proof}

\subsection{A Random-Graph Flooding Instantiation}
\label{app:graph-flooding}

The main theorem treats flooding through the abstract uniform quantity $\DG(\delta)$; we give one conservative instantiation.
Let $\{G_t\}$ be independent Erd\H{o}s--R\'enyi graphs on $[N]$ with edge probability $p$, and let $\mathcal{T}\subseteq[T]$, $|\mathcal{T}|\leq T$, be the analyzed communication rounds.

We claim that for a universal constant $c$,
\[
p\ \geq\ c\,\frac{\log(NT/\delta)}{N}
\quad\Longrightarrow\quad
\Prb\Bigl[\,\forall t\in\mathcal{T}:\ G_t\text{ connected}\,\Bigr]\ \geq\ 1-\delta/3,
\]
and that on this event every relevant source-traceable record reaches all agents within $\DG(\delta/3)\leq N-1$ rounds.

\emph{Connectivity.} For fixed $t$, a disconnection is witnessed by a cut $S$ with $s:=|S|\leq N/2$ and no crossing edge, so by a union bound over cuts,
\[
\Prb[G_t\text{ disconnected}]
\ \leq\ \sum_{s=1}^{\lfloor N/2\rfloor}\binom{N}{s}(1-p)^{s(N-s)}
\ \leq\ \sum_{s=1}^{\lfloor N/2\rfloor}\exp\bigl\{s\log(eN/s)-ps(N-s)\bigr\}.
\]
For $s\leq N/2$ one has $N-s\geq N/2$, so $ps(N-s)\geq \tfrac{pN}{2}s\geq 2s\log(eN/s)$ once $p\geq c\log(NT/\delta)/N$ with $c$ large enough, whence each summand is at most $e^{-s\log(eN/s)}\leq(\delta/(3NT))$ and $\Prb[G_t\text{ disconnected}]\leq\delta/(3T)$. A union bound over $t\in\mathcal{T}$ gives the claimed connectivity event with probability at least $1-\delta/3$.

\emph{Flooding time.} Fix a record with informed set $S_\tau\subsetneq[N]$ at round $\tau$. On the connectivity event $G_{\tau+1}$ has at least one edge between $S_\tau$ and $[N]\setminus S_\tau$, so $|S_{\tau+1}|\geq|S_\tau|+1$. Iterating, $|S_{\tau+k}|\geq\min\{N,|S_\tau|+k\}$, hence $S_{\tau+k}=[N]$ for some $k\leq N-1$, i.e.\ $\DG(\delta/3)\leq N-1$.
Sharper random-graph gossip bounds may replace this instantiation without altering the regret proof, which uses only the resulting $\DG(\delta)$.

\subsection{Uniform Good Event}
\label{app:good-event}

\begin{lemma}[Uniform good event]
\label{lem:uniform-good-event}
Suppose the random graph process satisfies the uniform flooding assumption with budget $\delta/3$, and agents use deterministic epoch tie-breaking for switch records.
Then, with probability at least $1-\delta$, the following events hold simultaneously:
\begin{enumerate}
    \item every tested estimator is within its radius $\beta(\mathcal{W})$;
    \item every source-traceable comparison or switch record relevant to the regret analysis reaches all agents within $\DG(\delta/3)$ rounds;
    \item every accepted switch produces a common candidate and epoch id within one synchronization window.
\end{enumerate}
\end{lemma}

\begin{proof}
We track the $\delta$-budget explicitly.

\textit{Concentration event ($\delta/3$).}
By Lemma~\ref{lem:app-concentration} with parameter $\delta/3$, the simultaneous bound $|\Ghat_{a,b}(\mathcal{W})-\Gbar_{a,b}(\mathcal{W})|\leq\beta(\mathcal{W})$ holds for all $\mathcal{W}\in\mathfrak{W}$ and all ordered active pairs with probability at least $1-\delta/3$.
Under the Gap-Adaptive DRFC schedule of supplement~B, the same bound applies simultaneously across the $k_{\max}+1\leq 1+\log_2(\Delta^{(0)}/\underline\Delta)$ doubling levels by replacing $\delta/3$ by $\delta/(3(k_{\max}+1))$ inside Lemma~\ref{lem:app-concentration}; the additional $\log(k_{\max}+1)$ factor is absorbed into the $\widetilde{O}$.

\textit{Flooding event ($\delta/3$).}
By the main random-graph flooding assumption with parameter $\delta/3$, every flooding window in $\mathcal{T}_{\mathrm{flood}}$ propagates its source record to all agents within $\DG(\delta/3)$ rounds, simultaneously with probability at least $1-\delta/3$.
The cardinality $|\mathcal{T}_{\mathrm{flood}}|\leq N\cdot K(K-1)T^3$ is absorbed in $\DG$ through the union bound described in the assumption.

\textit{Synchronization event ($\delta/3$).}
Condition on the previous two events.
Switch records contain source agent id, epoch id, old candidate, new candidate, and triggering comparison id.
Agents adopt the highest received epoch id; ties at the same epoch are broken by a fixed deterministic rule (e.g., smallest source agent id).
On the concentration event, Lemma~\ref{lem:app-pairwise} (proven below) shows that pairwise lower-confidence gaps cannot eliminate $a_r^\star$ in regime $r$, and they eliminate each non-best arm $a$ once the radius is below $\Delta_a^{(r)}/4$.
A scan that does not yet have a unique active arm is treated as inconclusive and does not emit a switch record.
Once a switch record is emitted, it reaches all agents within $\DG(\delta/3)$ rounds on the flooding event.
Two simultaneous switch records with the same epoch are resolved by the deterministic tie-breaking rule.
Therefore the synchronization event---all agents share the same $(\hat{a},e)$ after one synchronization window---fails with conditional probability zero.
The remaining $\delta/3$ slack is reserved for the doubling-level union bound above.

\textit{Conclusion.}
By the union bound, all three events hold simultaneously with probability at least $1-\delta$.
\end{proof}

\subsection{Sign Preservation and Pairwise Decisions}
\label{app:pairwise}

\begin{lemma}[Sign preservation]
\label{lem:app-sign}
If $\mathcal{W}$ is contained in decision regime $r$, then for every $b\neq a_r^\star$,
\[
\Gbar_{a_r^\star,b}(\mathcal{W})\geq \Delta_b^{(r)}.
\]
For every $a\neq a_r^\star$,
\[
\Gbar_{a,a_r^\star}(\mathcal{W})\leq -\Delta_a^{(r)}.
\]
\end{lemma}

\begin{proof}
Fix $\mathcal{W}\subseteq[\rho_r,\rho_{r+1})$ and $b\neq a_r^\star$.
\[
\begin{aligned}
\Gbar_{a_r^\star,b}(\mathcal{W})
={}& \frac1N\sum_{i=1}^N\frac1B\sum_{\ell\in\mathcal{W}}\frac1{q_\ell}\sum_{t\in\mathcal{U}_\ell}\bigl(\mu_{i,a_r^\star,t}-\mu_{i,b,t}\bigr)\\
\overset{\text{①}}{\ge}{}& \frac1N\sum_{i=1}^N\frac1B\sum_{\ell\in\mathcal{W}}\frac1{q_\ell}\sum_{t\in\mathcal{U}_\ell}\Delta_b^{(r)}\\
={}& \Delta_b^{(r)}.
\end{aligned}
\]
Symmetrically, for $a\neq a_r^\star$,
\[
\begin{aligned}
\Gbar_{a,a_r^\star}(\mathcal{W})
={}& \frac1N\sum_{i=1}^N\frac1B\sum_{\ell\in\mathcal{W}}\frac1{q_\ell}\sum_{t\in\mathcal{U}_\ell}\bigl(\mu_{i,a,t}-\mu_{i,a_r^\star,t}\bigr)\\
\overset{\text{②}}{\le}{}& -\Delta_a^{(r)}.
\end{aligned}
\]
Here ① applies the decision-regime margin bound $\mu_{i,a_r^\star,t}-\mu_{i,b,t}\ge\Delta_b^{(r)}$ at every $(i,t)$ with $t\in[\rho_r,\rho_{r+1})$, ② applies $\mu_{i,a,t}-\mu_{i,a_r^\star,t}\le-\Delta_a^{(r)}$ at every such $(i,t)$, and the closing equalities use the convex weights $\tfrac1N\cdot\tfrac1B\cdot\tfrac1{q_\ell}\ge0$ summing to $1$.
\end{proof}

\begin{lemma}[Pairwise correctness]
\label{lem:app-pairwise}
On the uniform good event, consider a tested window inside decision regime $r$.
\begin{enumerate}
    \item (\emph{Reverse direction, unconditional.}) For every $a\neq a_r^\star$, the lower-confidence comparison from $a$ to $a_r^\star$ is strictly negative:
    \[
    \Ghat_{a,a_r^\star}(\mathcal{W})-\beta(\mathcal{W})<0.
    \]
    No hypothesis on $\beta(\mathcal{W})$ is needed beyond the good event.
    \item (\emph{Forward direction, threshold.}) For every $b\neq a_r^\star$, if $\beta(\mathcal{W})\leq \Delta_b^{(r)}/4$, then the lower-confidence comparison from $a_r^\star$ to $b$ is strictly positive:
    \[
    \Ghat_{a_r^\star,b}(\mathcal{W})-\beta(\mathcal{W})>0.
    \]
\end{enumerate}
\end{lemma}

\begin{proof}
\emph{Part 1 (reverse).}
\[
\begin{aligned}
\Ghat_{a,a_r^\star}(\mathcal{W})-\beta(\mathcal{W})
\overset{\text{①}}{\le}{}& -\Delta_a^{(r)}\\
\overset{\text{②}}{<}{}& 0,
\end{aligned}
\]
where ① is Lemma~\ref{lem:app-sign} on the good event, $\Ghat_{a,a_r^\star}(\mathcal{W})\le-\Delta_a^{(r)}+\beta(\mathcal{W})$, subtracting $\beta(\mathcal{W})$, and ② uses $\Delta_a^{(r)}>0$; no upper bound on $\beta(\mathcal{W})$ is needed.

\emph{Part 2 (forward).}
\[
\begin{aligned}
\Ghat_{a_r^\star,b}(\mathcal{W})-\beta(\mathcal{W})
\overset{\text{①}}{\ge}{}& \Delta_b^{(r)}-2\beta(\mathcal{W})\\
\overset{\text{②}}{\ge}{}& \Delta_b^{(r)}/2\\
>{}& 0,
\end{aligned}
\]
where ① is Lemma~\ref{lem:app-sign} on the good event, $\Ghat_{a_r^\star,b}(\mathcal{W})\ge\Delta_b^{(r)}-\beta(\mathcal{W})$, subtracting $\beta(\mathcal{W})$, and ② uses the hypothesis $\beta(\mathcal{W})\le\Delta_b^{(r)}/4$.
\end{proof}

\begin{lemma}[Active-set scan certification]
\label{lem:app-active-scan}
On the uniform good event, consider an active-set full-certification scan whose blocks are contained in decision regime $r$.
The pairwise lower-confidence elimination rule never eliminates $a_r^\star$.
Moreover, each active arm $a\neq a_r^\star$ is eliminated once its pairwise projection against $a_r^\star$ has radius at most $\Delta_a^{(r)}/4$.
Let $\tau_a$ be the number of active-set blocks in which arm $a$ remains active.
Here $\tau_a$ counts active-set blocks until $a$'s \emph{synchronous} elimination, i.e.\ until the balanced count reaches $B_a^{(r)}$, so $1\le\tau_a\le B_a^{(r)}=\lceil x_a\rceil$ with $x_a=16C^2\log(K^2T^3/\delta)/(Nh(\Delta_a^{(r)})^2)$. Under the nonblocking-gossip convention the eliminating equal-agent records flood with lag $\le\DG(\delta/3)$, so beyond its $B_a^{(r)}$ certification blocks each arm stays sampled for at most $\DG(\delta/3)$ further rounds, a delayed-elimination over-sampling of at most $K\DG(\delta/3)$ group-rounds per scan (the same budget added to $\Hscan$ below) at group regret $\le N$ each. Using $\lceil x_a\rceil\le x_a+1$, the one completed scan contributes group regret
\[
\text{(one-scan group regret)}\ \le\ \widetilde{O}(\Ccert_r)\ +\ Nh\!\!\sum_{a\neq a_r^\star}\!\!\Gamma_a^{(r)}\ +\ NK\DG(\delta/3),
\qquad
\Ccert_r=\sum_{a\neq a_r^\star}\frac{\Gamma_a^{(r)}}{(\Delta_a^{(r)})^2}.
\]
The gap-dependent term collects the certification cost $Nhx_a\Gamma_a^{(r)}=\widetilde{O}(\Gamma_a^{(r)}/(\Delta_a^{(r)})^2)$ of each arm, the floor $Nh\sum_{a\neq a_r^\star}\Gamma_a^{(r)}\le Nh(K-1)$ is the mandatory single balanced block of each arm, and the last term is the nonblocking flooding-lag over-sampling. Summed over the $1+\Lreg_r/M$ scans of every regime, the floor and the flooding lag together give the periodic-monitoring term $\widetilde{O}(NK(h+\DG(\delta/3))T/M)$ of Theorem~\ref{thm:app-regret}, the per-switch part of the lag merging into the $N\Stdec\Hscan$ boundary term.
\end{lemma}

\begin{proof}
While arms $a$ and $a_r^\star$ are both active, each active-set block samples them over the same randomized time distribution.
The pairwise projection of those blocks is therefore a valid comparison window in the sense of Lemma~\ref{lem:app-concentration}.

\textit{Step 1: best arm is never eliminated.}
For any active non-best arm $b$ and any tested projection between $b$ and $a_r^\star$ contained in regime $r$, Lemma~\ref{lem:app-pairwise} gives
\[
\Ghat_{b,a_r^\star}(\mathcal{W})-\beta(\mathcal{W})<0.
\]
Thus no active arm has a positive lower-confidence gap against $a_r^\star$, and the elimination rule never removes $a_r^\star$.

\textit{Step 2: non-best arms eliminated at explicit threshold.}
Fix $a\neq a_r^\star$.
By the definition of $\beta$ from Appendix~\ref{app:windows},
\[
\beta(\mathcal{W})
=
C
\sqrt{\frac{\log(K^2T^3/\delta)}{Nh|\mathcal{W}|}}
\leq
\frac{\Delta_a^{(r)}}{4}
\iff
Nh|\mathcal{W}|
\geq
\frac{16C^2\log(K^2T^3/\delta)}{(\Delta_a^{(r)})^2}.
\]
Define
\[
B_a^{(r)}
=
\left\lceil\frac{16C^2\log(K^2T^3/\delta)}{Nh(\Delta_a^{(r)})^2}\right\rceil,
\]
which is the per-arm block count used by the scan-budget calibration after absorbing the constant $16C^2$ into the unspecified constant in that calibration.
Once arm $a$ has participated in $|\mathcal{W}|=B_a^{(r)}$ pairwise-projection blocks against $a_r^\star$, Lemma~\ref{lem:app-pairwise} gives
\[
\Ghat_{a_r^\star,a}(\mathcal{W})-\beta(\mathcal{W})>0,
\]
so the lower-confidence pairwise gap from $a_r^\star$ to $a$ is positive and the elimination rule removes $a$.
Therefore $\tau_a\leq B_a^{(r)}$, and
\[
h\tau_a
\leq
hB_a^{(r)}
=
\frac{16C^2\log(K^2T^3/\delta)}{N(\Delta_a^{(r)})^2}
=
\widetilde{O}\!\left(\frac{1}{N(\Delta_a^{(r)})^2}\right).
\]

\textit{Step 3: regret accounting.}
Fix arm $a\neq a_r^\star$.
Because the scan is contained in regime $r$, $a_u^\star=a_r^\star$ at every slot $u$ in its blocks.
Let $\mathcal{W}_a$ index the $\tau_a$ active-set blocks in which $a$ remains active, and write the group regret charged to arm $a$ as
\[
\mathcal{R}_a
=
\sum_{\ell\in\mathcal{W}_a}R_a^{(\ell)},
\qquad
R_a^{(\ell)}
=
\sum_{i=1}^N\sum_{u\in S_{i,a}^{(\ell)}}\bigl(\mu_{i,a_r^\star,u}-\mu_{i,a,u}\bigr)\in[-Nh,Nh].
\]
Note that the per-pull integrand $\mu_{i,a_r^\star,u}-\mu_{i,a,u}$ is a \emph{per-agent} gap that can exceed the \emph{group-average} gap $\Gamma_a^{(r)}$, so a deterministic per-pull bound by $\Gamma_a^{(r)}$ is not available; we instead control $\mathcal{R}_a$ in expectation through the schedule randomization and add a Hoeffding slack.

\textit{Expected regret.}
The marginal of $S_{i,a}^{(\ell)}$ is uniform over size-$h$ subsets of $\mathcal{U}_\ell$ and is drawn independently across $(i,\ell)$. By uniform sampling without replacement,
\[
\E\!\left[R_a^{(\ell)}\,\Big|\,\mathcal{F}_0\right]
=
\frac{Nh}{q_\ell}\sum_{u\in\mathcal{U}_\ell}\bigl(\bar\mu_{a_r^\star,u}-\bar\mu_{a,u}\bigr)
\leq Nh\Gamma_a^{(r)},
\]
where the inequality uses $\bar\mu_{a_r^\star,u}-\bar\mu_{a,u}\leq\Gamma_a^{(r)}$ for every $u\in[\rho_r,\rho_{r+1})$ from the main sign-stable decision-regime definition.

\textit{Concentration slack.}
Conditional on $\mathcal{F}_0$, the per-agent block contribution $\sum_{u\in S_{i,a}^{(\ell)}}(\mu_{i,a_r^\star,u}-\mu_{i,a,u})$ lies in $[-h,h]$ and is independent across $i$; block-level schedules are independent across $\ell$ (main Assumption~1). By Hoeffding's inequality applied to the $N\tau_a$ independent bounded summands,
\[
\Prb\!\left(\mathcal{R}_a-\E[\mathcal{R}_a\mid\mathcal{F}_0]>t\,\Big|\,\mathcal{F}_0\right)
\leq
\exp\!\left(-\frac{t^2}{2Nh^2\tau_a}\right).
\]
Choosing $t=h\sqrt{2N\tau_a\log(3K^2T^3/\delta)}$ makes the right-hand side at most $\delta/(3K^2T^3)$, and a union bound over the at most $|\mathfrak{W}|\leq K^2T^3$ tested $(a,\mathcal{W}_a)$ pairs absorbs the failure event into the $\delta/3$ slack already reserved in Lemma~\ref{lem:app-concentration}.
Hence, on the good event,
\[
\mathcal{R}_a
\leq
Nh\tau_a\Gamma_a^{(r)}
+
h\sqrt{2N\tau_a\log(3K^2T^3/\delta)}.
\]

\textit{Substitution.}
From Step~2, $\tau_a\le B_a^{(r)}=\lceil x_a\rceil\le x_a+1$ with $x_a=16C^2\log(K^2T^3/\delta)/(Nh(\Delta_a^{(r)})^2)$, so $Nh\tau_a\Gamma_a^{(r)}\le Nhx_a\Gamma_a^{(r)}+Nh\Gamma_a^{(r)}$ splits into a certification part and one mandatory block:
\[
Nhx_a\Gamma_a^{(r)}
=
\frac{16C^2\,\Gamma_a^{(r)}\log(K^2T^3/\delta)}{(\Delta_a^{(r)})^2}
=
\widetilde{O}\!\left(\frac{\Gamma_a^{(r)}}{(\Delta_a^{(r)})^2}\right),
\qquad
Nh\Gamma_a^{(r)}\le Nh,
\]
and
\[
h\sqrt{N\tau_a\log(\cdot)}
=
\sqrt{Nh\cdot h\tau_a\log(\cdot)}
\leq
\frac{4C\sqrt{h}\,\log(K^2T^3/\delta)}{\Delta_a^{(r)}}
=
\widetilde{O}\!\left(\frac{\sqrt{h}}{\Delta_a^{(r)}}\right).
\]
Since $\Gamma_a^{(r)}\geq\Delta_a^{(r)}$ (sup vs.\ inf of the same gap), $1/\Delta_a^{(r)}\leq\Gamma_a^{(r)}/(\Delta_a^{(r)})^2$, so the Hoeffding slack is dominated by the expected-regret term up to a $\sqrt{h}$ factor; treating $h$ as a fixed scheme parameter absorbs $\sqrt{h}$ into the $\widetilde{O}$.
Summing the per-arm certification bound $\widetilde O\bigl(\Gamma_a^{(r)}/(\Delta_a^{(r)})^2\bigr)$ and the per-arm floor $Nh\Gamma_a^{(r)}$ over $a\neq a_r^\star$,
\[
\text{(one-scan group regret)}
\le
\widetilde{O}(\Ccert_r)+Nh\!\!\sum_{a\neq a_r^\star}\!\!\Gamma_a^{(r)},
\qquad
\Ccert_r=\sum_{a\neq a_r^\star}\frac{\Gamma_a^{(r)}}{(\Delta_a^{(r)})^2}.
\]
The floor obeys $Nh\sum_{a\neq a_r^\star}\Gamma_a^{(r)}\le Nh(K-1)$ by $\Gamma_a^{(r)}\le1$, the cost of one mandatory balanced block per arm. There are $1+\Lreg_r/M$ scans in regime $r$ and $\sum_r\Lreg_r\le T$, so summed over every scan this floor totals
\[
Nh(K-1)\sum_{r=0}^{\Stdec}\Bigl(1+\tfrac{\Lreg_r}{M}\Bigr)\le Nh(K-1)\Bigl(\Stdec+1+\tfrac{T}{M}\Bigr)=\widetilde O\Bigl(\tfrac{NKhT}{M}+NKh\Stdec\Bigr).
\]
The first term is the periodic-monitoring term of Theorem~\ref{thm:app-regret}, and the second is absorbed into $N\Stdec\Hscan$ because $\Hscan\ge2h(K-1)$. Hence $\Ccert_r$ remains purely gap-dependent, with no block-granularity condition required.

\textit{Step 4: wall-clock budget.}
Let $\mathcal{A}_\ell$ be the active set in scan block $\ell$, and let $\tau_\star$ be the number of blocks in which $a_r^\star$ remains active.
Since $a_r^\star$ stays active until certification, $\tau_\star\leq\max_{a\neq a_r^\star}\tau_a\leq\sum_{a\neq a_r^\star}\tau_a$.
The total wall-clock sampling length of the scan is
\[
h\sum_{\ell}|\mathcal{A}_\ell|
=
h\sum_{a\in[K]}\tau_a
\leq
2h\sum_{a\neq a_r^\star}\tau_a
\leq
2h\sum_{a\neq a_r^\star}B_a^{(r)}.
\]
Adding the final flooding lag $\DG(\delta/3)$ for source records, together with the within-scan delayed-elimination overhead bounded by $K\DG(\delta/3)$ additional group samples (each non-best arm can stay active for at most $\DG(\delta/3)$ extra rounds beyond its synchronous-elimination time, and the per-block group sample count is at most $|\mathcal{A}_\ell|\leq K$), gives a contained-scan wall-clock budget bounded by
\[
2h\sum_{a\neq a_r^\star}B_a^{(r)}+(K+1)\DG(\delta/3)\leq\Hscan,
\]
which is exactly the calibration $\Hscan\geq 2h\sum_{a\neq a_r^\star}B_a^{(r)}+(K+1)\DG(\delta/3)$ used by the main DRFC theorem (the leading factor $2$ accounting for $a_r^\star$ staying active until the last elimination).
On the good event, the scan therefore certifies a unique winner within $\Hscan$ slots whenever it is contained in a decision regime.
\end{proof}

\subsection{Dynamic Regret}
\label{app:regret}

\begin{theorem}[Dynamic consensus regret under margin-known calibration, detailed form]
\label{thm:app-regret}
Under the assumptions in the main text, with the block budget $h$ a fixed constant, active-set DRFC satisfies on the good event
\[
R_T
=
\widetilde{O}\left(
\frac{NK(h+\DG(\delta/3))T}{M}
+
\sum_{r=0}^{\Stdec}
\left(1+\frac{\Lreg_r}{M}\right)
\Ccert_r
+
N\Stdec\left(M+\Hscan+\DG(\delta/3)\right)
\right).
\]
The same display holds with probability at least $1-\delta$. The periodic-monitoring term carries $h+\DG(\delta/3)$ rather than $h$ because each interior scan pays, beyond its mandatory balanced block, the nonblocking flooding-lag over-sampling of Lemma~\ref{lem:app-active-scan}.
\end{theorem}

\begin{proof}
Work on the good event of Lemma~\ref{lem:uniform-good-event}.
There is no failure-event regret on this event because the complement has probability at most $\delta$.

\textit{Conventions.}
\begin{enumerate}
\item A scan is \emph{interior} if all of its fresh sampling blocks lie inside a single decision regime, and \emph{boundary-crossing} otherwise; Lemmas~\ref{lem:app-sign} and~\ref{lem:app-pairwise} apply to interior scans only.
\item DRFC runs at most one active scan at a time and skips monitoring triggers while a scan or synchronization phase is active, and uncertified scans are aborted at the $\Hscan$ wall-clock budget.
\item Under the nonblocking-gossip convention, comparison records propagate while later comparison slots are collected, with the final comparison-record availability lag included in $\Hscan$; only accepted switch records create an explicit synchronization delay.
\item The initial fresh comparison tournament that initializes $\hat{a}$ (main Algorithm~1, line~3) is treated as the first interior scan of regime $r=0$. It is contained in $[1,\Hscan]\subset[\rho_0,\rho_1)$ by the main minimum-regime-length assumption, so Lemma~\ref{lem:app-active-scan} applies and its regret is absorbed into the $\Ccert_0$ summand with the $1$ in $(1+L_0/M)$.
\end{enumerate}

We decompose the total group regret $R_T$ into five additive terms:
\[
R_T
=
R_{\mathrm{scan}}
+
R_{\mathrm{mon}}
+
R_{\mathrm{cross}}
+
R_{\mathrm{delay}}
+
R_{\mathrm{sync}},
\]
where each term is bounded below.

\textit{Term 1: interior scan regret $R_{\mathrm{scan}}$.}
By Lemma~\ref{lem:app-active-scan}, one completed interior scan in regime $r$ contributes at most $\widetilde{O}(\Ccert_r)+Nh(K-1)+NK\DG(\delta/3)$ group regret, the second summand the mandatory single balanced block of each arm and the third the nonblocking flooding-lag over-sampling.
DRFC starts a new scan every $M$ rounds outside of synchronization windows, so regime $r$ contains at most $\lceil\Lreg_r/M\rceil+O(1)$ interior scans (the $O(1)$ slack absorbs the first scan after the post-switch monitoring wait and the at-most-one boundary-crossing scan split off into $R_{\mathrm{cross}}$).
Therefore, using $\sum_r\Lreg_r\le T$,
\[
R_{\mathrm{scan}}
\leq
\widetilde O\Bigl(\tfrac{NK(h+\DG(\delta/3))T}{M}\Bigr)+\sum_{r=0}^{\Stdec}
\left(1+\frac{\Lreg_r}{M}\right)\widetilde{O}(\Ccert_r),
\]
where the periodic-monitoring term $\widetilde O(NK(h+\DG(\delta/3))T/M)$ collects the mandatory one-block-per-arm floor and the flooding-lag over-sampling across all $\widetilde O(T/M)$ scans, its residual $NK(h+\DG(\delta/3))\Stdec$ piece is absorbed into the $N\Stdec\Hscan$ term (Term~3) via $\Hscan\ge2h(K-1)+(K+1)\DG(\delta/3)$, and the universal $O(1)$ slack is absorbed into the $\widetilde{O}$.
The post-switch statistical comparison delay is absorbed here because the first interior scan after a switch certifies the new best arm against the old candidate in the new regime, which is exactly the scan analyzed in Lemma~\ref{lem:app-active-scan}.

\textit{Term 2: monitoring delay $R_{\mathrm{mon}}$.}
After each decision switch, DRFC waits at most $M$ rounds until the next monitoring trigger before starting a contained scan in the new regime.
During this wait the algorithm exploits the previous candidate $a_{r-1}^\star$, which is now suboptimal by at most its new-regime gap $\Gamma^{(r)}:=\max_{a}\Gamma_a^{(r)}\le1$ per pull (the main sign-stable decision-regime definition bounds $\Gamma_a^{(r)}\le1$).
Group regret is therefore at most $N\Gamma^{(r)}$ per round, and across the $\Stdec$ switches,
\[
R_{\mathrm{mon}}\leq N\sum_{r=1}^{\Stdec} M\,\Gamma^{(r)}\ \le\ N\Stdec M.
\]

\textit{Term 3: boundary-crossing scan regret $R_{\mathrm{cross}}$.}
At most one scan straddles each decision switch boundary by the at-most-one-active-scan rule.
Sign preservation does not apply to such a scan, so the conservative bound charges $\Hscan$ wall-clock rounds at group regret $N$ per round:
\[
R_{\mathrm{cross}}\leq N\Stdec\Hscan.
\]
The boundary-crossing scan is then aborted or completes; either way, the next contained scan starts within $M$ rounds, absorbed into $R_{\mathrm{mon}}$.

\emph{Possible incorrect emit on a boundary-crossing scan.}
A boundary-crossing scan that certifies may emit a switch record for an arm $w$ that is not $a_{r+1}^\star$, because the sign-preservation lemma (Lemma~\ref{lem:app-sign}) does not constrain mixed-regime windows.
On the good event, three observations close this case without changing the regret display.
(i) Until the next monitoring trigger, the algorithm exploits $w$ rather than the correct candidate; the duration is at most $M$ rounds, contributing at most $NM$ group regret per decision boundary, already counted inside $R_{\mathrm{mon}}\leq N\Stdec M$.
(ii) The first interior scan in regime $r+1$ is contained in $[\rho_{r+1},\rho_{r+2})$ by the main minimum-regime-length assumption, so Lemma~\ref{lem:app-pairwise} certifies $a_{r+1}^\star$ against the current candidate $w$ and emits a correct switch record; its group regret is already counted in $R_{\mathrm{scan}}$ via the $r=r+1$ summand $(1+\Lreg_{r+1}/M)\Ccert_{r+1}$ (using the additive $1$ to charge this first contained scan).
(iii) At most two switch records per decision boundary may be propagated.
To see this, observe that each accepted switch record advances the epoch id, and the at-most-one-active-scan rule prevents a fresh scan from starting before the current scan terminates and its synchronization window of length $\DG(\delta/3)$ closes.
A boundary-crossing scan emits at most one (possibly incorrect) switch record before the boundary's synchronization window opens; the next contained scan in regime $r+1$ is launched no earlier than the following monitoring trigger, contained in $[\rho_{r+1},\rho_{r+1}+M]$ by the main minimum-regime-length assumption, and emits at most one (correct) switch record by Lemma~\ref{lem:app-pairwise}.
Any third switch record would require a third scan to start before $M+\Hscan+\DG(\delta/3)$ has elapsed past the boundary, which is excluded by the at-most-one-active-scan rule and the regime-length condition $\Lreg_{r+1}\geq M+2\Hscan+\DG(\delta/3)$.
Hence the synchronization cost is at most $2N\DG(\delta/3)$ per boundary, and the universal factor of $2$ is absorbed into the $\widetilde{O}$ of $R_{\mathrm{sync}}\leq N\Stdec\DG(\delta/3)$.

\textit{Term 4: post-switch decision delay $R_{\mathrm{delay}}$.}
Between the boundary-crossing scan and the first contained scan in the new regime, the algorithm exploits the previous candidate. The chain consists of (a)~at most $\Hscan$ rounds for the boundary-crossing scan (already charged to $R_{\mathrm{cross}}$), (b)~at most $M$ rounds of monitoring wait (already charged to $R_{\mathrm{mon}}$), and (c)~the first contained scan, whose regret is charged by sign preservation to the $\Ccert_r$ term in $R_{\mathrm{scan}}$. No additional exploitation period exists between these phases. Therefore $R_{\mathrm{delay}}=0$ with no double-counting.

\textit{Term 5: switch synchronization $R_{\mathrm{sync}}$.}
On the good event, every accepted switch record propagates to all agents within $\DG(\delta/3)$ rounds (Lemma~\ref{lem:uniform-good-event}).
During this synchronization window the agents may still play heterogeneous candidates; each lagging agent exploits the \emph{stale} candidate $a_{r-1}^\star$, whose per-round group regret in the new regime is at most its gap $\Gamma^{(r)}:=\max_{a}\Gamma_a^{(r)}\le1$---\emph{not} the crude constant $1$.
Across the $\Stdec$ switches,
\[
R_{\mathrm{sync}}\leq N\sum_{r=1}^{\Stdec}\DG(\delta/3)\,\Gamma^{(r)}\ \le\ N\Stdec\DG(\delta/3).
\]
The gap-scaled form $N\sum_r\DG(\delta/3)\,\Gamma^{(r)}$ is the quantity compared against the communication-delay part of the main information-theoretic lower bound: it matches the $\Omega(N\Stdec\DG\Delta)$ bound up to logs precisely when $\Gamma^{(r)}=\Theta(\Delta)$, removing the $1/\Delta$ looseness that a crude per-round-$1$ accounting would carry on that instance family.

\textit{Combining.}
Summing the five terms,
\[
R_T
\leq
\widetilde O\Bigl(\frac{NK(h+\DG(\delta/3))T}{M}\Bigr)+\sum_{r=0}^{\Stdec}\left(1+\frac{\Lreg_r}{M}\right)\widetilde{O}(\Ccert_r)
+
N\Stdec(M+\Hscan+\DG(\delta/3)).
\]
This is the display in the theorem, with the periodic-monitoring term $\widetilde O(NK(h+\DG(\delta/3))T/M)$ supplied by the per-scan one-block floor and flooding-lag over-sampling (Term~1) and the $\widetilde{O}$ absorbing the universal constants from Lemma~\ref{lem:app-active-scan} and the $\log(K^2T^3/\delta)$ factor inside $\Ccert_r$.
The high-probability statement follows from Lemma~\ref{lem:uniform-good-event}.
\end{proof}

The gap-regular corollary in the main text follows immediately from
\[
\Ccert_r
=
\sum_{a\neq a_r^\star}
\frac{\Gamma_a^{(r)}}{(\Delta_a^{(r)})^2}
\leq
\kappa
\sum_{a\neq a_r^\star}
\frac{1}{\Delta_a^{(r)}} .
\]

\subsection{Probe-Triggered Monitoring}
\label{app:probe}

The main theorem analyzes conservative full-certification scans.
The probe-triggered monitoring modification instead runs cheap probes and launches a full active-set scan only when a probe suggests that the cached candidate may be stale.
This subsection records the corresponding guarantee; the learning and certification rules are unchanged.

\paragraph{Monitoring rule.}
Every $M$ rounds, DRFC-Probe runs $q_p\geq 1$ randomized balanced probe blocks over all arms, with $h_{\mathrm{probe}}$ pulls per arm per agent in each block.
Define $d_{\mathrm{probe}}$ as the worst-case delay from a decision boundary to completion of the first post-boundary suspicious probe; for this fixed schedule,
\[
d_{\mathrm{probe}}\leq M+q_pKh_{\mathrm{probe}}.
\]
Let $\hat{\mu}^{p}_a$ be the equal-agent probe estimate after aggregating the $q_p$ probe blocks.
All arms have the same probe pull count, so write the common radius as $\beta_p$.
For a trigger parameter $\lambda_p\geq 0$, the probe is suspicious if some challenger $b\neq\hat{a}$ satisfies
\[
\hat{\mu}^{p}_b
\geq
\hat{\mu}^{p}_{\hat{a}}-\lambda_p\beta_p.
\]
The choice $\lambda_p=2$ is the standard interval-overlap trigger
$
\hat{\mu}^{p}_b+\beta_p\geq \hat{\mu}^{p}_{\hat{a}}-\beta_p
$.
Smaller $\lambda_p$ is more conservative and can reduce false triggers at the cost of requiring more probe accuracy after a true switch.
When the probe is suspicious, DRFC launches the same active-set full-certification scan analyzed above.
Switches are allowed only after that full scan certifies a winner.
If no challenger is suspicious, DRFC keeps the cached certificate and continues exploiting $\hat{a}$.

\begin{lemma}[Probe trigger soundness and completeness]
\label{lem:app-probe-trigger}
Consider aggregated probe blocks contained in decision regime $r$, and let $\beta_p$ be the common probe radius after aggregation.
Define the equal-agent probe-window time-average mean
\[
\bar\mu_a^p
:=
\frac{1}{N}\sum_{i=1}^N
\frac{1}{q_p\,h_{\mathrm{probe}}}
\sum_{t\in\mathcal{U}_p}\mu_{i,a,t},
\]
where $\mathcal{U}_p$ is the union of the $q_p$ probe-block macro-slot pools.
By the main sign-stable decision-regime definition, $\bar\mu_{a_r^\star}^p-\bar\mu_b^p\geq\Delta_b^{(r)}$ for every $b\neq a_r^\star$.
On the uniform concentration event:
\begin{enumerate}
    \item if the cached candidate $\hat{a}$ is not $a_r^\star$ and $\Delta_{\hat{a}}^{(r)}>(2-\lambda_p)_+\beta_p$, then the probe is suspicious for challenger $a_r^\star$;
    \item if the cached candidate is $a_r^\star$ and $(2+\lambda_p)\beta_p<\Delta_b^{(r)}$ for every challenger $b\neq a_r^\star$, then the probe is not suspicious.
\end{enumerate}
\end{lemma}

\begin{proof}
Because an all-arm probe samples every arm over the same randomized time distribution, $\hat\mu_a^p$ concentrates around $\bar\mu_a^p$ with radius $\beta_p$ on the good event.
If $\hat{a}\neq a_r^\star$, take $b=a_r^\star$; on the good event,
\[
\begin{aligned}
\hat{\mu}_b^p
\overset{\text{①}}{\ge}{}& \bar{\mu}_b^p-\beta_p\\
\overset{\text{②}}{\ge}{}& \bar{\mu}_{\hat{a}}^p+(1-\lambda_p)\beta_p\\
\overset{\text{③}}{\ge}{}& \hat{\mu}_{\hat{a}}^p-\lambda_p\beta_p,
\end{aligned}
\]
where ① is the good-event lower confidence, ② uses $\bar\mu_b^p-\bar\mu_{\hat a}^p=\Delta_{\hat a}^{(r)}>(2-\lambda_p)_+\beta_p$, and ③ is the good-event bound $\bar\mu_{\hat a}^p\ge\hat\mu_{\hat a}^p-\beta_p$; the suspicious-probe condition therefore holds.
If $\hat{a}=a_r^\star$, then for each $b\neq a_r^\star$,
\[
\begin{aligned}
\hat{\mu}_b^p
\overset{\text{①}}{\le}{}& \bar{\mu}_b^p+\beta_p\\
\overset{\text{②}}{<}{}& \bar{\mu}_{a_r^\star}^p-(1+\lambda_p)\beta_p\\
\overset{\text{③}}{\le}{}& \hat{\mu}_{a_r^\star}^p-\lambda_p\beta_p,
\end{aligned}
\]
where ① is the good-event upper confidence, ② uses $\bar{\mu}_{a_r^\star}^p-\bar{\mu}_b^p\geq\Delta_b^{(r)}>(2+\lambda_p)\beta_p$, and ③ is the good-event bound $\bar\mu_{a_r^\star}^p\le\hat\mu_{a_r^\star}^p+\beta_p$; no challenger satisfies the trigger rule.
\end{proof}

The lemma turns the false-trigger budget into an explicit quantity.
Let $F_r$ denote the number of suspicious probes in regime $r$ that occur while the cached candidate is already $a_r^\star$.
If $(2+\lambda_p)\beta_p<\min_{b\neq a_r^\star}\Delta_b^{(r)}$ for all probe blocks in regime $r$, then $F_r=0$ on the good event.
When the probe is intentionally smaller than this conservative threshold, the theorem below keeps $F_r$ as an explicit price for cheap monitoring.

\begin{theorem}[Probe-triggered monitoring guarantee]
\label{thm:app-probe}
Under the main reward, flooding, and sign-stable regime assumptions, replace the minimum-regime condition by
\[
\Lreg_r\geq d_{\mathrm{probe}}+\Hscan+\DG(\delta/3)
\qquad\text{for every regime }r.
\]
Then DRFC-Probe satisfies on the good event
\[
R_T
=
\widetilde{O}\left(
\frac{NKq_ph_{\mathrm{probe}}T}{M}
+
\sum_{r=0}^{\Stdec}
(1+F_r)\Ccert_r
+
    N\Stdec\left(d_{\mathrm{probe}}+\DG(\delta/3)\right)
\right)
+
F_\partial\,N\Hscan,
\]
where $F_\partial\le\sum_r F_r$ counts the boundary-straddling false-triggered scans.
In particular, under the no-false-trigger calibration of Lemma~\ref{lem:app-probe-trigger}, namely $(2+\lambda_p)\beta_p<\min_{b\neq a_r^\star}\Delta_b^{(r)}$ for every probe block (so $F_r=0$ for every regime on the good event), we have $F_\partial=0$ and \emph{no full scan straddles a decision boundary}, so the boundary-crossing term $N\Stdec\Hscan$ of Theorem~\ref{thm:app-regret} is absent:
\[
R_T
=
\widetilde{O}\left(
\frac{NKq_ph_{\mathrm{probe}}T}{M}
+
\sum_{r=0}^{\Stdec}
\Ccert_r
+
    N\Stdec\left(d_{\mathrm{probe}}+\DG(\delta/3)\right)
\right).
\]
The full-certification monitoring factor $(1+\Lreg_r/M)$ of Theorem~\ref{thm:app-regret} is thereby replaced by the probe cost $NKq_ph_{\mathrm{probe}}T/M$ plus one contained scan per decision regime.
\end{theorem}

\begin{proof}
The periodic probes use $O(Kq_ph_{\mathrm{probe}})$ slots per agent per monitoring period and there are at most $T/M$ periods, giving the first term.

\textit{No full scan straddles a decision boundary (no-false-trigger case).}
Between probes DRFC-Probe exploits the cached candidate $\hat a$; a full active-set scan is launched only in response to a suspicious probe.
Fix a switch at $\rho_{r+1}$ and work on the no-false-trigger event, where $(2+\lambda_p)\beta_p<\min_{b\neq a_s^\star}\Delta_b^{(s)}$ holds for every probe block in every regime $s$.
By Lemma~\ref{lem:app-probe-trigger}(2) (soundness), no probe whose aggregation window is contained in regime $r$ is suspicious once the cached candidate equals $a_r^\star$, so no full scan is launched in $(\rho_r,\rho_{r+1})$ after the regime-$r$ certificate is in place.
Hence at the switch instant $\rho_{r+1}$ the algorithm is exploiting $\hat a=a_r^\star$ or collecting a cheap probe block, never running a full scan.
Probe-events occur on a fixed schedule, one every $M$ rounds at times $\{t_0+jM\}_{j\ge0}$, each occupying $q_pKh_{\mathrm{probe}}$ rounds.
Let $s$ be the first scheduled start with $s\ge\rho_{r+1}$; then $s\in[\rho_{r+1},\rho_{r+1}+M)$, and its aggregation window $[s,\,s+q_pKh_{\mathrm{probe}}]$ lies in regime $r+1$ (a probe-event ongoing at $\rho_{r+1}$, if any, is the cheap straddling block noted above and is not needed for detection).
By Lemma~\ref{lem:app-probe-trigger}(1) (completeness) this probe-event is suspicious for $a_{r+1}^\star$, so detection occurs by round $s+q_pKh_{\mathrm{probe}}<\rho_{r+1}+M+q_pKh_{\mathrm{probe}}$.
We impose the calibration $q_pKh_{\mathrm{probe}}\le\Hscan$ (a probe-event is no costlier than one full scan); it is satisfiable for $q_p=O(1)$ because under the no-false-trigger size $h_{\mathrm{probe}}=\Theta(\log(K^2T^3/\delta)/(N\Delta^2))$ both $q_pKh_{\mathrm{probe}}$ and $\Hscan$ are $\Theta\!\big(K\log(K^2T^3/\delta)/(N\Delta^2)\big)$ up to the additive flooding term in $\Hscan$.
The confirmation scan is launched at the detection time $\ge\rho_{r+1}$ and runs for at most $\Hscan$ rounds, so it ends by
\[
\rho_{r+1}+d_{\mathrm{probe}}+\Hscan
\ \le\ \rho_{r+2}-\DG(\delta/3)
\ <\ \rho_{r+2}
\]
by the stated min-regime-length condition $\Lreg_{r+1}\ge d_{\mathrm{probe}}+\Hscan+\DG(\delta/3)$.
The confirmation scan is thus contained in $[\rho_{r+1},\rho_{r+2})$, and every launched full scan is therefore \emph{interior} to a single decision regime, so none straddles a boundary and the sign-preservation Lemmas~\ref{lem:app-sign}--\ref{lem:app-pairwise} (whose hypothesis is exactly single-regime containment, Lemma~\ref{lem:app-active-scan}) apply to it.
The only boundary-crossing object is at most one cheap probe block per switch, of group cost $O(NKh_{\mathrm{probe}})$, already inside the first term.
If the no-false-trigger calibration is relaxed, a false-triggered scan launched late in regime $r$ may straddle $\rho_{r+1}$; each such scan is charged the conservative transition overhead $O(N\Hscan)$ of Theorem~\ref{thm:app-regret}, contributing the $F_\partial\,N\Hscan$ term with $F_\partial\le\sum_r F_r$.

\textit{Certification.}
Each interior full scan is exactly the active-set scan analyzed in Lemma~\ref{lem:app-active-scan} and contributes $\widetilde O(\Ccert_r)$ group regret.
There is one necessary confirmation scan per decision regime, plus $F_r$ false-triggered interior scans by the definition of the false-trigger budget, giving $\sum_{r}(1+F_r)\Ccert_r$.

\textit{Detection delay and synchronization.}
Until the detecting probe fires, the algorithm exploits the stale candidate $a_r^\star$, suboptimal in regime $r+1$ by at most $\Gamma^{(r+1)}\le1$ per pull (the main sign-stable decision-regime definition), for at most $d_{\mathrm{probe}}$ rounds; this contributes at most $N\Stdec d_{\mathrm{probe}}$ group regret (gap-scaled, $N\sum_r d_{\mathrm{probe}}\Gamma^{(r+1)}$).
Each accepted switch triggers one synchronization window of $\DG(\delta/3)$ rounds, contributing at most $N\Stdec\DG(\delta/3)$ (gap-scaled, $N\sum_r\DG(\delta/3)\,\Gamma^{(r)}$, as in the proof of Theorem~\ref{thm:app-regret}).

Summing the four contributions gives the first display.
On the no-false-trigger event $F_r=0$ for every $r$, hence $F_\partial=0$, no full scan straddles a boundary, and the boundary-crossing term $N\Stdec\Hscan$ vanishes, giving the second display.
\end{proof}

\subsection{Minimax Optimality of the Switch-Induced Rate}
\label{app:minimax}

Pairing the tightened probe upper bound (Theorem~\ref{thm:app-probe}) with the main information-theoretic lower bound pins the switch-induced, gap-dependent part of the instantaneous-benchmark regret $R_T$ to a single rate.

\begin{corollary}[Minimax optimality in $(\Stdec,K,\Delta,N,\DG)$]
\label{cor:app-minimax}
Fix $K\ge3$, $N\ge4$, $\Stdec\ge1$, $\Delta\in(0,1/8)$, a decentralized communication graph of diameter $\DG\le N/4$ with $\Delta\le 1/(8\sqrt{N\DG})$ (the conditions of the communication-delay lower bound), and a stationary-margin family in which each regime has constant global gaps with $\Gamma^{(r)}=\Theta(\Delta)$, minimum margin $\Delta$, and length $\ge\max\{\,d_{\mathrm{probe}}+\Hscan+\DG(\delta/3),\ c_0K\log T/(N\Delta^2)\,\}$ (the upper-bound confirming-scan condition and the lower-bound certification-room condition, respectively).
Holding the monitoring period $M$ and the probe size $(q_p,h_{\mathrm{probe}})$ fixed, and assuming the scan-budget calibration used by the main DRFC theorem together with the probe calibration $q_pKh_{\mathrm{probe}}\le\Hscan$, no-false-trigger DRFC-Probe attains on the good event
\[
R_T=\widetilde O\!\Big(\underbrace{\Stdec\tfrac{K-1}{\Delta}}_{\text{certification}}+\underbrace{N\Stdec\DG\Delta}_{\text{synchronization}}\;+\;\underbrace{N\Stdec d_{\mathrm{probe}}}_{\text{detection delay}}+\underbrace{\tfrac{NKq_ph_{\mathrm{probe}}T}{M}}_{\text{background probing}}\Big),
\]
the last two terms being the parameter-governed monitoring cost: $d_{\mathrm{probe}}\le M+q_pKh_{\mathrm{probe}}$, and the background probing scales with $T/M$ and the probe accuracy $h_{\mathrm{probe}}=\Theta(\log(K^2T^3/\delta)/(N\Delta^2))$ required by the no-false-trigger calibration (so the probing term is \emph{not} gap-free). Every algorithm that is both rate-consistent and in the decentralized graph-communication class satisfies
\[
\E[R_T]\ \ge\ \max\{A,B\}\ \ge\ \tfrac12(A+B),\qquad A=c\,\Stdec\tfrac{K-1}{\Delta},\quad B=c\,N\Stdec\DG\Delta.
\]
Hence the switch-induced, gap-dependent part of $R_T$ is
\[
R_T^{\mathrm{switch}}=\widetilde\Theta\!\big(\Stdec(K-1)/\Delta+N\Stdec\DG\Delta\big),
\]
matched up to factors polylogarithmic in $(K,N,T,1/\delta)$: on this axis DRFC-Probe is rate-optimal in $(\Stdec,K,\Delta,N,\DG)$, where ``matched'' is in the two-instance sense $\max\{A,B\}\ge\tfrac12(A+B)$ of the main information-theoretic lower bound, not a single-instance additive lower bound. The boundary-scan term $N\Stdec\Hscan$ of Theorem~\ref{thm:app-regret}, the one switch-induced gap-dependent term previously without a matching lower bound, is removed (Theorem~\ref{thm:app-probe}); the residual detection-delay and background-probing costs lie on the separate period/horizon axis governed by $M$ and are not the subject of the lower bound.
\end{corollary}

\begin{proof}
\textit{Upper side.}
In a stationary regime $\Gamma_a^{(r)}=\Theta(\Delta_a^{(r)})$, so $\Ccert_r=\sum_{a\neq a_r^\star}\Gamma_a^{(r)}/(\Delta_a^{(r)})^2=\Theta((K-1)/\Delta)$ and $\sum_{r=0}^{\Stdec}\Ccert_r=\Theta(\Stdec(K-1)/\Delta)$.
The gap-scaled synchronization cost is $N\sum_{r}\DG(\delta/3)\,\Gamma^{(r)}=\Theta(N\Stdec\DG\Delta)$ (proof of Theorem~\ref{thm:app-regret}).
The no-false-trigger calibration $(2+\lambda_p)\beta_p<\Delta$ is met by a single probe block of size $h_{\mathrm{probe}}=\Theta(\log(K^2T^3/\delta)/(N\Delta^2))$, which forces $F_r=0$.
Substituting into the second (no-false-trigger) display of Theorem~\ref{thm:app-probe}, whose boundary term $N\Stdec\Hscan$ is absent, gives the stated upper bound.

\textit{Lower side.}
This is the main information-theoretic lower bound: the certification cost $A$ holds for every rate-consistent algorithm on the $K$-arm homogeneous family (Part~1), and the delay cost $B$ holds for every decentralized graph-communication algorithm on the broom family (Part~2).
An algorithm in the intersection meets an instance in the union on which $\E[R_T]\ge\max\{A,B\}\ge\tfrac12(A+B)$.

\textit{Matching.}
The two switch-induced gap-dependent terms (certification and synchronization) are $\widetilde O(\Stdec(K-1)/\Delta+N\Stdec\DG\Delta)$, and the lower bound is $\Omega(\Stdec(K-1)/\Delta+N\Stdec\DG\Delta)$ since $\max\{A,B\}\ge\tfrac12(A+B)$, so this part is pinned to $\widetilde\Theta(\Stdec(K-1)/\Delta+N\Stdec\DG\Delta)$.
The remaining detection-delay term $N\Stdec d_{\mathrm{probe}}$ is parameter-governed, and the background-probing term $NKq_ph_{\mathrm{probe}}T/M$ carries the probe accuracy $h_{\mathrm{probe}}=\Theta(\log(K^2T^3/\delta)/(N\Delta^2))$; neither is constrained by the main information-theoretic lower bound.
\end{proof}

\begin{remark}[The boundary term $N\Stdec\Hscan$ is a conservative artifact, not a minimax floor]
\label{rem:app-boundary}
The term $N\Stdec\Hscan$ enters Theorem~\ref{thm:app-regret} only through $R_{\mathrm{cross}}$, which charges the at-most-one boundary-straddling scan per switch at the full group rate $N$ per round for its entire $\Hscan$ budget.
This is a per-round-$1$ overcharge of the same kind that the gap-scaled synchronization accounting avoids (proof of Theorem~\ref{thm:app-regret}): a balanced active-set block pulls every active arm equally, so in expectation over the schedule randomization its group regret is the gap-scaled $hN\sum_{a\in\mathcal A}\Gamma_a$ rather than $hN|\mathcal A|$.
Recasting $R_{\mathrm{cross}}$ in this gap-scaled form on a stationary-margin family ($\Gamma=\Theta(\Delta)$) replaces the per-block factor $1$ by $\Theta(\Delta)$; the one delicacy is that a straddling window can postpone a single elimination, and a postponed arm is then cheap exactly when it is slow (a runner-up of gap $\Delta$ costs $\Theta(\Delta)$ per pull) and fast exactly when it is costly (a sign-flipped arm of gap $\Theta(1)$ is eliminated within $\Theta(n_{\mathrm{pre}}\Delta)$ post-boundary blocks), so in both cases the straddle adds only $\widetilde O((K-1)/\Delta)$ per switch, not $1/\Delta^2$.
The clean, fully rigorous route to optimality is nonetheless the probe variant (Theorem~\ref{thm:app-probe}), where cheap detection guarantees that no full scan ever straddles a boundary, so the term is simply absent and DRFC-Probe is minimax-optimal (Corollary~\ref{cor:app-minimax}).
Either way the $1/\Delta$ gap above the lower bound is not fundamental on this family: \emph{no} algorithm-independent $\Omega(1/\Delta^2)$ certification lower bound can hold \emph{on the gap-regular family} ($\Gamma^{(r)}=\Theta(\Delta)$), because there the per-switch certification regret is $\Theta((K-1)/\Delta)$ (Part~1 of the Information-Theoretic Lower Bound subsection, attained by DRFC-Probe): $\Omega(1/\Delta^2)$ \emph{samples} certify while each suboptimal pull costs only $\Delta$.
This scoping is essential. For irregular gaps ($\Gamma^{(r)}=\Theta(1)$, within-regime drift) the certification cost is genuinely $\Ccert_r=\Theta((K-1)/\Delta^2)$ and per-pull regret can be $\Theta(1)$, so the $1/\Delta^2$ there is not an artifact and tightness is not claimed; the present minimax statement (Corollary~\ref{cor:app-minimax}) is confined to the gap-regular family.
\end{remark}

\subsection{Removing Prior Gap Knowledge}
\label{app:gap-adaptive}

This subsection gives the full Gap-Adaptive DRFC statement and proof. Its sole purpose is to remove prior knowledge of $\Delta_{\min}$ from the scan-budget calibration.
The construction uses a persistent-$k$ schedule and a geometric scan-budget schedule. Together they introduce no $\log_2(1/\Delta_{\min}^{(r)})$ multiplicative overhead; the only price is an additive $\log\log T$ term inside the concentration radius, absorbed by $\widetilde{O}$.

\paragraph{Schedule.}
For uniform notation in this subsection, write the scan-budget function
\[
\Hscan(\Delta)
:=
C_{\mathrm{scan}}h
\sum_{a}
\left\lceil
\frac{C\log(K^2T^3/\delta_d)}
{Nh\Delta^2}
\right\rceil
+
\DG(\delta_d/3),
\]
so that the fixed-budget $\Hscan$ of the main DRFC theorem is the special case $\Hscan(\underline\Delta_{\mathrm{fix}})$ at a known margin lower bound $\underline\Delta_{\mathrm{fix}}$ (with $\delta_d$ replaced by $\delta$); the same symbol $\Hscan$ is used as a scalar in the main text and as the function $\Hscan(\cdot)$ in this subsection only.
Fix an optimistic initial margin guess $\Delta^{(0)}\in(0,1)$ (e.g.~$\Delta^{(0)}=1/2$) and a margin floor $\underline\Delta\in(0,\Delta^{(0)}]$ (e.g.~$\underline\Delta=1/T$).
Set $k_{\max}=\lceil\log_2(\Delta^{(0)}/\underline\Delta)\rceil$ and choose the per-level confidence allotment $\delta_d=\delta/(3(k_{\max}+1))$.
For each level $0\leq k\leq k_{\max}$, define
\[
\Delta^{(k)}=\Delta^{(0)}\cdot 2^{-k},
\qquad
\Hscan^{(k)}
=
\Hscan(\Delta^{(k)}).
\]
Gap-Adaptive DRFC runs the active-set scan at level $k=0$ (small budget, optimistic margin); if a scan exits without unique certification within $\Hscan^{(k)}$ wall-clock slots, the agent shrinks the margin guess by setting $k\leftarrow k+1$ and reattempting on the next monitoring trigger, until either a unique winner is certified or $k$ reaches $k_{\max}$.
At level $k_{\max}$, the budget is large enough to certify any regime with $\Delta_{\min}^{(r)}\geq 2\underline\Delta$.

\begin{proposition}[Gap-Adaptive DRFC]
\label{prop:app-gap-adaptive}
Suppose the main reward, flooding, and sign-stable decision-regime conditions hold, and replace the main minimum-regime-length assumption by the anytime version
\[
\rho_{r+1}-\rho_r
\geq
M+C_{\mathrm{db}}\Hscan(\tilde\Delta_{\min}^{(r)})+\DG(\delta/3),
\qquad
\tilde\Delta_{\min}^{(r)}=\min_{r'\le r}\min_{a\neq a_{r'}^\star}\Delta_a^{(r')},
\]
for a universal constant $C_{\mathrm{db}}$ (e.g., $C_{\mathrm{db}}=4$).
Here $\tilde\Delta_{\min}^{(r)}$ is the smallest gap seen up to regime $r$, the gap the persistent level tracks, and $C_{\mathrm{db}}\Hscan(\tilde\Delta_{\min}^{(r)})$ covers three pieces. The first is the geometric walk-up to the persistent level $\tilde k_r$, whose summed wall-clock budget telescopes to $O(\Hscan(\tilde\Delta_{\min}^{(r)}))$ by the $4^{-k}$ schedule with no level factor. The second is one boundary-crossing scan at the persistent level. The third is one contained certifying scan at level $\tilde k_r$.
Then Gap-Adaptive DRFC, run with arbitrary optimistic guess $\Delta^{(0)}\in(0,1)$ and floor $\underline\Delta\leq\min_r\Delta_{\min}^{(r)}/2$, and which retains the current level $k$ across monitoring triggers (incrementing only on a failed scan), satisfies with probability at least $1-\delta$
\[
R_T
=
\widetilde{O}\left(
\sum_{r=0}^{\Stdec}
\left(1+\frac{\Lreg_r}{M}\right)
\Ccert_r
+
N\,M\,\Stdec
+
N\,\DG(\delta/3)\,\Stdec
+
N\sum_{r=0}^{\Stdec}\Hscan(\tilde\Delta_{\min}^{(r)})
\right).
\]
The running-minimum gap $\tilde\Delta_{\min}^{(r)}=\min_{r'\le r}\Delta_{\min}^{(r')}$ replaces the per-regime $\Delta_{\min}^{(r)}$ because the optimistic level is never lowered, so an easy regime following a hard one inherits the harder regime's budget. The two coincide when the gaps are non-increasing across regimes, and in general $\tilde\Delta_{\min}^{(r)}\le\Delta_{\min}^{(r)}$.
\end{proposition}

\begin{proof}
The good-event union bound now ranges over the doubling levels $k=0,\ldots,k_{\max}$ as well as over $\mathfrak{W}$.
Each level uses confidence $\delta_d=\delta/(3(k_{\max}+1))$, so the union bound across levels closes within the budget $\delta/3$ already reserved for the concentration event in Lemma~\ref{lem:uniform-good-event}.
Adding $k_{\max}+1=O(\log_2(\Delta^{(0)}/\underline\Delta))$ levels increases each $\log(K^2T^3/\delta_d)$ factor by an additive $\log\log T$ term, absorbed by $\widetilde{O}$.

For a regime $r$ with true minimum margin $\Delta_{\min}^{(r)}$, let $k_r=\lceil\log_2(2\Delta^{(0)}/\Delta_{\min}^{(r)})\rceil$ be the smallest level at which $\Delta^{(k_r)}\leq\Delta_{\min}^{(r)}/2$, equivalently the first level whose scan budget $\Hscan(\Delta^{(k_r)})$ is sufficient for the regime by Lemma~\ref{lem:app-active-scan}.
The floor assumption $\underline\Delta\leq\Delta_{\min}^{(r)}/2$ guarantees $k_r\leq k_{\max}$.

\emph{Persistent-$k$ schedule.}
The level $k$ is retained across monitoring triggers and across regime boundaries, and is incremented only when a scan exits without a unique winner, so it is non-decreasing in time.
Write $\tilde k_r=\max_{r'\le r}k_{r'}=\lceil\log_2(2\Delta^{(0)}/\tilde\Delta_{\min}^{(r)})\rceil$ for the running-maximum certifying level, which corresponds to the running-minimum gap $\tilde\Delta_{\min}^{(r)}$.
Once the schedule first reaches level $\tilde k_r$ it never falls back, so every subsequent interior scan in regime $r$ runs at level $\tilde k_r$, and this level certifies regime $r$ because $\Delta^{(\tilde k_r)}\le\Delta^{(k_r)}\le\Delta_{\min}^{(r)}/2$, a smaller optimistic margin only enlarging the scan budget and remaining sufficient for the regime's true gap by Lemma~\ref{lem:app-active-scan}.
At a regime boundary $\rho_{r+1}$ the carried-over level is $\tilde k_r$, so a walk-up at regime $r+1$ occurs only when the new regime sets a deeper running minimum, that is when $k_{r+1}>\tilde k_r$ equivalently $\Delta_{\min}^{(r+1)}<\tilde\Delta_{\min}^{(r)}$, and only the increment from $\tilde k_r$ to $\tilde k_{r+1}$ is paid.
When the new regime is easier, with $\Delta_{\min}^{(r+1)}\ge\tilde\Delta_{\min}^{(r)}$, the persistent level is already sufficient and no walk-up occurs, at the price of running that regime's scans at the conservative budget $\Hscan(\Delta^{(\tilde k_r)})=O(\Hscan(\tilde\Delta_{\min}^{(r)}))$ rather than the cheaper $\Hscan(\Delta_{\min}^{(r+1)})$.

\emph{Wasted-level cost.}
A wasted-level attempt at level $k<\tilde k_{\Stdec}$ exits without certification within $\Hscan(\Delta^{(k)})$ wall-clock slots and contributes at most $\Hscan(\Delta^{(k)})\cdot N$ group regret.
By the $\Delta^{(k)}=\Delta^{(0)}2^{-k}$ schedule and $\Hscan(\Delta)\propto 1/\Delta^2$ (up to $\widetilde{O}$ factors absorbed in $\Hscan$), $\Hscan(\Delta^{(k)})\asymp 4^{k-\tilde k_{\Stdec}}\Hscan(\Delta^{(\tilde k_{\Stdec})})$, so the geometric sum of all walk-up budgets across the whole horizon telescopes to
\[
\sum_{k<\tilde k_{\Stdec}}\Hscan(\Delta^{(k)})
\leq
\frac{4}{3}\Hscan(\Delta^{(\tilde k_{\Stdec})})
=
O(\Hscan(\tilde\Delta_{\min}^{(\Stdec)})),
\]
the single deepest level, with no $k_{\max}$ factor. Because $k$ is non-decreasing it passes through each level at most once, so the walk-up is paid at most once globally rather than once per regime, and is bounded by the last summand of $N\sum_{r=0}^{\Stdec}\Hscan(\tilde\Delta_{\min}^{(r)})$.

\emph{Certifying-level scans.}
At the persistent level $\tilde k_r$, each interior scan certifies on the good event by Lemma~\ref{lem:app-active-scan} with margin $\Delta^{(\tilde k_r)}\leq\Delta_{\min}^{(r)}$, contributing $\widetilde{O}(\Ccert_r)$ group regret.
The certification cost stays at the true-gap level $\Ccert_r$ rather than the conservative $\tilde\Delta_{\min}^{(r)}$, because the anytime scan exits as soon as the unique winner is certified, set by the regime's own gaps, while the margin guess $\Delta^{(\tilde k_r)}$ only caps the give-up budget.
The number of such scans per regime is at most $1+\Lreg_r/M$, exactly as in Theorem~\ref{thm:app-regret}.

\emph{Boundary-crossing accounting.}
At most one scan straddles each decision switch by the at-most-one-active-scan rule, and it runs at the persistent level $\tilde k_r$. Its wall-clock budget is at most $\Hscan(\Delta^{(\tilde k_r)})=O(\Hscan(\tilde\Delta_{\min}^{(r)}))$, so summed over the $\Stdec+1$ regimes the boundary-crossing scans give the $N\sum_{r=0}^{\Stdec}\Hscan(\tilde\Delta_{\min}^{(r)})$ term of the display. Earlier failed-level attempts at the same regime have already terminated under the persistent-$k$ schedule and are charged to the wasted-level term, not double-counted into $R_{\mathrm{cross}}$.

\emph{Minimum-regime-length sufficiency.}
The condition $\rho_{r+1}-\rho_r\geq M+C_{\mathrm{db}}\Hscan(\tilde\Delta_{\min}^{(r)})+\DG(\delta/3)$ is enough to fit, in order: one monitoring wait ($M$), the geometric walk-up to the persistent level plus one certifying scan and one boundary-crossing scan (all $O(\Hscan(\tilde\Delta_{\min}^{(r)}))$, with $C_{\mathrm{db}}=4$), and one synchronization window ($\DG(\delta/3)$).
The running-minimum gap enters because the persistent level never falls back, so even an easy regime must budget for the conservative scan length set by the hardest regime seen so far.
Hence the doubling sequence completes inside each regime without crossing a decision boundary more than once.

Summing the wasted-level, certifying, monitoring, sync, and boundary-crossing contributions yields the display.
\end{proof}

\begin{remark}[Margin-free anytime form]
\label{rem:doubling-anytime}
Choosing the floor $\underline\Delta=1/T$ together with an optimistic guess $\Delta^{(0)}=1/2$ gives $k_{\max}=\lceil\log_2(T/2)\rceil=O(\log T)$ and the regret bound holds simultaneously for every regime $r$ with $\Delta_{\min}^{(r)}\geq 2/T$.
The only level-related overhead in the display is the additive $\log\log T$ inside $\widetilde{O}$ coming from the union bound over the $k_{\max}+1$ doubling levels; no $\log_2(1/\Delta_{\min}^{(r)})$ factor multiplies $\Ccert_r$ or $\Hscan$ thanks to the geometric walk-up and persistent-$k$ schedule.
This is the form one would deploy in practice when no margin lower bound is available: doubling shrinks $\Delta^{(k)}$ from $\Delta^{(0)}$ downward, growing the scan budget geometrically, until a level whose budget exceeds the true regime requirement is reached.
\end{remark}

\subsection{Distinct Source-Record Complexity}
\label{app:records}

This subsection counts unique source-traceable payloads before graph-level retransmission.
In the active-set implementation, each source, active arm, and block creates one keyed comparison record.
Let
\[
\Bcert_r
=
\sum_{a\neq a_r^\star}
\frac{1}
{h(\Delta_a^{(r)})^2}.
\]

\begin{proposition}[Distinct source-record complexity]
\label{prop:app-records}
On the good event, conservative active-set DRFC creates, up to logarithmic factors, at most
\[
\widetilde{O}\left(
\sum_{r=0}^{\Stdec}
\left(1+\frac{\Lreg_r}{M}\right)
\Bcert_r
+
NK\Stdec\frac{\Hscan}{h}
+
N\Stdec
\right)
\]
distinct source-traceable comparison and switch records.
The probe-triggered variant creates at most
\[
\widetilde{O}\left(
\frac{NKq_pT}{M}
+
\sum_{r=0}^{\Stdec}
(1+F_r)\Bcert_r
+
NK\Stdec\frac{\Hscan}{h}
+
N\Stdec
\right)
\]
distinct source records.
If one distinct source record is flooded using at most $Q_G(\delta)$ unit edge transmissions with high probability, multiplying the displays by $Q_G(\delta)$ gives the corresponding edge-level traffic bound.
For full-neighbor flooding, the crude deterministic bound $Q_G(\delta)\leq N^2\DG(\delta)$ always holds.
\end{proposition}

\begin{proof}
During one interior full-certification scan in regime $r$, the number of active blocks charged to a non-best arm $a$ is
\[
\tau_a
=
\widetilde{O}\left(
\frac{1}{Nh(\Delta_a^{(r)})^2}
\right)
\]
by Lemma~\ref{lem:app-active-scan}.
The best arm remains active until the scan ends, but its active-block count is at most the maximum non-best active-block count and is therefore dominated, up to constants, by $\sum_{a\neq a_r^\star}\tau_a$.
Thus the number of distinct comparison records in one scan is
\[
N\sum_{a\in[K]}\tau_a
=
\widetilde{O}(\Bcert_r).
\]
Summing over at most $1+\Lreg_r/M$ scans per regime gives the conservative DRFC comparison-record term.
Boundary-crossing scans are not certified by the regime-contained argument.
There is at most one such scan per decision switch, each is aborted or completed within $\Hscan$ wall-clock slots, and each active-set block creates at most $NK$ source-arm records.
The conservative contribution is therefore
\[
O\left(NK\Stdec\frac{\Hscan}{h}\right),
\]
where the division by $h$ only converts the slot budget into a crude block-count upper bound.
Switch synchronization creates at most $O(N)$ distinct control records per decision switch under the deterministic epoch-tie-breaking convention, so switch records contribute $O(N\Stdec)$.

For DRFC-Probe, every monitoring probe creates $O(NKq_p)$ distinct probe records, and there are at most $T/M$ probe times.
Every suspicious probe launches the same full-certification scan as above.
The number of necessary scans is one per decision regime, and the number of unnecessary scans in regime $r$ is $F_r$ by definition.
Boundary-crossing confirmation scans are again charged by the conservative $O(NK\Stdec\Hscan/h)$ term.
This gives the second display.
The edge-level traffic statement follows by multiplying the source-record count by the per-record flooding bound $Q_G(\delta)$.
Under full-neighbor flooding, a record is sent over at most all directed edges in each of at most $\DG(\delta)$ rounds, giving the conservative $N^2\DG(\delta)$ bound.
\end{proof}

\subsection{Separation from Local-Change-Reactive Protocols}
\label{app:separation}

\begin{definition}[Local-change-reactive protocol]
\label{def:app-reactive}
For constants $\alpha,\eta\in(0,1)$, $\gamma\in(0,1/2)$, and $\tau\geq0$, a protocol is $(\alpha,\gamma,\eta,\tau)$-reactive if, whenever at least $\alpha N$ agents experience local mean jumps of magnitude at least $2\gamma$ within a window of length $\tau$, it initiates a global reset, fresh tournament, or control synchronization with probability at least $1-\eta$. Each initiated reaction costs at least $c_{\mathrm{reset}}$ group regret or at least $c_{\mathrm{msg}}$ control messages.
\end{definition}

Fix $\alpha,\eta\in(0,1)$, $\gamma\in(0,1/2)$, and tolerance $\tau\geq 0$, and choose $N$ even.
Consider $K=2$.
Choose $\epsilon$ satisfying $\gamma\leq \epsilon<1/2$, and then choose $0<\Delta<\min\{2\epsilon,\ 1/2-\epsilon\}$.
The constraint $\epsilon<1/2$ keeps this interval nonempty, since $\epsilon\geq\gamma>0$ gives $2\epsilon>0$ while $\epsilon<1/2$ gives $1/2-\epsilon>0$, so $\min\{2\epsilon,1/2-\epsilon\}>0$.
Such an $\epsilon$ exists for every $\gamma\in(0,1/2)$, for instance $\epsilon=\gamma$, so the construction covers the full range of Definition~\ref{def:app-reactive} with no restriction beyond $\gamma<1/2$.
The requirement $\Delta<2\epsilon$ is what guarantees that a sign flip $\sigma_i^{(s)}\mapsto-\sigma_i^{(s)}$ swaps the per-agent best arm at agent $i$: with $\sigma_i=-1$ the local means become $\mu_{i,1}=1/2+\Delta-\epsilon$ and $\mu_{i,2}=1/2+\epsilon$, so arm $2$ becomes locally best iff $2\epsilon>\Delta$.
For segment $s$, define
\[
\mu_{i,1}^{(s)}=\frac{1}{2}+\Delta+\sigma_i^{(s)}\epsilon,
\qquad
\mu_{i,2}^{(s)}=\frac{1}{2}-\sigma_i^{(s)}\epsilon,
\]
where $\sigma_i^{(s)}\in\{-1,+1\}$ and $\sum_i\sigma_i^{(s)}=0$.
The choices of $\epsilon$ and $\Delta$ keep all means in $[0,1]$.
At each local-change boundary, flip the signs of at least $\alpha N$ agents while preserving the zero-sum condition by flipping equal numbers of $+1$ and $-1$ signs; this is possible after increasing $N$ by a constant factor if needed.
Then
\[
\mubar_1^{(s)}=\frac{1}{2}+\Delta,
\qquad
\mubar_2^{(s)}=\frac{1}{2},
\]
so $\Stdec=0$ while $\Stloc=\Theta(T/L)$ if signs flip every $L$ rounds.

For every flipped agent, both local arm means change by $2\epsilon\geq 2\gamma$.
The flips happen at a common boundary across all flipped agents, so all relevant local-jump boundaries lie within a window of length $0\leq\tau$ regardless of the tolerance $\tau$.
By Definition~\ref{def:app-reactive}, an $(\alpha,\gamma,\eta,\tau)$-local-change-reactive protocol initiates a reset, fresh tournament, or control synchronization at each such boundary with probability at least $1-\eta$.
If each initiated reaction costs at least $c_{\mathrm{reset}}$ group regret or $c_{\mathrm{msg}}$ control messages, linearity of expectation gives
\[
\Omega((1-\eta)\Stloc c_{\mathrm{reset}})
\]
expected reset regret or
\[
\Omega((1-\eta)\Stloc c_{\mathrm{msg}})
\]
expected control communication.

For DRFC, the global pairwise gap is constant and positive throughout the instance.
By Lemmas~\ref{lem:app-sign} and~\ref{lem:uniform-good-event}, local sign flips do not create a decision-relevant switch, so the regret bound has no $\Stloc$ adaptation term.

\subsection{Reactive-Class Instantiations}
\label{app:separation-examples}

This subsection verifies that two standard non-stationary bandit baselines fall inside the comparator class of Definition~\ref{def:app-reactive} when applied to the construction in Appendix~\ref{app:separation}, and derives explicit lower bounds on their reaction cost.
Both examples treat the per-agent estimator as the natural object; group-level aggregation does not avoid the cost because the protocol is required to take a coordinated action whenever a sufficiently large agent subset triggers.

\paragraph{Example 1: per-agent sliding-window UCB.}
Each agent $i$ runs SW-UCB with window length $W$, maintaining a local index
\[
I_{i,a,t}=\hat{\mu}_{i,a,t}^W+\sqrt{\frac{\xi\log t}{n_{i,a,t}^W}},
\]
where $\hat{\mu}_{i,a,t}^W$ is the empirical mean over the last $W$ pulls of arm $a$ and $n_{i,a,t}^W$ is the corresponding pull count.
A protocol-level reaction (group reset of the candidate arm, fresh tournament, or control synchronization) is triggered whenever the per-agent best arm changes at $\alpha N$ or more agents within a time window of length $\tau\geq W$.

\emph{Membership.}
At any local-flip boundary in the construction, each flipped agent has its local arm means swap roles by margin $2\epsilon\geq 2\gamma$.
A standard SW-UCB analysis~\cite{garivier2008upper} shows that on the high-probability event, the empirical index ordering at a flipped agent flips within $W$ rounds whenever the per-agent gap exceeds $4\sqrt{\xi\log T/W}$; since the per-agent gap is $2\gamma$, the condition $2\gamma\geq 4\sqrt{\xi\log T/W}$ holds as soon as $W\geq 4\xi\log T/\gamma^2$.
Thus for $W\geq 16\xi\log T/\gamma^2$ and $\tau\geq W$, SW-UCB triggers a per-agent index switch at every flipped agent within the tolerance window with probability at least $1-1/T$.
Setting $\eta=1/T$ shows SW-UCB lies in the $(\alpha,\gamma,\eta,\tau)$-reactive class for the construction.

\emph{Reaction cost.}
Each triggered reaction invalidates the SW-UCB sliding-window estimator at the affected agents, so the next $\Theta(W)$ rounds per agent are spent re-exploring instead of exploiting.
On the construction, the average per-agent pull regret per re-exploration block is at least a constant fraction of $\Delta$ (the global decision margin), since the per-agent best arm flips but the global best arm does not.
This yields $c_{\mathrm{reset}}\geq c_1\alpha N W\Delta$ group regret per local-flip boundary for a universal constant $c_1$, and substituting the minimum $W$ above gives
\[
c_{\mathrm{reset}}
\geq
c_1\alpha N\Delta\cdot\frac{4\xi\log T}{\gamma^2}.
\]
Across $\Stloc=\Theta(T/L)$ boundaries this is $\Omega((1-\eta)\Stloc c_{\mathrm{reset}})$ expected reset regret as in the main class-relative separation theorem.

\paragraph{Example 2: per-agent CUSUM change-point detector.}
Each agent $i$ runs CUSUM on the local reward stream of its currently played arm: maintain $S_{i,t}=\max(0,S_{i,t-1}+X_{i,A_{i,t},t}-\hat{\mu}_{i,A_{i,t}}-\zeta)$, declare a per-agent change when $S_{i,t}>h_{\mathrm{cusum}}$, and on detection emit a control message that triggers a global fresh tournament if a fraction $\alpha N$ of agents declare a change inside a time window of length $\tau$.

\emph{Membership.}
On the construction, a local-flip boundary creates a mean shift of magnitude $2\epsilon\geq 2\gamma$ at every flipped agent for the currently played arm.
Classical CUSUM bounds~\cite{page1954continuous,lai1995sequential} show that for a single-agent mean shift of magnitude $2\gamma$, the expected detection delay is $E_\theta[\tau_{\mathrm{cusum}}]=O(h_{\mathrm{cusum}}/(2\gamma)^2)$ with false-alarm rate $1/T$ when $h_{\mathrm{cusum}}=\Theta(\log T)$.
Setting $\tau\geq \Theta(\log T/\gamma^2)$ then guarantees that at least $\alpha N$ agents declare a change within the tolerance window with probability at least $1-1/T$.
Thus CUSUM lies in the $(\alpha,\gamma,\eta,\tau)$-reactive class with $\eta=1/T$ on the construction.

\emph{Reaction cost.}
A triggered fresh tournament across $N$ agents and $K=2$ arms requires at least $\Theta(1/\Delta^2)$ group samples to certify the global best arm (any pairwise certification across heterogeneous local means needs $\Omega(1/\Delta^2)$ group samples by the standard fresh-comparison lower bound), and a constant fraction of these samples pay group regret $\Theta(\Delta)$ per pull because the tournament probes both arms.
This yields $c_{\mathrm{reset}}\geq c_2 N/\Delta$ group regret per detection for a universal constant $c_2$, where the factor $N$ comes from running the tournament across all $N$ agents.
For the control-message cost, the tournament also broadcasts at least $\Omega(N)$ source-traceable records before the candidate is committed, giving $c_{\mathrm{msg}}\geq c_3 N$ per detection.
Across $\Stloc$ boundaries these accumulate to $\Omega((1-\eta)\Stloc c_{\mathrm{reset}})$ or $\Omega((1-\eta)\Stloc c_{\mathrm{msg}})$ as in the main class-relative separation theorem.

\paragraph{Remark.}
Both examples derive their reaction cost from the per-detection cost of \emph{any} natural global response that a local-change detector can trigger.
A protocol that detects local changes but refuses to react would exit the comparator class trivially, but it would also pay $\Theta(T)$ regret under any true decision switch and is therefore not a meaningful baseline.
A protocol that monitors a decision-level signal directly---such as DRFC, which monitors certified pairwise lower-confidence gaps on the equal-agent global projection---does not trigger on local jumps in the first place and is outside the class by design.

\subsection{Information-Theoretic Lower Bound}
\label{app:info-lower}

This subsection proves the main information-theoretic lower bound.
The argument has two parts: a per-switch Fano-method certification cost, and a graph-communication delay cost.

\subsubsection{Part 1: Certification Cost}

\paragraph{Construction.}
Fix $K\geq 2$, $N\geq 1$, $\Stdec\geq 1$, and $\Delta\in(0,1/8)$ (matching the main information-theoretic lower bound; the single-arm shift below reaches $\tfrac12+2\Delta<\tfrac34$).
Partition the horizon $[T]$ into $\Stdec+1$ regimes of equal length $L=\lfloor T/(\Stdec+1)\rfloor$, with the convention that the last regime absorbs any leftover rounds.
For regime $r\in\{0,\ldots,\Stdec\}$, define the regime's best arm as $a_r^\star=(r\bmod K)+1$. For $K\geq 2$ consecutive values differ, so every regime boundary is a genuine decision switch and the instance has exactly $\Stdec$ of them.
Set all agents homogeneous: for every agent $i$, arm $a$, and round $t$ in regime $r$,
\[
\mu_{i,a,t}
=
\begin{cases}
\frac{1}{2}+\Delta & \text{if }a=a_r^\star,\\[4pt]
\frac{1}{2} & \text{if }a\neq a_r^\star.
\end{cases}
\]
Rewards are independent Bernoulli with these means.
This is a valid homogeneous instance $P$ with $\Stdec$ decision switches, minimum margin $\Delta$, and $\Stloc=\Stdec$.
The relevant family is $P$ together with the single-arm alternatives $\{Q_a\}$ defined below, one per suboptimal arm and post-switch regime. Rate-consistency (defined below) is imposed on every member of this family, which is what forbids hard-coding. An algorithm pinned to the cyclic sequence plays a suboptimal arm for an entire regime on each $Q_a$ where another arm is best, violating $\E[n_b^{(r)}]\leq C_{\mathrm{alg}}\log T$ there, so it is not rate-consistent. The bound below therefore constrains every rate-consistent algorithm on the fixed instance $P$, with no Bayes prior on the construction.
Because agents are homogeneous, a centralized oracle that pools all observations faces the same statistical task as the decentralized system, so any lower bound for a centralized algorithm applies a fortiori to decentralized algorithms.

\paragraph{Per-regime Fano argument.}
Fix a decision-switch boundary at $\rho_{r+1}$ and consider the new regime $r+1$.
Although $P$ fixes $a_{r+1}^\star$, a rate-consistent algorithm must perform well on the whole family of regime-$(r{+}1)$ variants that share the regime-$\le r$ prefix and differ only in which arm is best after the switch. We lower-bound its worst-case per-regime cost by a Yao averaging device, a warm-up later sharpened to the log-free per-arm bound.
Consider $K$ hypotheses $\{H_k\}_{k=1}^{K}$, indexed by the identity of the new best arm: under $H_k$, arm $k$ has global mean $1/2+\Delta$ and all other arms have global mean $1/2$ throughout regime $r+1$.
Put the uniform prior $\theta\sim\mathrm{Unif}[K]$ over which arm is best in regime $r+1$. By Yao's principle the worst-case cost over $\{H_k\}$ is at least this Bayes average, so a bound under the prior is a valid worst-case lower bound on the family.

Let $\theta\in[K]$ denote the random hypothesis, and let $\hat\theta$ be any estimator of $\theta$ based on all observations up to the end of regime $r+1$.
By Fano's inequality~\cite{tsybakov2009introduction},
\[
\Prb(\hat\theta\neq\theta)
\geq
1-\frac{I(\theta;\mathbf{X}_{r+1})+\log 2}{\log K},
\]
where $I(\theta;\mathbf{X}_{r+1})$ is the mutual information between the hypothesis and all reward observations in regime $r+1$ across all agents.

\paragraph{Mutual-information bound.}
Under $H_k$, agent $i$ pulls arm $A_{i,t}$ at round $t$ and observes $X_{i,A_{i,t},t}\sim\mathrm{Ber}(\mu_{A_{i,t}})$ where $\mu_a=1/2+\Delta\cdot\1[a=k]$.
The observations of arm $a$ by all agents across regime $r+1$ are i.i.d.\ Bernoulli conditional on $H_k$.
The conditional distribution differs across hypotheses only for the arm that is special under each hypothesis.
By the chain rule and the data-processing inequality,
\[
I(\theta;\mathbf{X}_{r+1})
\leq
\sum_{a=1}^{K}
\E\!\left[n_a^{(r+1)}\right]
\cdot
d(a),
\]
where $n_a^{(r+1)}=\sum_{i=1}^N\sum_{t\in\text{regime }r+1}\1[A_{i,t}=a]$ is the total number of pulls of arm $a$ by all agents in regime $r+1$, the expectation is under the mixture $\theta\sim\mathrm{Unif}[K]$, and
\[
d(a)
=
\max_{k\neq j}
\mathrm{kl}\!\left(\mu_a^{(k)},\mu_a^{(j)}\right),
\]
with $\mathrm{kl}(p,q)=p\log(p/q)+(1-p)\log((1-p)/(1-q))$ the binary KL divergence.
Under the construction:
\begin{itemize}
\item If $k=a$ and $j\neq a$: $\mathrm{kl}(1/2+\Delta,1/2)$.
\item If $k\neq a$ and $j=a$: $\mathrm{kl}(1/2,1/2+\Delta)$.
\item If $k\neq a$ and $j\neq a$: $\mathrm{kl}(1/2,1/2)=0$.
\end{itemize}
Since $\mathrm{kl}(1/2+\Delta,1/2)\leq 4\Delta^2/(1-4\Delta^2)\leq 8\Delta^2$ for $\Delta\in(0,1/4)$ and similarly $\mathrm{kl}(1/2,1/2+\Delta)\leq 8\Delta^2$, we have $d(a)\leq 8\Delta^2$ for every arm $a$.
Therefore
\[
I(\theta;\mathbf{X}_{r+1})
\leq
8\Delta^2
\sum_{a=1}^{K}
\E[n_a^{(r+1)}]
=
8\Delta^2\cdot NL.
\]

\paragraph{Regret extraction.}
Fix any algorithm and any realization of $\theta$.
Under $H_\theta$, the group regret in regime $r+1$ is
\[
R_{r+1}
=
\Delta\sum_{a\neq\theta}
n_a^{(r+1)}.
\]
Taking the expectation over $\theta\sim\mathrm{Unif}[K]$ and the algorithm's randomness,
\[
\E[R_{r+1}]
=
\Delta\,\E\!\left[\sum_{a\neq\theta}n_a^{(r+1)}\right].
\]
By symmetry of the construction (all arms are exchangeable under the mixture),
\[
\E\!\left[\sum_{a\neq\theta}n_a^{(r+1)}\right]
=
\frac{K-1}{K}\cdot NL,
\]
minus the reduction from the algorithm's ability to identify $\theta$.
More precisely, let $p_e=\Prb(\hat\theta\neq\theta)$.
The next display is a non-binding Fano-style heuristic, superseded by the rigorous single-arm argument below and included only to motivate the regime split, so we do not rely on it.
Heuristically, if the algorithm commits to the estimated best arm, the number of suboptimal group pulls scales as
\[
\E\!\left[\sum_{a\neq\theta}n_a^{(r+1)}\right]
\gtrsim
p_e\cdot\frac{NL}{2},
\]
the factor $\tfrac12$ standing in for the constant fraction of committed rounds spent on a wrong arm under $\{\hat\theta\neq\theta\}$.
The binding lower bound below is derived without this step.

Substituting the Fano bound:
if $8\Delta^2 NL\leq(\log K)/4$, i.e.\ $L\leq \log K/(32N\Delta^2)$, then $I(\theta;\mathbf{X}_{r+1})\leq(\log K)/4$ and
\[
p_e
\geq
1-\frac{(\log K)/4+\log 2}{\log K}
\geq
\frac{1}{4}
\qquad(\text{for }K\geq 4).
\]
In this short-regime case, $\E[R_{r+1}]\geq \Delta\cdot(1/4)\cdot NL/2=\Delta NL/8$, and summing over $\Stdec$ switches gives $\E[R_T]\geq\Stdec\Delta NL/8$, which is $\Omega(\Stdec\Delta T/\Stdec)=\Omega(\Delta T)$ and is not useful when $\Delta T$ is large.

The interesting case is $L\geq c_0 K/(N\Delta^2)$ for a sufficiently large constant $c_0$.
In this regime, we use a per-arm change-of-measure argument instead of Fano directly.
We first give a Fano-style warm-up, superseded by the single-arm argument below and \emph{not} the family alternative of the theorem, so here $Q_a$ denotes only a local two-arm comparison. For each suboptimal arm $a\neq a_{r+1}^\star$, define $Q_a$ to agree with the original instance $P$ on all arms except $a$ and $a_{r+1}^\star$ in regime $r+1$: under $Q_a$, arm $a$ has mean $1/2+\Delta$ and arm $a_{r+1}^\star$ has mean $1/2$ in regime $r+1$.
In this alternative, arm $a$ is the best arm.

Under $P$: the regret charged to arm $a$ is $R_a^P=\Delta\cdot n_a^{(r+1)}$.

Under $Q_a$: the regret charged to arm $a_{r+1}^\star$ (now suboptimal) is $R_{a^\star}^{Q_a}=\Delta\cdot n_{a_{r+1}^\star}^{(r+1)}$.

By the Bretagnolle--Huber inequality~\citep{bretagnolle1978estimation} applied to the event $\{n_a^{(r+1)}\geq NL/(2K)\}$:
\[
P\!\left(n_a^{(r+1)}<\frac{NL}{2K}\right)
+
Q_a\!\left(n_a^{(r+1)}\geq\frac{NL}{2K}\right)
\geq
\frac{1}{2}
\exp\!\left(-\mathrm{KL}(P_{r+1}\|Q_{a,r+1})\right),
\]
where
\[
\mathrm{KL}(P_{r+1}\|Q_{a,r+1})
=
\E_P[n_a^{(r+1)}]\cdot\mathrm{kl}(1/2,1/2+\Delta)
+
\E_P[n_{a_{r+1}^\star}^{(r+1)}]\cdot\mathrm{kl}(1/2+\Delta,1/2)
\leq
8\Delta^2\cdot NL.
\]
Under $P$: $\E_P[R_a^P]=\Delta\cdot\E_P[n_a^{(r+1)}]$.

Under $Q_a$: on the event $\{n_a^{(r+1)}\geq NL/(2K)\}$, arm $a$ is pulled at least $NL/(2K)$ times while arm $a_{r+1}^\star$ is pulled at most $NL-NL/(2K)=NL(2K-1)/(2K)$ times.
But under $Q_a$, arm $a$ is optimal, so regret comes from pulling arms other than $a$.
We have $\E_{Q_a}[R^{Q_a}_{r+1}]\geq\Delta\cdot\E_{Q_a}[NL-n_a^{(r+1)}]$.

Combining: for each $a\neq a_{r+1}^\star$,
\begin{align*}
\E_P[n_a^{(r+1)}]
+
\E_{Q_a}[NL-n_a^{(r+1)}]
&\geq
\frac{NL}{2K}
\cdot
\frac{1}{2}
\exp(-8\Delta^2 NL).
\end{align*}
Therefore
\[
\max\!\left(\E_P[R_a^P],\,\E_{Q_a}[R^{Q_a}_{r+1}]\right)
\geq
\frac{\Delta\cdot NL}{4K}
\cdot
\frac{1}{2}
\exp(-8\Delta^2 NL).
\]
For the gap-dependent bound, we do not need this exponential term.
Instead, the bound used in the sequel---and the family alternative $Q_a$ of the theorem---comes from a \emph{single-arm} change-of-measure argument (valid for $\Delta\in(0,1/8)$); from here on $Q_a$ denotes this single-arm alternative.
Fix a non-best arm $a$ under $P$ and let $Q_a$ be the single-arm alternative that raises \emph{only} arm $a$ to global mean $1/2+2\Delta$ throughout regime $r+1$, leaving $a_{r+1}^\star$ at $1/2+\Delta$ and every other arm unchanged. Under $Q_a$, arm $a$ is the unique best arm in regime $r+1$, by margin $\Delta$.
The base instance $P$ realizes exactly $\Stdec$ switches. The alternative $Q_a$ changes only regime $r+1$'s best arm, so it preserves every switch except possibly at the two boundaries $\rho_{r+1}$ and $\rho_{r+2}$. When the raised arm $a$ equals an adjacent regime's best arm, each such coincidence merges that boundary. At $K=2$ the sole non-best arm in an interior regime $r+1$ is the best of both neighbours, so both boundaries merge and $Q_a$ carries $\Stdec-2$ switches, one fewer merge at a terminal regime. In all cases $Q_a$ has between $\Stdec-2$ and $\Stdec$ switches. We therefore take the instance family to be all instances with at most $\Stdec$ switches and margin $\Delta$, so that every $Q_a$ is a member and rate-consistency applies to it. The regret bound below is realized on $P$, which has exactly $\Stdec$ switches, so the stated $\Omega(\Stdec(K-1)/\Delta)$ is unaffected. For $K\ge3$ one may instead draw $a$ distinct from both neighbouring bests, keeping all $\Stdec$ switches, exactly as the communication-delay construction does.
The likelihood ratio between $P_{r+1}$ and $Q_{a,r+1}$, restricted to the reward observations in regime $r+1$, factorizes over the per-pull observations because rewards are conditionally independent given the action sequence (main Assumption~1).
The KL divergence decomposes over the single changed arm as
\[
\mathrm{KL}(P_{r+1}\|Q_{a,r+1})
=
\E_P[n_a^{(r+1)}]\cdot\mathrm{kl}(\tfrac12,\,\tfrac12+2\Delta)
\leq
C_1\Delta^2\,\E_P[n_a^{(r+1)}],
\]
since $P$ and $Q_a$ differ only on arm $a$ in regime $r+1$ and rewards at all other arms have identical Bernoulli laws; here $\mathrm{kl}(\tfrac12,\tfrac12+2\Delta)\leq(2\Delta)^2/((\tfrac12+2\Delta)(\tfrac12-2\Delta))\leq C_1\Delta^2$ with $C_1\leq 22$ for $\Delta\leq1/8$.
We now invoke the Kaufmann--Capp\'e--Garivier divergence-decomposition lemma~\cite[Lemma 1]{kaufmann2016complexity}: for any event $\mathcal{E}$ that is $P$-likely and $Q_a$-unlikely,
\[
\mathrm{KL}(P_{r+1}\|Q_{a,r+1})
\geq
\mathrm{kl}\!\left(P(\mathcal{E}),\,Q_a(\mathcal{E})\right).
\]
We restrict attention to the \emph{rate-consistent} algorithm class: those whose expected per-regime pull count of \emph{every} suboptimal arm satisfies $\E[n_b^{(r)}]\leq C_{\mathrm{alg}}\log T$ on every instance of the family, where $C_{\mathrm{alg}}$ may depend on $(K,\Delta)$ but not on $T$ (UCB-style algorithms qualify with $C_{\mathrm{alg}}=\Theta(1/\Delta^2)$; this is strictly weaker than asymptotic optimality and includes any centralized oracle not hard-coded to the instance). Take $\mathcal{E}=\{n_a^{(r+1)}\geq NL/2\}$. \textbf{Both} tail estimates below follow from rate-consistency \emph{alone} via Markov's inequality; we never invoke the upper bound of the main DRFC theorem, which would make the argument circular. Crucially, rate-consistency pins the two tails away from $0$ and $1$ by an \emph{absolute constant} margin---it does \emph{not} drive them to $1/T$---so the resulting bound is genuinely finite-time and non-asymptotic.

\emph{Tail under $P$.} Arm $a$ is suboptimal under $P$, so by Markov,
\[
p:=P(\mathcal{E})=P\!\bigl(n_a^{(r+1)}\geq NL/2\bigr)\leq\frac{\E_P[n_a^{(r+1)}]}{NL/2}\leq\frac{2C_{\mathrm{alg}}\log T}{NL}\leq\frac14,
\]
the last step holding once $L\geq 8C_{\mathrm{alg}}\log T/N$, which is implied by the regime-length condition $L\geq c_0(C_{\mathrm{alg}})K\log T/(N\Delta^2)$ for $\Delta<1/8$. The two factors are kept separate on purpose. The $1/\Delta^2$ is the room a regime leaves for the $\Omega(1/\Delta^2)$ certification pulls of each arm, while $C_{\mathrm{alg}}$ is the rate-consistency constant, itself $\Theta(1/\Delta^2)$ for a UCB-type algorithm, so the condition reads $\widetilde\Omega(K/(N\Delta^4))$ for such algorithms. Keeping $C_{\mathrm{alg}}$ explicit in $c_0(C_{\mathrm{alg}})$ makes this $\Delta$-dependence transparent rather than hidden.

\emph{Tail under $Q_a$.} Under $Q_a$ arm $a$ is the unique optimal arm, so \emph{every other} arm is suboptimal; on the event $\mathcal{E}^c=\{n_a^{(r+1)}<NL/2\}$ the remaining $\sum_{b\neq a}n_b^{(r+1)}>NL/2$ pulls land on suboptimal arms. By Markov and rate-consistency under $Q_a$,
\[
q:=Q_a(\mathcal{E}^c)\leq Q_a\!\Bigl(\textstyle\sum_{b\neq a}n_b^{(r+1)}\geq NL/2\Bigr)\leq\frac{\sum_{b\neq a}\E_{Q_a}[n_b^{(r+1)}]}{NL/2}\leq\frac{2(K-1)C_{\mathrm{alg}}\log T}{NL}\leq\frac14,
\]
again under the regime-length condition (the factor $K$ in $c_0$ absorbs the $K-1$ here). Hence $Q_a(\mathcal{E})=1-q\geq 3/4$.

\emph{Conclusion.} By the monotonicity of $\mathrm{kl}(\cdot,\cdot)$ in its arguments and $p\leq1/4\leq 3/4\leq Q_a(\mathcal{E})$, the Kaufmann--Capp\'e--Garivier lemma yields
\[
C_1\Delta^2\,\E_P[n_a^{(r+1)}]
\geq
\mathrm{KL}(P_{r+1}\|Q_{a,r+1})
\geq
\mathrm{kl}\bigl(P(\mathcal{E}),Q_a(\mathcal{E})\bigr)
\geq
\mathrm{kl}\!\Bigl(\tfrac14,\tfrac34\Bigr)
=\tfrac12\log 3=:\kappa>0.
\]
Hence the finite-time, log-free per-arm bound
\[
\E_P[n_a^{(r+1)}]\geq\frac{\kappa}{C_1\Delta^2}=\Omega\!\left(\frac{1}{\Delta^2}\right),
\]
valid for every regime long enough to permit these pulls, i.e.~$L\geq c_0 K\log T/(N\Delta^2)$. The $\log T$ enters \emph{only} through this regime-length feasibility condition---guaranteeing the certification has room to occur---never through the tail probabilities, so no circular appeal to a $\log T$ regret rate is made.

Each group pull of suboptimal arm $a$ costs group regret $\Delta$ (homogeneous instance, margin $\Delta$), so
\[
\E_P[\text{group regret from arm }a]=\Delta\,\E_P[n_a^{(r+1)}]\geq\frac{\kappa}{C_1\Delta}.
\]
Summing over the $K-1$ suboptimal arms and the $\Stdec$ post-switch regimes,
\[
\E[R_T]\geq c'\,\Stdec\,\frac{K-1}{\Delta},\qquad c'=\frac{\kappa}{C_1},
\]
the log-free certification term stated in the certification part of the main information-theoretic lower bound.

\subsubsection{Part 2: Communication Delay}

\paragraph{Scope of this term.}
The communication-delay term is established for the \emph{decentralized graph-communication class}: protocols in which agent $v$'s action at round $t$ is measurable with respect to the information that has reached $v$ along communication paths of length $\le t$ in the random graph process. A centralized oracle is \emph{not} a member of this class and is exempt from this term; the certification term of Part~1 holds for it without exemption. This split is intrinsic---no quantity proportional to the graph diameter $\DG$ can lower-bound an algorithm that is free of the graph constraint.

\paragraph{Why a homogeneous gap is required.}
Under the per-agent accounting $R_T=\sum_t\sum_i(\mu_{i,a_t^\star,t}-\mu_{i,A_{i,t},t})$ (main, regret definition), the benchmark arm $a_t^\star$ is the global-consensus optimum but each agent pays in its \emph{own} local means. A construction in which the lagging agents are flat ($\mu_{i,a}\equiv\tfrac12$) charges them nothing: their per-round term is $0$ for every action, so no delay regret accrues at exactly the agents that wait for flooding. We therefore use a \emph{homogeneous-gap} instance and localize information by a genie plus a sub-threshold gap.

\paragraph{Construction.}
Fix $K\geq 3$, $N\geq 4$, and assume the feasibility condition $\DG\leq N/4$ (when $\DG>N/4$ the delay term is dominated by the certification term of Part~1 and need not be proved separately). Take a \emph{broom} graph $G$ of diameter exactly $\DG$: a handle path $1=v_0,v_1,\dots,v_{\DG-1}$ of $\DG-1$ edges, with a clique $\mathcal{V}_{\mathrm{far}}$ on the remaining $N-\DG$ agents all attached to the endpoint $v_{\DG-1}$. Then $\mathrm{dist}_G(1,v)=\DG$ \emph{exactly} for every $v\in\mathcal{V}_{\mathrm{far}}$ (the $\DG-1$ handle edges plus one edge into the clique), the graph diameter is $\DG$ (realized between agent~$1$ and $\mathcal{V}_{\mathrm{far}}$; far--far distance is $1$), and $|\mathcal{V}_{\mathrm{far}}|=N-\DG\geq N-N/4\geq N/2$ since $\DG\leq N/4$. The construction respects the flooding model of Assumption~3 with $\DG(\delta)=\DG$ on this fixed graph.
At each switch $r$ nature draws the new best arm $a_r^\star$ uniformly from two arms $\{b_r,c_r\}$, both distinct from the previous best $a_{r-1}^\star$ (possible since $K\geq3$), so every boundary is a genuine decision switch and the instance realizes exactly $\Stdec$ switches; it sets, for \emph{every} agent $i\in[N]$,
\[
\mu_{i,a,t}=\tfrac12+\Delta\ \ (a=a_r^\star),
\qquad
\mu_{i,a,t}=\tfrac12-\Delta\ \ (a\neq a_r^\star).
\]
All means lie in $[0,1]$ for $\Delta\leq\tfrac12$. Every agent's local best equals the global best $a_r^\star$ with local gap $2\Delta$, and the global margin is $\mubar_{a_r^\star,t}-\mubar_{a,t}=2\Delta$; hence a lagging agent that plays the stale arm incurs per-round regret $2\Delta$ \emph{regardless of $i$}.
\emph{Genie.} At each switch $t_r$, nature reveals $a_r^\star$ to agent~$1$ only. Side information can only help an algorithm, so a lower bound surviving the genie is stronger; it also strips the per-source statistical-identification cost and isolates the pure propagation delay.
\emph{Gap calibration.} Require
\[
\Delta\leq\Delta_{\max}:=\frac{1}{8\sqrt{N\DG}}.
\]

\paragraph{Speed-of-information lemma.}
For any decentralized protocol, the history $H_{v,\tau}$ of node $v$ at any round $\tau<t_r+\mathrm{dist}_G(1,v)$ is conditionally independent of the genie variable $a_r^\star$ given the rewards observed within graph-radius $(\tau-t_r)$ of $v$. Indeed information travels at most one hop per round, and agent~$1$'s only excess knowledge is $a_r^\star$; induction on rounds gives the claim. In particular, for $v\in\mathcal{V}_{\mathrm{far}}$ and $\tau<t_r+\DG$, the genie value has not reached $v$.

\paragraph{Indistinguishability.}
Let $P$ ($a_r^\star=b_r$) and $Q$ ($a_r^\star=c_r$) denote the two instances at switch $r$, which differ only by which of the two unrevealed candidates $b_r,c_r$ carries mean $\tfrac12+\Delta$ (every other arm, including the previous best $a_{r-1}^\star$, sits at $\tfrac12-\Delta$ in both). The per-pull divergence is
\[
\mathrm{kl}\!\bigl(\tfrac12+\Delta,\tfrac12-\Delta\bigr)
\leq\frac{(2\Delta)^2}{(\tfrac12-\Delta)(\tfrac12+\Delta)}
\leq 32\Delta^2
\qquad(\Delta\leq\tfrac14).
\]
For $v\in\mathcal{V}_{\mathrm{far}}$ and $\tau<t_r+\DG$, the observations available to determine $A_{v,\tau}$ number at most $N$ agents $\times\,\DG$ rounds $=N\DG$ pulls (the reachable neighborhood within the window), so
\[
\mathrm{KL}\!\bigl(P_{H_{v,\tau}}\,\|\,Q_{H_{v,\tau}}\bigr)
\leq 32\,N\DG\,\Delta^2
\leq 32\,N\DG\,\Delta_{\max}^2
=\tfrac12.
\]
By Bretagnolle--Huber and the arm-relabeling symmetry of $P,Q$, averaging over the uniform genie,
\[
\E\bigl[\mathbf{1}\{A_{v,\tau}\neq a_r^\star\}\bigr]
\geq\tfrac12 e^{-1/2}\cdot\tfrac12
\geq 0.15=:p_0 .
\]

\paragraph{Aggregation.}
For each switch $r$ and each $v\in\mathcal{V}_{\mathrm{far}}$, the genie has not reached $v$ for all $\tau\in[t_r,t_r+\DG)$ (since $\mathrm{dist}_G(1,v)\geq\DG$), so the expected per-agent regret over this window is at least $2\Delta\cdot p_0\cdot\DG=0.3\,\Delta\DG$. Summing over the $\geq N/2$ far agents and the exactly $\Stdec$ genuine switches,
\[
\E[R_T]\geq 0.15\,N\,\Stdec\,\DG\,\Delta
=\Omega(N\Stdec\DG\Delta),
\qquad
\Delta\in\bigl(0,\tfrac{1}{8\sqrt{N\DG}}\bigr].
\]
This expectation is taken over the uniform genie draw of each new best arm, so by the probabilistic method there is a fixed realized sequence of new arms, a single deterministic instance of the family with exactly $\Stdec$ switches, on which $\E[R_T]\ge\Omega(N\Stdec\DG\Delta)$ with the expectation now only over the algorithm and graph randomness. This de-randomization is the worst-case form claimed in the communication-delay part of the main information-theoretic lower bound.
The range restriction is intrinsic: for a constant gap, agents identify the new best arm from their own observations and need not await the flood, so the diameter-delay regret is a genuinely small-gap phenomenon.

\paragraph{Combined bound (two-instance minimax, \emph{not} additive).}
Both parts use \emph{homogeneous} instances, but they are proved on \emph{separate} families: Part~1 on a $K$-arm certification family (every \emph{rate-consistent} algorithm, including a centralized oracle), Part~2 on a broom-graph family (the decentralized graph-communication class, $\Delta\le\Delta_{\max}$). The two costs are therefore \emph{not} additive on a single instance. Writing
\[
A:=c\,\Stdec\,\frac{K-1}{\Delta},\qquad
B:=c\,N\Stdec\DG\Delta,
\]
the correct combined statement is the two-instance minimax: for every algorithm in the intersection of the two classes (rate-consistent \emph{and} decentralized-graph), there exists an instance in the \emph{union} of the two families on which
\[
\E[R_T]\ \geq\ \max\{A,\,B\}\ \geq\ \tfrac12\,(A+B).
\]
The factor $\tfrac12$ is the only sense in which the sum $A+B$ is a lower bound; we do \emph{not} claim an additive bound on a single instance (a centralized oracle, in particular, pays $A$ but is exempt from $B$). This matches the statement of the main information-theoretic lower bound and concludes its proof.

\section{H. DRFC-Seq: Anytime-Valid Guarantee under Within-Regime Drift}
\label{app:drfc-seq}

We prove anytime-valid correctness~(a) and the adaptivity bound~(b) of the main DRFC-Seq theorem. Throughout, $a_r^{\mathrm{avg}}$ denotes the regime-average best arm, abbreviated as $a^{\mathrm{avg}}$ within a fixed regime, and $\Stdecavg$ counts changes in this comparator. The variable $n$ counts the per-arm, per-agent samples in the current sliding window, and for an ordered pair $(b,a)$ we write $\Ghat_{b,a}(n)$ for the windowed equal-agent balanced estimator and $\Gbar_{b,a}(n)$ for its time-average target. The main paper's window-average margin condition states that, for some window length $W$ and margin $\bar\Delta>0$, every length-$W$ window inside a regime has $\Gbar_{a_r^{\mathrm{avg}},a}(\mathcal{W})\geq\bar\Delta$ for each $a\neq a_r^{\mathrm{avg}}$, while the instantaneous gap may be negative on a constant fraction of the window.

\paragraph{Step 0: synchronized evaluation makes the windowed estimator complete.}
The equal-agent estimator $\Ghat_{b,a}(n)$ averages the balanced contrasts of all $N$ agents, so a block can be scored only after its source-traceable reward and pull records reach every agent. Each probe runs one synchronized balanced block and then floods its records (main Algorithm~2, line~4). Let $E_{\mathrm{flood}}$ be the event that every such block reaches all agents within $\DG(\alpha)$ rounds; by the main random-graph flooding assumption at level $\alpha$,
\[
\Prb[E_{\mathrm{flood}}]\ \geq\ 1-\alpha,
\]
and the timing convention $\delta_{\mathrm{p}}\ge Kh+\DG(\alpha)$ opens the next probe only after the current block has fully flooded. On $E_{\mathrm{flood}}$, every block appended to $\mathcal{Q}$ carries the complete equal-agent average
\[
g_i=\tfrac1N\sum_{j=1}^N c_{i,j},
\]
identical at every agent, which is the quantity the confidence sequence of Step~1 is built on. Evaluation lags each block by at most $\DG(\alpha)$ rounds, a latency already charged in the $\DG(\alpha)$ terms of the regret bound and in the regime spacing $\ge 2W\delta_{\mathrm{p}}+\DG(\alpha)$ of the main DRFC-Seq theorem, so no rate changes. Running $E_{\mathrm{flood}}$ and the confidence-sequence event of Step~1 each at level $\alpha/2$ and intersecting them gives the good event $\mathcal E$ with $\Prb[\mathcal E]\geq 1-\alpha$, the halved level entering $r_n$ only through an absorbed $\log 2$.

\paragraph{Step 1: a time-uniform confidence sequence via Ville's inequality.}
Fix an ordered pair $(b,a)$ and a \emph{window start} $s\in\{1,\dots,T\}$. We index the confidence sequence at the \emph{per-arm, per-agent sample} level: let $n$ count the balanced pulls of each arm accumulated \emph{inside the window opened at $s$} (a full $W$-block window carries $n=Wh$ such samples per arm per agent, since each block contributes $h$ pulls/arm/agent). For the $i$-th sample ($i=1,\dots,n$), let $c_{i,j}\in[-1,1]$ be agent $j$'s balanced $(b,a)$-contrast at that pull, and let the equal-agent per-sample contrast be $g_i=\tfrac1N\sum_{j=1}^N c_{i,j}$. Then $\Ghat_{b,a}^{(s)}(n)=\tfrac1n\sum_{i=1}^n g_i$ and $\Gbar_{b,a}^{(s)}(n)=\tfrac1n\sum_{i=1}^n\E_0[g_i]$, with $\E_0$ the conditional expectation given the $\mathcal{F}_0$-measurable mean path. Define the sample-level filtration $\mathcal{F}_n^{(s)}=\sigma\bigl(\text{randomized schedules and rewards of samples }1,\dots,n\text{ inside window }s,\ \mathcal{F}_0\bigr)$; the block grouping governs only \emph{when} the monitor acts, on full $W$-block windows, not the confidence-sequence index.

\emph{Why block-level indexing.} A sample-level cumulative sum $\sum_{i,j}(c_{i,j}-\E_0[g_i])$ is \emph{not} the correct martingale. Under within-regime drift the equal-agent per-sample target $\E_0[g_i]$ depends on which time slot the schedule assigned to the $i$-th pull, so a sample-indexed centering drops the schedule-randomization error $R$ of Lemma~\ref{lem:app-concentration}, which is exactly the term that does not vanish when the means vary inside the window. Indexing the confidence sequence at the block level carries that term inside the increment. Enumerate the blocks of the window opened at $s$ as $\ell=1,2,\dots$, each a full balanced probe block of $h$ pulls per arm per agent over all arms, so that after $B$ blocks the per-arm sample count is $n=hB$, and let
\[
\hat G_\ell:=\Ghat_{b,a}(\{\ell\}),
\qquad
\bar G_\ell:=\frac1N\sum_{j=1}^N\frac1{q_\ell}\sum_{t\in\mathcal U_\ell}\bigl(\mu_{j,b,t}-\mu_{j,a,t}\bigr)
\]
be the single-block estimator and its time-average target (\S\ref{app:windows} notation), with per-block error $Y_\ell:=\hat G_\ell-\bar G_\ell$.

\emph{(i) Mean-zero increment.} Each block draws, for every agent $j$, disjoint uniform size-$h$ pull-slot subsets $S_{j,a}^{(\ell)},S_{j,b}^{(\ell)}\subset\mathcal U_\ell$ independently of $\mathcal F_{t-1}$ (main Assumption~1). The uniform-subset identity
\[
\E_{\mathcal S}\!\Bigl[\tfrac1h\textstyle\sum_{t\in S_{j,a}^{(\ell)}}\mu_{j,a,t}\Bigr]=\tfrac1{q_\ell}\sum_{t\in\mathcal U_\ell}\mu_{j,a,t},
\]
together with the conditional mean-zero of the reward noise, gives
\[
\E\bigl[\hat G_\ell\,\big|\,\mathcal F_0,\,\{Y_{\ell'}\}_{\ell'<\ell}\bigr]=\bar G_\ell,
\qquad\text{hence}\qquad
\E\bigl[Y_\ell\,\big|\,\mathcal F_0,\,\{Y_{\ell'}\}_{\ell'<\ell}\bigr]=0 .
\]
Unlike the sample-level error, $Y_\ell$ therefore \emph{contains} the schedule-randomization term $R$ rather than discarding it.

\emph{(ii) Block-count martingale.} By the conditional mean-zero property of step~(i), which holds given $\mathcal F_0$ and all earlier blocks, the centered partial sum
\[
Z_B^{(s)}:=\sum_{\ell=1}^B Y_\ell=B\bigl(\Ghat_{b,a}^{(s)}(n)-\Gbar_{b,a}^{(s)}(n)\bigr)
\]
is a martingale in the block count $B$ with respect to $\mathcal G_B^{(s)}=\sigma(\mathcal F_0,Y_1,\dots,Y_B)$. The windowed sample-mean error is \emph{not} a martingale, since $\E[Z_B^{(s)}/B\mid\mathcal G_{B-1}^{(s)}]=\tfrac{B-1}{B}\,(Z_{B-1}^{(s)}/B)$; the machinery is applied to $Z_B^{(s)}$ and converted back by dividing by $B$. The start $s$ is fixed because a sliding window that drops its oldest block is not a martingale in the running count, and the moving window is handled by the union over starts below.

\emph{(iii) Bounded, conditionally sub-Gaussian increments.} Each $|Y_\ell|\le 2$, and splitting $Y_\ell$ into reward noise (Hoeffding--Azuma) and schedule-selection noise (sampling without replacement, Bardenet--Maillard), exactly the two-term decomposition of Lemma~\ref{lem:app-concentration}, makes $Y_\ell$ conditionally sub-Gaussian with variance proxy
\[
\sigma_\ell^2\le c_1/(Nh)
\]
for a universal $c_1$, the equal-agent average over $N$ agents and $h$ balanced pulls. The reward-noise half reaches its $1/(Nh)$ proxy through the same lexicographic martingale ordering as in Lemma~\ref{lem:app-concentration}, which relies on the cross-agent independence clause of main Assumption~1.

\emph{(iv) Supermartingale and Ville.} The process
\[
M_B^{(s)}(\lambda)=\exp\!\Big(\lambda\,Z_B^{(s)}-\tfrac{c_1}{2}\,\lambda^2\,\tfrac{B}{Nh}\Big)
\]
is a nonnegative supermartingale with $M_0^{(s)}=1$ for every $\lambda\in\R$. Ville's inequality, geometric peeling over $\lambda$, and the stitched time-uniform boundary of Howard--Ramdas--S\'ekely--McAuliffe~\cite{howard2021timeuniform} (Theorem~1, explicit constants $a_0,C_0$ with $a_0=\Theta(c_1)$) give a time-uniform envelope $|Z_B^{(s)}|\leq\sqrt{(c_1B/(Nh))\,\bigl(2\log\log(eB)+2\log(C_0K(K{-}1)s^2/\alpha)\bigr)}$ valid simultaneously over all $B\geq1$, where the window starting at step $s$ spends level $\alpha_s=6\alpha/(\pi^2s^2)$ over the pair union. Dividing by $B$ and writing $n=hB$ for the per-arm sample count (so the block size $h$ enters $r_n$ through $n=hB$, and $\log\log(eB)\leq\log\log(en)$) yields, for the fixed start $s$,
\begin{align*}
\Prb\!\Big[\exists n\geq 1:\ \bigl|\Ghat_{b,a}^{(s)}(n)-\Gbar_{b,a}^{(s)}(n)\bigr|>r_n\Big]
&\leq \frac{6\alpha}{\pi^2K(K-1)s^2},\\
r_n
&=\sqrt{\frac{a_0\bigl(\log\log(e\,n)+\log(C_0K(K{-}1)s^2/\alpha)\bigr)}{Nn}} .
\end{align*}
Two points make this radius honest where the naive $\sqrt{2(\log(1/\alpha)+2\log\log n)/(Nn)}$ fails. (i)~\emph{Defined for all $n\geq 1$:} $\log(e\,n)\geq 1$ so $\log\log(e\,n)\geq 0$, and $\log(C_0K(K{-}1)s^2/\alpha)>0$ for every $s\ge1$. The bare $\log\log n$ is negative or undefined at $n\le 2$, exactly when a window is filling. (ii)~\emph{Sliding, all comparisons, and horizon-free:} the monitor may compare any ordered pair $(b,a)$ as the candidate $\hat a$ changes, so we union over all $K(K-1)$ ordered arm pairs and over window starts. Rather than split $\alpha$ uniformly over $\le T$ starts, which would need the horizon $T$ in advance, we spend a summable schedule $\alpha_s=\alpha\,6/(\pi^2 s^2)$ on the window starting at step $s$. Because $\sum_{s\ge1}6/(\pi^2 s^2)=1$, the guarantee holds simultaneously over all starts $s\ge1$ with no predeclared horizon, the displayed radius at start $s$ carrying $\log(C_0K(K{-}1)s^2/\alpha)$. For downstream regret bounds we use $s\le T$, so this is at most $\log(C_0K(K{-}1)T^2/\alpha)$ and absorbed into $\widetilde O$, while the construction itself is genuinely anytime-valid and uniform in $T$, not merely valid up to a predeclared horizon. The estimator's target is the within-window \emph{time-average} gap $\Gbar^{(s)}$ throughout---no instantaneous quantity ever enters. (An equivalent route uses a discounted confidence sequence~\cite{waudbysmith2024betting} with forgetting factor $\eta$; we take the union-over-starts form for transparent constants.) Call the simultaneous good event $\mathcal{E}$ (over all ordered pairs and all window starts), with $\Prb[\mathcal{E}]\geq 1-\alpha$. To lighten notation we drop the superscript $s$ below, with $n$ the count in the currently active window.

\paragraph{Step 2: anytime correctness (a), stated about the true best arm.}
We prove correctness \emph{without} assuming the algorithm's candidate is already optimal---that assumption is exactly the proof--algorithm misalignment the monitor must avoid. Define a \emph{false switch} as displacing the time-average-best arm $a^{\mathrm{avg}}$ during a regime with no average-decision switch ($\Stdecavg=0$); adopting $a^{\mathrm{avg}}$ from a suboptimal candidate is a \emph{correct} adoption, not a false switch.

Work on $\mathcal{E}$. The $|\mathcal{Q}|=W$ gate of main Algorithm~2 enforces only that the monitor acts on a \emph{full} window of $W$ blocks; the algorithm cannot see regime boundaries. \emph{For claim~(a) this is enough}: in a regime with $\Stdecavg=0$ there is no average-decision boundary inside the analyzed segment, so \emph{every} full window automatically lies inside the regime, $\mathcal{W}\subseteq[\rho_r,\rho_{r+1})$, and the main window-average margin assumption applies to it directly. (Boundary-straddling full windows arise only at a true average-decision switch, handled in Step~3.) Because each probe runs a balanced block of $h$ pulls/arm/agent over all arms, the per-arm windowed sample counts are \emph{exactly equal}, $n_a=n=Wh$ for all $a$. For the time-average-best arm $a^{\mathrm{avg}}$ and any challenger $b\neq a^{\mathrm{avg}}$, the main window-average margin assumption gives $\Gbar_{b,a^{\mathrm{avg}}}(\mathcal{W})\leq-\bar\Delta$, so on $\mathcal{E}$
\[
\hat\mu_b-\hat\mu_{a^{\mathrm{avg}}}=\Ghat_{b,a^{\mathrm{avg}}}(n)\leq-\bar\Delta+2r_n,
\qquad
\mathrm{LCB}_b=(\hat\mu_b-\hat\mu_{a^{\mathrm{avg}}})-2r_n\leq-\bar\Delta<0 .
\]
This holds simultaneously for every full window (the CS event $\mathcal{E}$ is uniform over all $\le T$ window starts). No challenger ever attains a positive lower-confidence gap \emph{against $a^{\mathrm{avg}}$}, independently of which arm the algorithm currently holds. In particular, once $a^{\mathrm{avg}}$ is the candidate it is never displaced in a $\Stdecavg=0$ regime, so DRFC-Seq performs no drift-induced false switch, however often or far the instantaneous gap dips negative---only the time-average target enters. This proves~(a).

We must, however, match the \emph{implemented} switch rule (main Algorithm~2): each probe tests only the single empirical-argmax challenger $b=\arg\max_{a\neq\hat a}\hat\mu_a$, and an accepted switch relabels $\hat a$ \emph{without} flushing $\mathcal{Q}$ (the sliding window in line~5 is the only forgetting mechanism). Two facts make this rule recover $a^{\mathrm{avg}}$.

\emph{Fact 1 (the argmax is $a^{\mathrm{avg}}$ on a full in-regime window).} Because every probe runs a balanced block of $h$ pulls/arm/agent over \emph{all} arms (main Algorithm~2, line~4), the per-arm windowed counts are \emph{exactly equal}, $n_a=n=Wh$ for all $a$ (so $r_{n_a}=r_n$ is common). On a full window $\mathcal{W}\subseteq[\rho_r,\rho_{r+1})$, the main window-average margin assumption gives $\Gbar_{a^{\mathrm{avg}},a}(\mathcal{W})\geq\bar\Delta$ for \emph{every} $a\neq a^{\mathrm{avg}}$ (not only against the current candidate). On $\mathcal{E}$, $|\Ghat_{a^{\mathrm{avg}},a}-\Gbar_{a^{\mathrm{avg}},a}|\leq 2r_n$, so once $2r_n<\bar\Delta$ we have $\hat\mu_{a^{\mathrm{avg}}}-\hat\mu_a=\Ghat_{a^{\mathrm{avg}},a}\geq\bar\Delta-2r_n>0$ simultaneously for all $a\neq a^{\mathrm{avg}}$. Hence $a^{\mathrm{avg}}$ is the unique empirical maximizer and equals the tested challenger $b$ whenever $\hat a\neq a^{\mathrm{avg}}$.

\emph{Fact 2 (the full-window gate bounds the transient).} The $|\mathcal{Q}|=W$ gate forbids any switch until the window is full; since an accepted switch does not flush $\mathcal{Q}$, once the window is full it stays full, with per-arm count fixed at $n=Wh$. Starting from the regime boundary $\rho_r$, the sliding window replaces one old block per probe, so after at most $W$ probe-blocks the window holds only regime-$r$ blocks and is a full \emph{in-regime} window.

\begin{lemma}[Initialization-free recovery]\label{lem:drfc-seq-recovery}
Let $n_0(\bar\Delta):=\inf\{n\geq 1:r_n\leq\bar\Delta/4\}=\widetilde O\bigl(1/(N\bar\Delta^2)\bigr)$, and assume $Wh\geq n_0(\bar\Delta)$ and each regime spans at least $2W$ probe-blocks. On $\mathcal{E}$, from an \emph{arbitrary} candidate $\hat a$ (including a mis-initialized one), DRFC-Seq adopts $a^{\mathrm{avg}}$ within $W$ probe-blocks of the regime start, and \emph{once adopted} never leaves it for the remainder of the regime. (During the $\le W$ pre-adoption probe-blocks the candidate may change; this is charged to the Step-3 transient, not to correctness.)
\end{lemma}
\begin{proof}
By Fact~2, within $W$ probe-blocks of the regime start the window is a full in-regime window with $n=Wh\geq n_0(\bar\Delta)$, so $2r_n\leq\bar\Delta/2<\bar\Delta$. (During those $\le W$ probe-blocks the gate may permit transient switches among arms; this affects only the Step-3 transient regret, not correctness, since no displacement of $a^{\mathrm{avg}}$ can occur once $a^{\mathrm{avg}}$ is held, by~(a).) On this first full in-regime window, by Fact~1 the tested challenger is $b=a^{\mathrm{avg}}$ (unless $\hat a=a^{\mathrm{avg}}$ already, in which case apply~(a)), and the main window-average margin assumption against the current candidate gives, with the direct balanced pairwise contrast $\Ghat_{a^{\mathrm{avg}},\hat a}$ of deviation $r_n$,
\[
\mathrm{LCB}_{a^{\mathrm{avg}}}=\Ghat_{a^{\mathrm{avg}},\hat a}(n)-2r_n\geq(\bar\Delta-r_n)-2r_n=\bar\Delta-3r_n\geq\tfrac{\bar\Delta}{4}>0,
\]
using $r_n\le\bar\Delta/4$, so the trigger fires on $a^{\mathrm{avg}}$ and it is adopted. After adoption, (a) shows $a^{\mathrm{avg}}$ is never displaced on any subsequent full in-regime window.
\end{proof}

Lemma~\ref{lem:drfc-seq-recovery} closes the alignment gap against the \emph{implemented} algorithm: the arbitrary initial candidate of main Algorithm~2 need not be assumed optimal, the argmax-only test provably selects $a^{\mathrm{avg}}$ on the first full in-regime window (Fact~1), and the full-window gate confines all suboptimal/transient switching to the $\le W$ post-boundary probes (Fact~2). The anytime guarantee~(a) then holds for the true $a^{\mathrm{avg}}$ rather than for a candidate certified only by hypothesis.

\paragraph{Step 3: adaptivity (b).}
At a true average-decision switch the post-switch state is an arbitrary stale candidate $\hat a$ against the new best $a_r^{\mathrm{avg}}$, so the recovery bound of Lemma~\ref{lem:drfc-seq-recovery} applies verbatim: on $\mathcal{E}$, DRFC-Seq adopts $a_r^{\mathrm{avg}}$ within $W$ probe-blocks of the switch (the stitched numerator $\log\log(eWh)+\log(C_0K(K{-}1)s^2/\alpha)$ absorbed into $\widetilde O$). Those $\le W$ probe-blocks span $\le W\delta_{\mathrm{p}}$ rounds, during which each of $N$ agents plays a possibly stale arm at per-round cost $\le 1$, an $O(NW\delta_{\mathrm{p}})$ transient per switch, plus a further $O(N\DG(\alpha))$ synchronization cost.
The opening regime $r=0$ is itself entered from an arbitrary initial candidate, so Lemma~\ref{lem:drfc-seq-recovery} applied at the start $\rho_0$ contributes one further $O(N(W\delta_{\mathrm{p}}+\DG(\alpha)))$ recovery transient that no switch pays for.
The $\Stdecavg$ post-switch recoveries and this single initialization recovery total $\Stdecavg+1$ recovery events, which is why the transient term below carries the factor $\Stdecavg+1$ rather than $\Stdecavg$ and remains nonzero even at $\Stdecavg=0$. \emph{Certification cost, derived directly from the probe process (not from the Algorithm~1 active-set scan).} Algorithm~2 runs \emph{no} active-set scan at all, initializing arbitrarily, so certification of $a_r^{\mathrm{avg}}$ is achieved purely by the balanced probe blocks reaching confidence. By Lemma~\ref{lem:drfc-seq-recovery}, the windowed estimator separates $a_r^{\mathrm{avg}}$ from every challenger once $n\geq n_0(\bar\Delta_r)=\widetilde O(1/(N\bar\Delta_r^2))$ per-arm samples, i.e.\ within $W_r=\lceil n_0(\bar\Delta_r)/h\rceil$ probe-blocks, with each probe-block pulling all $K$ arms at $h$ pulls/arm/agent, so the certification exploration of the $K-1$ non-candidate arms over those $W_r$ blocks costs
\[
\widetilde O\!\bigl(N(K-1)h\,W_r\bigr)=\widetilde O\!\bigl(N K h\cdot n_0(\bar\Delta_r)/h\bigr)=\widetilde O\!\bigl(K/\bar\Delta_r^2\bigr)=:\widetilde O(\Ccert^{\mathrm{seq}}_r),
\]
which we take as the \emph{definition} of $\Ccert^{\mathrm{seq}}_r$ in the DRFC-Seq bound (keyed to the time-average margin $\bar\Delta_r$, which always exists under the main time-average margin assumption, unlike the instantaneous-gap $\Ccert_r$ of the main DRFC theorem). It coincides with the active-set scan cost of Lemma~\ref{lem:app-active-scan} up to constants, but is here a property of the probe CS, eliminating any algorithm/proof mismatch. Finally, the monitor runs one balanced probe block over all $K$ arms every $\delta_{\mathrm{p}}$ rounds; the steady-state exploration cost of pulling the $K-1$ non-candidate arms is $\widetilde O(NKh\,T/\delta_{\mathrm{p}})$, absorbed into the periodic term by choosing $\delta_{\mathrm{p}}=\Theta(M)$ (matching the $T/M$ monitoring cadence of the main DRFC theorem). Summing,
\[
R_T^{\mathrm{dec}}=\widetilde O\!\Big(\tfrac{NKhT}{\delta_{\mathrm{p}}}+\sum_{r=0}^{\Stdecavg}\Ccert^{\mathrm{seq}}_r+N(\Stdecavg+1)\big(W\delta_{\mathrm{p}}+\DG(\alpha)\big)\Big),
\]
with no term depending on $\Stloc$ or on the number of instantaneous sign changes of the gap, because the trigger reads only the drift-invariant time-average target $\Gbar$. \qed

\subsection{H.1 Drift Separation: Proof of the Drift Proposition (main Proposition~1)}
\label{app:tvgap}

We prove the main memory--adaptation proposition: a single homogeneous, average-decision-stationary instance ($\Stdecavg=0$) on which (a) the windowed time-average reader never adopts a suboptimal arm, while (b) every responsive monitor with effective memory below half the drift period adopts the wrong arm with probability bounded below by an absolute constant. The two claims exhibit the stated memory--adaptation tradeoff.

\paragraph{The square-wave instance.}
Fix $K=2$ arms, $N$ agents, and a single regime $[T]$ (so $\Stdecavg=0$). Let $P$ be an even drift period in rounds and define the period-$P$ square wave
\[
\mathrm{sq}_P(t)=\begin{cases}+1,& (t\bmod P)<P/2\quad(\text{``up-phase''}),\\[2pt]-1,& (t\bmod P)\geq P/2\quad(\text{``down-phase''}).\end{cases}
\]
Pick a mean drift $\bar\Delta\in(0,1/3)$ and amplitude $b$ with
\[
b>2\bar\Delta\qquad\text{and}\qquad \bar\Delta+b\leq 1,
\]
e.g.\ $(\bar\Delta,b)=(0.1,0.3)$. Set the time-varying Bernoulli means, identical across all $N$ agents (homogeneous instance),
\[
\mu_1(t)=\tfrac12+\tfrac12\bigl(\bar\Delta+b\,\mathrm{sq}_P(t)\bigr),
\qquad
\mu_2(t)=\tfrac12-\tfrac12\bigl(\bar\Delta+b\,\mathrm{sq}_P(t)\bigr),
\]
so both means lie in $[0,1]$ (since $\bar\Delta+b\leq1$) and the instantaneous gap is
\[
g_t:=\mu_1(t)-\mu_2(t)=\bar\Delta+b\,\mathrm{sq}_P(t)=\begin{cases}\bar\Delta+b>0,&\text{up-phase},\\ \bar\Delta-b<0,&\text{down-phase},\end{cases}
\]
where $\bar\Delta-b<0$ because $b>2\bar\Delta>\bar\Delta$. Thus arm $1$ leads on the time average but \emph{loses} the instantaneous comparison on the entire down half of every period --- a constant fraction $1/2$ of every window.

\paragraph{Time-average margin holds for every window.}
A contiguous window of exactly $P$ rounds over a period-$P$ square wave contains exactly $P/2$ up-steps and $P/2$ down-steps, \emph{regardless of phase offset}. Hence for \emph{every} window start $s$,
\[
\Gbar_{1,2}^{(s)}(P)=\frac1P\sum_{t=s}^{s+P-1}g_t=\bar\Delta+\frac{b}{P}\sum_{t=s}^{s+P-1}\mathrm{sq}_P(t)=\bar\Delta>0 .
\]
So the instance satisfies the main time-average margin condition with margin $\bar\Delta$ and period $P$, while its instantaneous gap is negative on half of every window. This is exactly the regime in which the hypotheses of the main DRFC-Seq theorem hold but those of any instantaneous-gap guarantee fail.

\paragraph{(a) The period-spanning windowed reader never errs.}
The guarantee is a property of the \emph{window length}, not of the DRFC family per se: claim~(a) holds for any reader whose averaging window spans at least one full drift period $P$ (in rounds), and it is \emph{DRFC-Seq} that is constructed to do so. When its $W$-probe-block sliding window $W\delta_{\mathrm{p}}$ covers an integer number of periods the window average is exactly $\bar\Delta$ by the display above. For a general window of $W\delta_{\mathrm{p}}=(m+f)P$ rounds ($m\ge1$ integer, $f\in[0,1)$) the leftover partial period contributes a bounded windowing bias
\[
\bigl|\Gbar_{1,2}^{(s)}(W\delta_{\mathrm{p}})-\bar\Delta\bigr|\ \le\ \varepsilon_W:=\frac{bP}{2\,W\delta_{\mathrm{p}}}=\frac{b}{2(m+f)},
\]
since at most $P/2$ same-sign steps fall in the partial period. The windowed margin is thus $\bar\Delta-\varepsilon_W$, still strictly positive once $W\delta_{\mathrm{p}}>bP/(2\bar\Delta)$; this is the window-stable margin of the main extension, and the integer-multiple case is the special point $\varepsilon_W=0$. The experiments operate here: a window of $\approx 5.7$ periods gives $\varepsilon_W\le b/11.4$, below the headline margin (Section~I). A monitor whose adopt test has effective memory below $P/2$ belongs to the responsive class of part~(b), regardless of its nominal monitoring period; the binding condition is $\ell_{\mathrm{eff}}<P/2$, not $M<P$. DRFC-Seq reads the full-window balanced estimator $\Ghat_{1,2}(n)$ over $n=Wh$ per-arm, per-agent samples, whose target is $\Gbar_{1,2}^{(s)}(P)=\bar\Delta$ by the display above. By the time-uniform confidence sequence of Section~H (Step~1), on the good event $\mathcal{E}$ with $\Prb[\mathcal{E}]\geq1-\alpha$, simultaneously over all window starts,
\[
\bigl|\Ghat_{1,2}(n)-\bar\Delta\bigr|\leq r_n,\qquad r_n=\widetilde O\bigl(1/\sqrt{Nn}\bigr).
\]
Once $n\geq n_0(\bar\Delta)=\widetilde O(1/(N\bar\Delta^2))$ we have $2r_n\leq\bar\Delta/2$, so the challenger arm $2$ has $\mathrm{LCB}_2=\Ghat_{2,1}(n)-2r_n\leq-\bar\Delta+2r_n\leq-\bar\Delta/2<0$ on \emph{every} full window. The wrong arm never attains a positive lower-confidence gap; by Section~H, Step~2, arm $1$ is held and never displaced. No drift-induced false switch occurs, however deep or frequent the instantaneous dips. This is claim~(a).

\paragraph{A waveform-class bias bound.}
The bound $\varepsilon_W=bP/(2L(W))$ was read off the square wave, but it is the worst case over a natural waveform class, which is the only property the period-agnostic monitor of Section~H.2 uses. Let the within-regime drift be any $P$-periodic excursion $g_t=\bar\Delta+b\,s_P(t)$ whose normalized shape $s_P$ has period $P$, per-period mean zero ($\sum_{t=s}^{s+P-1}s_P(t)=0$ for every start $s$), and bounded swing $|s_P(t)|\le1$. The square, triangle, and sine drifts of Section~I all lie in this class. For a window of $L$ rounds the full periods cancel, so $\Gbar^{(s)}(L)-\bar\Delta=(b/L)\sum_{t\in A}s_P(t)$ for a contiguous index set $A$ inside one period. Per-period mean zero makes the positive mass equal half the total variation, $\sum_{t:\,s_P(t)>0}s_P(t)=\tfrac12\sum_{t}|s_P(t)|\le P/2$, and any partial sum is at most this positive mass. Hence
\[
\bigl|\Gbar^{(s)}(L)-\bar\Delta\bigr|\ \le\ \frac{bP}{2L}\ =:\ \varepsilon_W
\qquad\text{for every waveform in the class and every phase,}
\]
with the square wave attaining the bound. Drift-safety of a window therefore depends only on $L(W)$ versus $bP/(2\bar\Delta)$, uniformly over the class.

\paragraph{(b) Every short-memory responsive monitor false-adopts.}
We make the comparator class precise. A \emph{linear short-memory monitor} maintains, for the ordered pair (arm $2$ over arm $1$), a statistic
\[
\widehat C_t=\sum_{u\geq 0} w_u\,\xi_{t-u},\qquad w_u\geq0,\quad \sum_{u\geq0}w_u=1,\quad \sum_{u\geq0}w_u\,\mathbf{1}[u\geq P/2]\leq\epsilon_0,
\]
where $\xi_{t}$ is the per-round candidate-vs-challenger reward contrast (mean $-g_t$ in our instance) and the weights $\{w_u\}$ are the monitor's memory kernel; the last condition says an at-most-$\epsilon_0$ fraction of the kernel mass reaches beyond half a drift period (the \emph{effective memory} is $<P/2$). The monitor adopts arm $2$ when $\widehat C_t\geq\theta$. This class contains: (i) sliding sub-window means of length $w\le P/2$ (uniform kernel on $[0,w)$, $\epsilon_0=0$); (ii) exponentially discounted/forgetting confidence sequences with factor $\rho$ and effective horizon $1/(1-\rho)<P/2$ (geometric kernel, $\epsilon_0=\rho^{P/2}$); and (iii) the per-increment test inside a CUSUM/Page detector, whose alarm at the end of a down-phase requires its most recent run-length statistic---a non-negatively weighted sum of recent increments with the same effective-memory bound---to cross the alarm level $\theta$ (a CUSUM that integrated mass uniformly over a \emph{full} period would be exactly the windowed reader of part~(a), not a short-memory monitor). The only excluded monitor is one whose kernel spans a full drift period, which is precisely DRFC-Seq.

\emph{Responsiveness forces $\theta\le\bar\Delta$.} To be \emph{adaptive} in the sense of the main DRFC-Seq theorem --- i.e.\ to adopt the new best arm after a genuine decision switch, which changes the time-average gap to magnitude $\bar\Delta$ --- the monitor's statistic, an average of post-switch local contrasts whose mean is the post-switch time-average gap $\bar\Delta$, must be able to exceed $\theta$; this requires $\theta\le\bar\Delta$ (a monitor with $\theta>\bar\Delta$ never crosses on a margin-$\bar\Delta$ switch and is non-adaptive). Fix any such $\theta\in[0,\bar\Delta]$.

\emph{Down-phase false adoption.} Evaluate the statistic at a time $t^\star$ that is $P/2$ rounds into a down-phase, so the entire effective-memory window of $\widehat C_{t^\star}$ lies in the current down-phase except for the $\leq\epsilon_0$ kernel mass that leaks into the preceding up-phase. During the down-phase $\E[\xi_t]=-g_t=b-\bar\Delta$; during the leaked up-phase $\E[\xi_t]=-(\bar\Delta+b)\geq-1$. Hence, assuming a small leakage budget $\epsilon_0\leq(b-2\bar\Delta)/4$,
\[
\E[\widehat C_{t^\star}]\;\geq\;(1-\epsilon_0)(b-\bar\Delta)-\epsilon_0(\bar\Delta+b)\;\geq\;(b-\bar\Delta)-2\epsilon_0 b\;\geq\;\bar\Delta+\tfrac12(b-2\bar\Delta)\;>\;\bar\Delta\geq\theta,
\]
using $b>2\bar\Delta$. The statistic is a non-negatively weighted average ($\sum_u w_u=1$) of the per-round, per-agent bounded contrasts in $[-1,1]$; writing $m^{-1}:=\sum_{u}w_u^2/(hN)$ for the kernel's effective sample size ($m=\Theta(\tau_{\mathrm{eff}}hN)$ for any kernel with effective memory $\tau_{\mathrm{eff}}$), the bounded-difference (Azuma--Hoeffding) inequality for weighted sums gives
\[
\Prb\bigl[\widehat C_{t^\star}<\theta\bigr]\leq\exp\!\bigl(-\tfrac12 m\,(\E[\widehat C_{t^\star}]-\theta)^2\bigr)\leq\exp\!\bigl(-\tfrac18 m\,(b-2\bar\Delta)^2\bigr),
\]
using $\E[\widehat C_{t^\star}]-\theta\geq\tfrac12(b-2\bar\Delta)$. Hence the monitor fires on arm $2$ (a false adoption, since arm $1$ is the time-average best and $\Stdecavg=0$) with probability
\[
\Prb[\text{false adoption in this down-phase}]\;\geq\;p_0:=1-\exp\!\bigl(-\tfrac18 (b-2\bar\Delta)^2\bigr)>0,
\]
a constant independent of $T$ and of the threshold $\theta\in[0,\bar\Delta]$ (and $\to1$ as the per-phase sample count $m$ grows). Each regime contains $\Theta(T/P)$ down-phases, so the expected number of false adoptions is $\Omega(p_0\,T/P)$, each charging $\Omega(\bar\Delta)$ instantaneous regret on switch-back; the monitor's drift-induced regret is therefore $\Omega(p_0\bar\Delta\,T/P)$, linear in the number of drift periods.

\paragraph{(b$'$) The general impossibility: any responsive monitor, linear or not.}
The linear-kernel bound above is convenient, but the separation needs no linearity: it follows from a single black-box property---bounded detection delay---by an \emph{exact} change of measure. Call a (possibly randomized, possibly nonlinear) monitor $\mathcal M$ \emph{$(\bar\Delta,d)$-responsive} if, on \emph{every} instance in which arm $2$ becomes the time-average-best arm by margin $\ge\bar\Delta$ starting at some round $\tau$, $\mathcal M$ adopts arm $2$ by round $\tau+d$ with probability at least a fixed constant $q_0$ (e.g.\ $q_0=\tfrac12$). Bounded detection delay $d$ is exactly what switch-adaptivity in the sense of the main DRFC-Seq theorem demands: a monitor whose delay is unbounded never adopts a genuine $\bar\Delta$-margin new best arm in time and is non-adaptive.

We make the responsiveness hypothesis precise, including its dependence on prehistory, so that the lemma is not vacuously strong. Call $\mathcal M$ \emph{$(\bar\Delta,d)$-responsive (uniformly over prehistory)} if, for \emph{every} instance---including one whose pre-change rounds oscillate---on which arm~$2$ becomes time-average-best by margin $\ge\bar\Delta$ at some round $\tau$, $\mathcal M$ adopts arm~$2$ by $\tau+d$ with probability $\ge q_0$. The qualifier ``uniformly over prehistory'' is what the lemma uses and is exactly the property a \emph{short-effective-memory} monitor has automatically: if the adopt-decision at $\tau+d$ depends only on the last $<P/2$ rounds of observations, the monitor \emph{cannot} condition on the oscillatory prefix, so its responsiveness is necessarily uniform. A monitor that escapes the conclusion must therefore make its adopt-decision depend on a window of length $\ge P/2$---i.e.\ have effective memory $\ge P/2$---which is precisely the non-adaptivity regime quantified in the dichotomy below. The hypothesis is thus a genuine property of bounded-memory detectors, not a smuggled assumption.

We make the resource on which the dichotomy turns precise, so that ``effective memory'' is an operational quantity rather than an intuition.

\begin{definition}[Effective memory of a switch monitor]\label{def:eff-mem}
Model a monitor $\mathcal M$ as carrying a latched candidate arm $\hat a_t$ and, at each round $t$, emitting an adopt-decision $D_t\in\{\mathrm{hold}\}\cup\{\mathrm{switch}\ b:b\neq\hat a_t\}$ as a possibly randomized function of its observation history $\mathcal H_t$. Write $\mathcal H_{(t-m,\,t]}$ for the restriction of that history to the last $m$ rounds. The monitor has \emph{effective memory at most $m$} if, for every $t$ and conditioned on the current candidate $\hat a_t$, the conditional law of $D_t$ is determined by $\mathcal H_{(t-m,\,t]}$ alone, that is any two histories agreeing on $\mathcal H_{(t-m,\,t]}$ and on $\hat a_t$ induce the same conditional law of $D_t$. The \emph{effective memory} $\mathrm{mem}(\mathcal M)$ is the least such $m$, measured in rounds, and $+\infty$ when no finite $m$ works. The candidate $\hat a_t$ is a $\lceil\log_2 K\rceil$-bit latch, so a finite $\mathrm{mem}(\mathcal M)$ bounds only the \emph{observation window} the adopt-test reads, not the candidate state a monitor may carry indefinitely; equivalently $\mathrm{mem}(\mathcal M)<m$ is exactly the ``responsive uniformly over prehistory'' property above, since a decision measurable with respect to the last $m$ rounds cannot condition on anything earlier.
\end{definition}

\begin{lemma}[Responsiveness forces drift false alarms: uniform sub-$P/2$ responsiveness]\label{lem:general-responsive}
On the square-wave instance above, every monitor that is $(\bar\Delta,d)$-responsive uniformly over prehistory with detection delay $d<P/2$---equivalently, every monitor whose adopt-decision depends only on a window of $<P/2$ rounds and that is switch-adaptive on clean instances---false-adopts arm $2$ within each down-phase with probability at least $q_0$. The binding hypothesis is the \emph{effective memory} ($<P/2$, uniform over prehistory), not the monitor's internal form: \emph{within that class} the bound is indifferent to whether the statistic is nonlinear (e.g.\ a CUSUM/Page run-length test, a GLR statistic, or an adaptive threshold) or randomized. Classical CUSUM/Page/GLR detectors fall in the class precisely when tuned to detect within half a drift period---i.e.\ in the switch-adaptive regime that is the only regime of interest here---while a variant whose effective memory reaches $\ge P/2$ is exempt from the conclusion but then, by the dichotomy below, fails to adopt genuine $\bar\Delta$-margin switches occurring within $P/2$.
\end{lemma}
\begin{proof}
Fix a down-phase with onset $t^\star$. Build a \emph{genuine-switch} comparison instance $\mathcal I_1$ that (i)~is \emph{identical} to the drift instance $\mathcal I_0$ at every round $t<t^\star+P/2$, and (ii)~for $t\ge t^\star$ \emph{freezes} the down-phase means---arm~$2$ at $\tfrac12+\tfrac12(b-\bar\Delta)$, arm~$1$ at $\tfrac12-\tfrac12(b-\bar\Delta)$---so that under $\mathcal I_1$ arm~$2$ is the time-average-best arm from $\tau=t^\star$ onward, with margin $b-\bar\Delta\ge\bar\Delta$ (using $b>2\bar\Delta$). Because the drift instance $\mathcal I_0$ is itself at the down-phase means throughout $[t^\star,t^\star+P/2)$, the two instances have \emph{identical per-round mean paths---hence identical observation laws---on the entire prefix} $[0,\,t^\star+P/2)$.

Apply responsiveness to $\mathcal I_1$ at $\tau=t^\star$: $\mathcal M$ adopts arm~$2$ by round $t^\star+d$ with probability $\ge q_0$. This adoption event is measurable with respect to the observations in $[0,t^\star+d)\subseteq[0,t^\star+P/2)$ (since $d<P/2$), a prefix on which $\mathcal I_0$ and $\mathcal I_1$ are equal in law. Hence the \emph{same} event has probability $\ge q_0$ under $\mathcal I_0$. Under $\mathcal I_0$ arm~$1$ is the time-average-best arm ($\Stdecavg=0$), so adopting arm~$2$ is a false switch. Each down-phase therefore carries false-adoption probability $\ge q_0$, independent of $T$ and of the monitor's internal form. \qed
\end{proof}

The escape clause is now visible in full generality: a monitor avoids the down-phase false alarm only with detection delay $d\ge P/2$---memory spanning at least half a drift period, enough to net the down-phase against the preceding up-phase and recover the time-average sign---which is precisely \emph{non-responsiveness} to a true $\bar\Delta$-margin switch within half a period. No monitor with sub-$P/2$ detection delay, linear or nonlinear, escapes; this strictly generalizes part~(b), which is the special case of a linear normalized kernel.

\paragraph{The dichotomy.}
Combining: for the square-wave instance, any responsive monitor with sub-$P/2$ detection delay---a linear short-memory monitor with threshold $\theta\le\bar\Delta$ (claim~(b)), or \emph{any} nonlinear/randomized monitor (Lemma~\ref{lem:general-responsive})---false-adopts with constant probability per down-phase, whereas detection delay $\ge P/2$ (equivalently $\theta>\bar\Delta$ in the linear case) makes it non-adaptive to genuine $\bar\Delta$-margin switches. No single threshold yields both drift-robustness and switch-adaptivity. DRFC-Seq does \emph{not} escape the dichotomy---no monitor can---but it attains its drift-robust corner at the minimal cost the impossibility permits: its memory spans an integer number of full periods (its $W$ probe-blocks covering $W\delta_{\mathrm{p}}$, a multiple of $P$, rounds), so its target is the drift-invariant average $\bar\Delta$ (not a phase-dependent local contrast), giving claim~(a); and because a single period suffices to recover the time-average sign, it still adopts a genuine $\bar\Delta$-margin switch within $W$ probe-blocks (Section~H, adaptivity~(b))---one period of detection latency, the price the dichotomy charges any drift-robust monitor, rather than an evasion of it. This is the separation asserted in main Proposition~1. \qed

\begin{corollary}[Memory--latency dichotomy on the effective-memory axis]\label{cor:mem-latency}
Fix the period-$P$ square-wave cancellation instance above, with time-average margin $\bar\Delta>0$ and amplitude $b>2\bar\Delta$. For every monitor $\mathcal M$ in the sense of Definition~\ref{def:eff-mem}:
\begin{enumerate}
\item[(i)] \emph{Short memory false-alarms.} If $\mathrm{mem}(\mathcal M)<P/2$ and $\mathcal M$ is switch-adaptive, adopting a genuine $\bar\Delta$-margin switch within $\mathrm{mem}(\mathcal M)$ rounds with probability at least $q_0$, then in every down-phase $\mathcal M$ false-adopts with probability at least $q_0$.
\item[(ii)] \emph{Drift safety costs $\ge P/2$ latency.} If $\mathcal M$ is drift-safe on this instance, performing no false adopt with probability at least $1-\alpha$, then a genuine switch placed at a down-phase onset is not adopted within $P/2$ rounds with probability at least $1-\alpha$, so its worst-case genuine-switch detection latency is at least $P/2$.
\end{enumerate}
Contrapositively, (i) shows any drift-safe \emph{and} switch-adaptive monitor has $\mathrm{mem}(\mathcal M)\ge P/2$, and (ii) charges any drift-safe monitor at least $P/2$ latency regardless. The two horns are exhaustive over the memory--latency plane, so no monitor is both drift-safe and switch-adaptive with latency below $P/2$. DRFC-Seq attains the floor of (ii): its effective memory is the $W$ probe-blocks spanning $W\delta_{\mathrm p}\ge P$ rounds, and it adopts a genuine $\bar\Delta$-margin switch within that window (main DRFC-Seq theorem), one drift period of latency, the least (ii) permits up to the factor $2$ and the probe-block granularity.
\end{corollary}
\begin{proof}
Part~(i) is Lemma~\ref{lem:general-responsive} read through Definition~\ref{def:eff-mem}: $\mathrm{mem}(\mathcal M)<P/2$ is precisely uniform-over-prehistory responsiveness with detection delay below $P/2$, so the exact change of measure of that lemma transfers the genuine-switch adoption to a down-phase false adoption.

For part~(ii), let $s$ be the onset of a down-phase of the drift instance $\mathcal I_0$ (no decision switch, arm~$1$ time-average-best). Construct the genuine-switch instance $\mathcal I_1$ that is identical in law to $\mathcal I_0$ on the entire prefix $[0,\,s+P/2)$ and, from round $s$ on, freezes the down-phase means so that arm~$2$ is time-average-best by margin $b-\bar\Delta\ge\bar\Delta$ from $\tau=s$ onward, exactly the construction in the proof of Lemma~\ref{lem:general-responsive}. The event ``$\mathcal M$ performs no adopt before round $s+P/2$'' is measurable with respect to the observations on $[0,\,s+P/2)$, on which $\mathcal I_0$ and $\mathcal I_1$ have identical laws, so it carries the same probability under both. Under $\mathcal I_0$ the regime has no switch, so any adoption is a false switch and drift safety gives this event probability at least $1-\alpha$. Transferring to $\mathcal I_1$, where a genuine switch occurred at $\tau=s$, the monitor fails to adopt within $P/2$ rounds with probability at least $1-\alpha$; since the switch phase is adversarial, the worst-case detection latency is at least $P/2$. The argument is an exact two-instance coupling and uses no concentration inequality and no appeal to any upper bound of this paper. \qed
\end{proof}

\begin{remark}[Scope of the dichotomy vs.\ classical sequential change detection]\label{rem:scope-dichotomy}
We state precisely what Proposition~1 does and does not claim, to avoid conflation with the classical sequential change-detection literature (Page; Lorden; Lai). Our result is a dichotomy on the \emph{effective-memory} axis of a detector, \emph{not} a false-alarm-rate-constrained minimax over arbitrary stateful procedures. Concretely: (i)~we do \emph{not} prove a lower bound on the average-run-length-to-false-alarm of an optimally tuned Page/GLR procedure with full state; (ii)~we \emph{do} prove that on the cancellation-drift instance, any detector whose adopt-decision is determined by a window of $<P/2$ rounds (uniformly over prehistory) false-adopts with constant probability per down-phase, while a detector that nets the cancellation, that is whose adopt-decision delay reaches $\ge P/2$, is correspondingly slow to adopt a genuine switch occurring within that half-period. A fully stateful, optimally tuned Page/GLR detector is not claimed to fail, instead paying this $\ge P/2$ latency to stay drift-robust, which is the second horn. The two horns are exhaustive over the responsiveness--delay plane, which is why the trade-off cannot be tuned away \emph{within the switch-adaptive regime}. This is the property the experiments confirm (the tuning sweep of Section~I.1 traces both horns empirically). The novelty relative to the classical theory is the \emph{instance}: cancellation drift makes the time-average and instantaneous decision signals disagree, so the memory length---not the false-alarm threshold---becomes the binding resource, a regime the single-stream change-detection minimaxes do not address.
\end{remark}

\subsection{H.2 Doubling-Window DRFC-Seq: Period-Agnostic Monitoring}
\label{app:adaptive-window}

The main drift proposition charges any drift-robust monitor one drift period of memory, and Section~H attains that corner with a \emph{single} window $W$ tuned so that $W\delta_{\mathrm{p}}$ spans the period $P$. That tuning is the one benign knob of DRFC-Seq, and the main time-average margin assumption simply postulates it. The construction below removes the knob: it learns the shortest drift-safe window online, with no knowledge of $P$, paying only a one-time calibration and an additive $\log J$ inside the confidence radius from the union over the ladder. Write $L(W):=W\delta_{\mathrm{p}}$ for the in-rounds length of a $W$-block window.

\paragraph{The algorithm.}
Doubling-Window DRFC-Seq runs the Section~H probe process \emph{once} and reads it through a geometric ladder of sliding windows
\[
W_0<W_1<\dots<W_{J-1}=W_{\max},\qquad W_i=2^{i}W_0,\qquad J=\bigl\lceil\log_2(W_{\max}/W_0)\bigr\rceil+1 .
\]
All windows share the one probe stream, so the sampling cost is that of the single window $W_{\max}$, not the sum. Window $i$ runs the Section~H confidence sequence at level $\alpha/J$; let $a_i(t)$ be its \emph{absolute best-arm certificate} at probe-step $t$, the empirically best arm whose lower-confidence gap against the runner-up is positive (and hence positive against every other arm, the runner-up being the closest), or $\bot$ if no arm separates from the field. This is exactly the top-versus-runner-up test the implemented monitor runs, so the certificate is an all-pairs dominance statement, not a comparison against the held candidate only. Window $i$ is \emph{anointed} at the first step at which $a_i$ has equalled the top window's certificate $a_{J-1}$ on every one of the previous $W_{\max}$ consecutive full-window steps, where matching requires the same non-$\bot$ arm at each step. The monitor's operating window is the shortest anointed one, $W_{\mathrm{op}}(t)=\min\{W_i:\text{$i$ anointed by }t\}$, and it adopts a switch exactly when $a_{\mathrm{op}}$ moves off the held candidate. No adoption is made before the first anointment.

\begin{theorem}[Period-agnostic anytime-valid monitoring]\label{thm:app-adaptive-window}
Run Doubling-Window DRFC-Seq on a within-regime drift instance from the waveform class of Section~H.1 with time-average margin $\bar\Delta>0$ (the main time-average margin condition) and unknown period $P$ and amplitude $b$. Define a window $W$ to be \emph{drift-safe} when its windowing bias is strictly below the margin, $\varepsilon_W<\bar\Delta$, equivalently $L(W)>bP/(2\bar\Delta)$, and write its windowed margin $\rho_W:=\bar\Delta-\varepsilon_W>0$. Suppose the ladder top covers the period, is drift-safe with a constant-fraction margin, and is sample-sufficient as the calibration authority,
\[
L(W_{\max})\;\ge\;\max\!\Bigl(P,\ \tfrac{bP}{\bar\Delta}\Bigr)\quad(\text{so }\varepsilon_{W_{\max}}\le\tfrac{\bar\Delta}{2}),
\qquad
W_{\max}\,h\;\ge\;n_0(\bar\Delta/2),
\]
where $n_0(\rho):=\inf\{n:r_n\le\rho/4\}=\widetilde O(1/(N\rho^2))$ is the per-arm count of the recovery Lemma~\ref{lem:drfc-seq-recovery}, so that at margin $\rho$ a window's absolute certificate is non-$\bot$ with a strict gap (the balanced pairwise estimator has deviation $r_n$, giving $\mathrm{LCB}=\Ghat-2r_n\ge\rho-3r_n\ge\rho/4>0$). The window advances one probe block per evaluation, the only times the monitor acts, so the relevant phases are the evaluated block-phases, and we take the period to be a whole number of probe blocks (a benign coarsening, since a block spans far fewer rounds than a period) so that a contiguous run of $P/s_{\mathrm{blk}}$ blocks enumerates every block-phase exactly, where $s_{\mathrm{blk}}=L(W)/W$ is the per-block round span (Remark~\ref{rem:off-grid} removes this coarsening for arbitrary real $P$ at an explicit $2b/W$ off-grid slack). Let $W^\star$ be the shortest ladder window the period-oracle would run, the shortest that is both drift-safe and sample-sufficient, $L(W^\star)\ge bP/\bar\Delta$ \emph{and} $W^\star h\ge n_0(\bar\Delta/2)$; such a window exists with $L(W^\star)\le L(W_{\max})$ because $W_{\max}$ meets both. Then, with $r_n$ the Section~H stitched radius:
\begin{enumerate}
\item[(a)] \emph{(Anytime-valid drift safety.)} With probability at least $1-\alpha$, simultaneously over all rounds $t\le T$, the monitor performs no false switch, that is no displacement of the time-average-best arm once it is held, during any pure-drift regime ($\Stdecavg=0$); the false-switch probability is $\le\alpha$, uniformly in $T$ and independent of $\Stloc$.
\item[(b)] \emph{(Period-agnostic competitive latency.)} After a calibration of at most $2L(W_{\max})$ rounds the shortest anointed window satisfies $W_{\mathrm{op}}\le W^\star$, and any genuine $\bar\Delta$-margin switch whose new regime spans at least $2W_{\mathrm{op}}$ probe-blocks (the recovery-lemma spacing of the main DRFC-Seq theorem) is adopted within $O\bigl(L(W^\star)\bigr)=O\bigl(W^\star\delta_{\mathrm{p}}\bigr)$ rounds, the per-switch latency of the oracle told $P$ that runs the single window $W^\star$, up to the doubling factor~$2$ and the radius inflation of~(c), and using no knowledge of $P$.
\item[(c)] \emph{(Union cost.)} The ladder inflates the radius over a single tuned window only by the additive $\log J=\log\log_2(W_{\max}/W_0)$ inside the stitched logarithm of $r_n$; since $\log J$ is dominated by $\log(C_0K(K{-}1)s^2/\alpha)$ this is a $(1+o(1))$ inflation of the radius, not a standalone multiplicative factor.
\end{enumerate}
\end{theorem}

\begin{proof}
\emph{Step 0 (union confidence sequence; part (c)).} Apply the Section~H, Step~1 confidence sequence to each of the $J$ windows at level $\alpha/J$ and union-bound. On a good event $\mathcal{E}_J$ with $\Prb[\mathcal{E}_J]\ge 1-\alpha$, every window's balanced estimate obeys $|\Ghat-\Gbar|\le r_n^{(J)}$ for all per-arm counts $n$ simultaneously, where $r_n^{(J)}$ is $r_n$ with $\alpha$ replaced by $\alpha/J$, i.e.\ with the stitched term $\log(C_0K(K{-}1)s^2/\alpha)$ raised by the additive $\log J$. Since $\log J=\log\log_2(W_{\max}/W_0)$ is dominated by the $\log(C_0K(K{-}1)s^2/\alpha)$ already present, $r_n^{(J)}=r_n\bigl(1+o(1)\bigr)$; this is~(c). Write $r_n$ for $r_n^{(J)}$ and work on $\mathcal{E}_J$ throughout.

\emph{Step 1 (anointment certifies true safety and a usable margin).} For a window $W$ write its \emph{worst-phase margin} $m(W):=\min_s \Gbar^{(s)}_{a^{\mathrm{avg}}}(L(W))$, the smallest over phases of the time-average-best arm's windowed gap to its nearest challenger. The waveform-class bound of Section~H.1 is used only in its safe direction, as an upper bound on the bias: $m(W)\ge\bar\Delta-\varepsilon_W$ with $\varepsilon_W=bP/(2L(W))$, for every waveform in the class. We never need the bias to be attained, so the argument makes no square-wave-specific claim. The top window is drift-safe with a constant-fraction margin, $m(W_{\max})\ge\bar\Delta-\varepsilon_{W_{\max}}\ge\bar\Delta/2>0$, and once filled is sample-sufficient ($W_{\max}h\ge n_0(\bar\Delta/2)$, so $r_n\le\bar\Delta/8$ and its certified gap $\Ghat-2r_n\ge m(W_{\max})-3r_n>0$); hence on $\mathcal{E}_J$ its absolute certificate is the time-average-best arm at every phase and is never $\bot$, $a_{J-1}\equiv a^{\mathrm{avg}}$. Now consider any window $W$ anointed at $t$. The monitor evaluates only at probe blocks and the window start advances one block per evaluation, so the phases the monitor ever sees form the block-grid. Anointment requires $W$'s certificate to equal $W_{\max}$'s on $W_{\max}$ consecutive full-window blocks, a span over which the start sweeps $L(W_{\max})\ge P$ rounds, at least one full drift period. Under the whole-block period convention this contiguous run enumerates every block-phase the monitor will ever evaluate, $W$'s own worst block-phase $s_W$ included (Remark~\ref{rem:off-grid} removes the convention for arbitrary real $P$, carrying an explicit off-grid slack $2b/W$). Since $W_{\max}$'s certificate there is the non-$\bot$ arm $a^{\mathrm{avg}}$, matching forces $W$ to certify the same $a^{\mathrm{avg}}$ at full count $n_W=W h$, so its lower-confidence gap of $a^{\mathrm{avg}}$ over the runner-up was positive, $\Ghat-2r_{n_W}>0$, whence on $\mathcal{E}_J$
\[
m(W)=\Gbar^{(s_W)}_{a^{\mathrm{avg}}}\ \ge\ \Ghat-r_{n_W}\ >\ r_{n_W}\ >\ 0 ,
\]
the worst-phase margin against the nearest challenger, not merely against the held candidate.
Thus every anointed window is \emph{truly} drift-safe ($m(W)>0$, its candidate-arm target positive at every phase, for any waveform in the class) and carries the self-certified margin $m(W)>r_{n_W}$. No appeal to a drift-unsafe window producing a wrong certificate is made, so the step holds for sine and triangle drifts exactly as for the square wave.

\emph{Step 2 (part (a)).} The monitor adopts only through the operating window $W_{\mathrm{op}}=\min\{W_i:\text{anointed}\}$ (and adopts nothing before the first anointment), which by Step~1 is truly drift-safe, with candidate-arm target $\ge m(W_{\mathrm{op}})>0$ at every phase. In a pure-drift regime ($\Stdecavg=0$) the time-average-best arm $a^{\mathrm{avg}}$ holds, and any challenger $a'$ has windowed gap $\Gbar_{a',a^{\mathrm{avg}}}\le-m(W_{\mathrm{op}})<0$; using only $m(W_{\mathrm{op}})>0$ from Step~1 and the anytime confidence sequence, its lower-confidence gap is $\Ghat_{a',a^{\mathrm{avg}}}-2r_n\le(-m(W_{\mathrm{op}})+r_n)-2r_n<0$ at \emph{every} full-window count $n$ the monitor evaluates, with no $n_0$ needed. No challenger ever attains a positive lower-confidence gap, so the operating window never moves off $a^{\mathrm{avg}}$ and no adoption occurs. A handoff to a newly anointed shorter window is between two truly drift-safe windows that both certify $a^{\mathrm{avg}}$, so it is silent. The false-switch probability is at most $\Prb[\mathcal{E}_J^{\mathrm c}]\le\alpha$, uniformly over all $t\le T$ and independent of $\Stloc$. This is the proof--algorithm alignment point: adoption is gated on anointment, the empirically checkable certificate that Step~1 turns into true drift-safety, and the proof never assumes a window safe that the algorithm has not anointed.

\emph{Step 3 (part (b)).} The top window fills in $W_{\max}$ blocks and, being drift-safe and sample-sufficient, agrees with itself thereafter, so it is anointed by $2L(W_{\max})$ rounds. Consider $W^\star$, the shortest ladder window that is both drift-safe and sample-sufficient ($L(W^\star)\ge bP/\bar\Delta$ and $W^\star h\ge n_0(\bar\Delta/2)$); by the safe direction of the bias bound $m(W^\star)\ge\bar\Delta-\varepsilon_{W^\star}\ge\bar\Delta/2$, and its sample-sufficiency holds by definition. Since $L(W^\star)\le L(W_{\max})$ it fills inside the calibration window; on $\mathcal{E}_J$ it certifies $a^{\mathrm{avg}}$ at every probe-step phase (margin $\ge\bar\Delta/2$, $2r_n\le\bar\Delta/2$ once filled) and so equals $W_{\max}$'s certificate over any period, hence is anointed by $2L(W_{\max})$. Thus the shortest anointed window satisfies $W_{\mathrm{op}}\le W^\star$. Whatever $W_{\mathrm{op}}$ is, Step~1 gives it a self-certified margin $m(W_{\mathrm{op}})>r_{n_{\mathrm{op}}}$ at $n_{\mathrm{op}}=W_{\mathrm{op}}h$. By the polynomial dependence of the stitched radius, $\rho>r_n$ at $n$ forces $n_0(\rho)=O(n)$, so $n_0(m(W_{\mathrm{op}}))=O(W_{\mathrm{op}}h)$: the operating window already holds enough samples to recover. Now take a genuine average-decision switch at $\tau$ past calibration, where $a_r^{\mathrm{avg}}$ becomes time-average-best by margin $\bar\Delta$ and the new regime spans at least $2W_{\mathrm{op}}$ probe-blocks. The recovery lemma (Lemma~\ref{lem:drfc-seq-recovery}), whose regime-spacing hypothesis is thus met, applied to $W_{\mathrm{op}}$ adopts $a_r^{\mathrm{avg}}$ once the post-switch window fills and reaches $n_0(m(W_{\mathrm{op}}))=O(W_{\mathrm{op}}h)$ samples, i.e.\ within $O(W_{\mathrm{op}}\delta_{\mathrm{p}})=O(L(W_{\mathrm{op}}))\le O(L(W^\star))$ rounds. This is the per-switch latency of the oracle told $P$ that runs $W^\star$, up to the doubling granularity and the $r_n$ inflation of~(c). The latency is in rounds, set by the window length in probe-blocks times $\delta_{\mathrm{p}}$, with no spurious division by $N$. No knowledge of $P$ enters and the calibration $2L(W_{\max})$ is paid once.
\end{proof}

\begin{remark}[Arbitrary real periods, off the block grid]\label{rem:off-grid}
The whole-block period convention used in Step~1 ($P$ an integer multiple of the probe-block span $\delta_{\mathrm{p}}$) is a simplifying device, not a requirement, and the guarantee degrades gracefully for an arbitrary real $P$. The windowed time-average gap is Lipschitz in the start phase, since advancing the window by one block replaces a single boundary block of per-block gap within $b$ of the mean,
\[
\bigl|\Gbar^{(s+\delta_{\mathrm{p}})}_{a^{\mathrm{avg}}}(L(W))-\Gbar^{(s)}_{a^{\mathrm{avg}}}(L(W))\bigr|
\le\frac{2b}{W}=:\varepsilon_W^{\mathrm{off}} .
\]
The monitor evaluates only block-grid phases, so the worst evaluated phase lies within one block of the true worst phase $s_W$ and underestimates the worst-phase margin by at most $\varepsilon_W^{\mathrm{off}}$. Anointment over a $W_{\max}$-block run still sweeps $L(W_{\max})\ge P$ rounds, hence an evaluated phase within $\delta_{\mathrm{p}}$ of $s_W$, and matching $W_{\max}$'s non-$\bot$ certificate there gives, in place of the exact Step~1 conclusion, $m(W)\ge r_{n_W}-\varepsilon_W^{\mathrm{off}}$. Every anointed window is therefore truly drift-safe whenever $\varepsilon_W^{\mathrm{off}}<r_{n_W}$, that is $W\gtrsim b^2Nh/\log(C_0K(K-1)s^2/\alpha)$, an explicit lower bound the operating window meets once $W_0$ is enlarged by a constant factor. Under this mild strengthening part~(a)'s no-false-switch guarantee holds for \emph{arbitrary real} $P$, with the whole-block case recovered as the $\varepsilon_W^{\mathrm{off}}=0$ specialization, and the latency~(b) is unchanged up to the additive $\varepsilon^{\mathrm{off}}$ inside the margin. The experiment of Section~I samples periods off the block grid and reports the predicted zero false-switch rate, so it matches this analysis rather than relying on the convention.
\end{remark}

\begin{remark}[The price of period-agnosticism is honest, not zero]
Theorem~\ref{thm:app-adaptive-window} does not evade the Proposition~1 dichotomy. It still spends one drift period of memory at the operating window $W_{\mathrm{op}}\le W^\star$, exactly the charge the dichotomy levies, and it adds a one-time $O(L(W_{\max}))$ calibration that is the information cost of \emph{discovering} the safe scale without being told $P$: any monitor that adopts a $\bar\Delta$-margin switch faster than $L(W_{\max})$ could not yet have ruled out a drift period just below $L(W_{\max})$, so some calibration is unavoidable. What the theorem buys is that this cost is paid \emph{once} rather than per regime, after which the monitor runs at the oracle's per-switch latency. Section~I reports the matching experiment: across randomized periods, phases, and square, triangle, and sine drift shapes, Doubling-Window DRFC-Seq holds the oracle's false-switch rate and tracks its detection latency without the period as input, while any single fixed window is either drift-unsafe at the long end or needlessly slow at the short end.
\end{remark}

\begin{remark}[Relation to multi-scale selection]
The ladder-of-windows shape recalls Lepski's method for bandwidth selection~\cite{lepski1991problem}, but the selection principle differs in three ways the drift-monitoring problem forces. First, the objects compared are \emph{decisions}, the certified argmax arms $a_i$, not estimator values within overlapping confidence intervals, with agreement on which arm is best rather than on a real-valued estimate. Second, the agreement test is \emph{temporal}, a match over a full $W_{\max}$-block sweep of one drift period rather than at a single time, which is what converts a pointwise certificate into the anytime-valid drift-safety of part~(a) by forcing a candidate window to certify the best arm at its own worst phase. Third, the authority is anchored to \emph{safety}, the longest drift-safe window, not merely to the largest scale, and the whole construction is anytime-valid through the union confidence sequence over the ladder. The output therefore carries a uniform-in-time no-false-switch guarantee, a property estimation-risk bandwidth selection does not target.
\end{remark}

\section{I. Additional Experimental Figures}
\label{app:extra-figures}

This section contains the experimental figures referenced in the main paper as ``supplement.''
All figures and tables are reproduced from the anonymized code-and-data package accompanying the paper. The simulators use only the Python standard library, and the figure layer uses \texttt{numpy} and \texttt{matplotlib}; all tabular inputs needed for deterministic reproduction are included in that package. The descriptions below preserve the parameters, seeds, and statistical aggregation needed to reproduce each reported result.

\begin{figure}[t]
\centering
\includegraphics[width=0.62\linewidth]{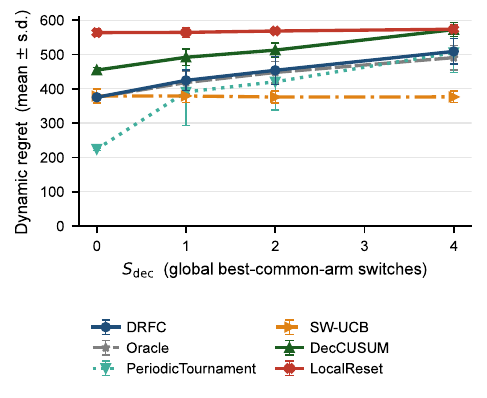}
\caption{Full stochastic decision-switch sweep. Regret grows linearly with $\Stdec$ across all baselines; DRFC tracks the switch count with low false-switch rate.}
\label{fig:supp-switches}
\end{figure}

\begin{figure}[t]
\centering
\includegraphics[width=0.62\linewidth]{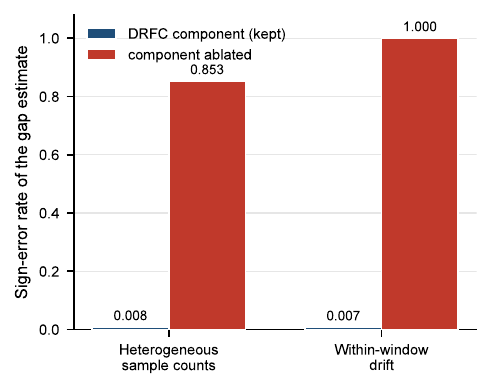}
\caption{Component ablations. Sample-weighted aggregation flips sign on heterogeneous-sample-count instances; sequential block estimation flips sign under within-window drift. Equal-agent averaging and randomized interleaving avoid both failure modes.}
\label{fig:supp-component}
\end{figure}

\begin{figure}[t]
\centering
\includegraphics[width=0.62\linewidth]{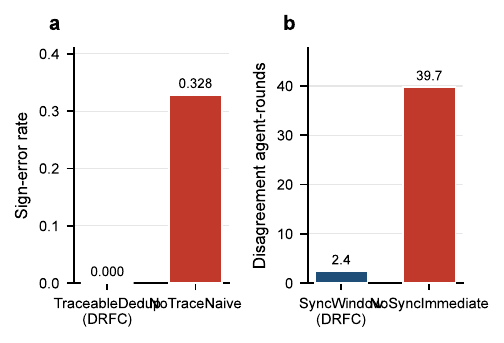}
\caption{Communication ablations. Naive gossip over-counts high-propagation records and biases the reconstructed gap; source-traceable deduplication preserves the equal-source estimator. Synchronization windows reduce agent disagreement after switches.}
\label{fig:supp-gossip}
\end{figure}

\begin{figure}[t]
\centering
\includegraphics[width=0.62\linewidth]{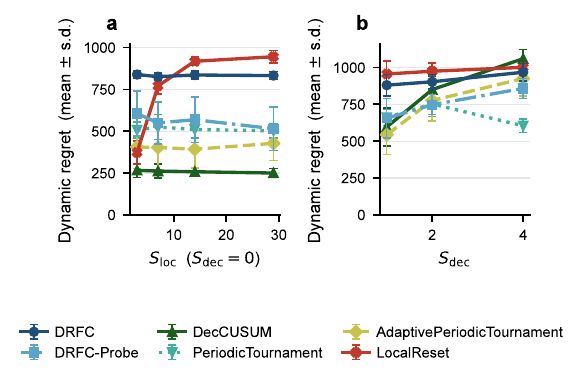}
\caption{User-cluster semi-real stress test. Left: with $\Stdec=0$ and varying $\Stloc$, conservative DRFC and DRFC-Probe are insensitive to the drift rate, while LocalReset, SW-UCB-Dec, and CUSUM-Reset all grow with $\Stloc$. Right: with $\Stloc$ fixed, DRFC tracks $\Stdec$.}
\label{fig:supp-cluster}
\end{figure}

\begin{figure}[t]
\centering
\includegraphics[width=0.62\linewidth]{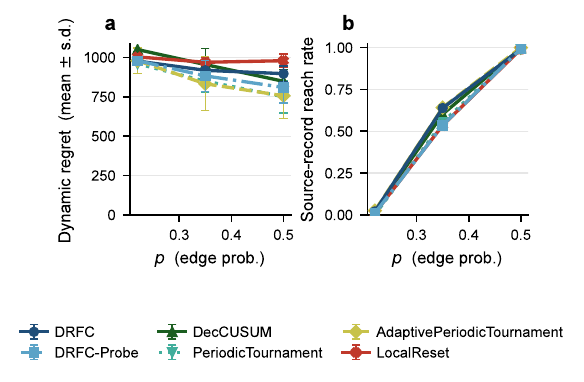}
\caption{Graph-probability sweep. As Erd\H{o}s--R\'enyi edge probability $p$ rises from $0.22$ to $0.50$, source-record reach rate climbs from $0.04$ to $0.99$, decision-switch recovery delay shrinks, and DRFC regret falls accordingly.}
\label{fig:supp-graphp}
\end{figure}

\begin{figure}[t]
\centering
\includegraphics[width=0.62\linewidth]{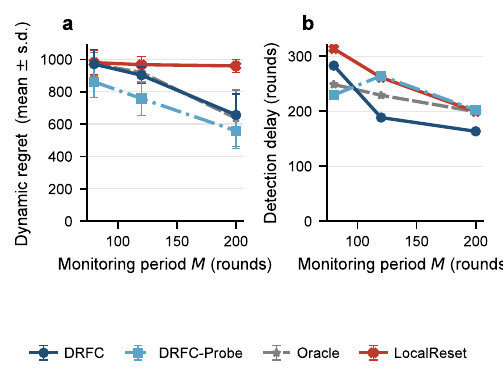}
\caption{Monitor period $M$ sensitivity at $\Stdec=2$, $K=4$, $\Delta=0.10$. Left: regret mean $\pm$ std vs $M\in\{80,120,200\}$. Right: detection delay vs $M$. The expected trade-off appears: larger $M$ reduces regret from periodic certification cost but lengthens detection delay after a true switch. DRFC-Probe and Oracle dominate at large $M$; LocalReset is insensitive to $M$ because it triggers on local-change boundaries, not monitoring schedule.}
\label{fig:supp-monitor}
\end{figure}

\begin{table}[t]
\centering
\footnotesize
\setlength{\tabcolsep}{3.5pt}
\caption{Scaling of decision regret with arms $K$ and agents $N$ (mean $\pm$ std over 12 seeds for this sweep, mixed topologies). For the decision-level methods regret grows with $K$ and with $N$, consistent with the $N\Stdec$ term of the upper bound. Because this is group regret, a sum over the $N$ agents, the near-linear growth in $N$ is expected from the summation alone and is not by itself a test of the communication term $N\Stdec\DG$, which is the smaller higher-order departure isolated separately by the graph-probability sweep below. DRFC-KGraph tracks Oracle (the perfect-communication reference, not a lower bound) to within noise. DRFC-Probe spends source traffic to stay nearly flat in $N$ (the probe trade-off of Appendix~C), so it is non-monotone in $N$ and lowest at large $N$. This sweep isolates $K$/$N$ scaling of the decision-level mechanism, the separation from local-change-reactive protocols is the distinct experiment in Figure~\ref{fig:separation}. Top block DRFC family, bottom block decision-level baselines.}
\label{tab:scaling-breadth}
\begin{tabular}{lcccc}
\toprule
\multicolumn{5}{c}{(a) Arms $K$ ($N=12$)}\\
\midrule
Method & $K{=}2$ & $K{=}4$ & $K{=}8$ & $K{=}16$\\
\midrule
DRFC-KGraph & $331{\pm}40$ & $823{\pm}22$ & $1037{\pm}37$ & $1154{\pm}28$\\
DRFC-Probe  & $118{\pm}42$ & $590{\pm}127$ & $981{\pm}56$ & $1138{\pm}25$\\
Oracle      & $317{\pm}29$ & $856{\pm}25$ & $1056{\pm}20$ & $1156{\pm}15$\\
\midrule
SW-UCB-Dec  & $356{\pm}23$ & $791{\pm}23$ & $1124{\pm}19$ & $1309{\pm}7$\\
CUSUM-Reset & $458{\pm}11$ & $821{\pm}14$ & $1073{\pm}18$ & $1236{\pm}22$\\
\bottomrule
\addlinespace[2pt]
\toprule
\multicolumn{5}{c}{(b) Agents $N$ ($K=4$)}\\
\midrule
Method & $N{=}6$ & $N{=}12$ & $N{=}24$ & $N{=}48$\\
\midrule
DRFC-KGraph & $408{\pm}13$ & $823{\pm}22$ & $1588{\pm}34$ & $3110{\pm}0$\\
DRFC-Probe  & $349{\pm}69$ & $590{\pm}127$ & $490{\pm}131$ & $778{\pm}0$\\
Oracle      & $418{\pm}13$ & $856{\pm}25$ & $1612{\pm}52$ & $3110{\pm}0$\\
\midrule
SW-UCB-Dec  & $442{\pm}13$ & $791{\pm}23$ & $1315{\pm}50$ & $1984{\pm}50$\\
CUSUM-Reset & $397{\pm}12$ & $821{\pm}14$ & $1681{\pm}11$ & $3385{\pm}17$\\
\bottomrule
\end{tabular}
\end{table}

\subsection{Full versions of main-paper experiment figures}
\label{app:main-figs-full}
These are the full-size figures and tables summarized in the main paper's Experiments and Related Work sections (positioning, real-data corroboration, and secondary scaling results). Main-result references below use result names whenever possible so they remain stable under section reordering.

\begin{table}[t]
\centering
\small
\setlength{\tabcolsep}{3pt}
\footnotesize
\begin{tabular}{lllll}
\toprule
Closest prior line & Non-stat.\ unit & Dec. & Regret scales with & On $\Stloc\!\gg\!\Stdec$\\
\midrule
Garivier--Moulines & opt-arm switches & no & \# switches & resets per local change\\
Besbes et al. & variation $V_T$ & no & $V_T^{1/3}T^{2/3}$ & local shift inflates $V_T$\\
Cao; Besson et al. & change points & no & \# changes & restarts each local change\\
Abbasi-Yadkori et al. & best-arm switches & no & \# switches & single stream, no average\\
M.-Rubio; Chawla et al. & none (stationary) & yes & spec.\ gap, $N,K$ & no temporal-drift model\\
Komiyama et al. & global changes & no & \# changes & local counted as global\\
\midrule
\textbf{DRFC (ours)} & decision $\Stdec$ & yes & $\Stdec$ not $\Stloc$ & adapts to $\Stdec$ only\\
\bottomrule
\end{tabular}
\caption{Why the closest accepted lines cannot exploit heterogeneous cancellation. Each either counts every local change, so its adaptation term scales with $\Stloc$ rather than $\Stdec$, or assumes stationary rewards; only DRFC certifies non-stationarity at the agent-average decision level and adapts to $\Stdec\ll\Stloc$. The cancellation instance of the two-agent illustration (main paper) has $\Stloc=\Theta(T/L)$ yet $\Stdec=0$, so every $\Stloc$-scaling row pays $\Omega(T/L)$ while DRFC pays only monitoring cost. (Main-paper Related Work.)}
\label{tab:related}
\end{table}

\begin{table}[t]
\centering
\small
\setlength{\tabcolsep}{4pt}
\begin{tabular}{lcccc}
\toprule
Scenario & DRFC-Seq & DecCUSUM & DL-GLR & DL-ADR \\
\midrule
\multicolumn{5}{l}{\emph{Genuine-switch adopt rate} ($\Stdecavg{=}1$, no drift)}\\
$\Delta{=}0.10$            & $1.00$ & $1.00$ & $0.87$ & $0.97$ \\
$\Delta{=}0.20$            & $1.00$ & $1.00$ & $0.93$ & $0.93$ \\
\midrule
\multicolumn{5}{l}{\emph{False-switch rate} (drift present)}\\
$\Delta{=}0.10$, drift ($\Stdecavg{=}1$) & $0.00$ & $0.90$ & $0.83$ & $0.63$ \\
$\Delta{=}0.20$, drift ($\Stdecavg{=}1$) & $0.00$ & $0.03$ & $0.63$ & $0.20$ \\
drift only ($b{=}0.20$, $\Stdecavg{=}0$) & $0.00$ & $0.90$ & $0.90$ & $0.60$ \\
\bottomrule
\end{tabular}
\caption{Adaptivity / non-vacuity battery at the \emph{live} operating point ($N=24$, $2r_n\!\approx\!0.085<\Delta$; $30$ seeds, horizon $24000$, switch at $12000$); DL-GLR is DL-GLR-klUCB, DL-ADR is DL-ADR-bandit. \emph{Top block}: genuine-switch adopt rate (fraction of runs adopting the correct new arm)---all four decision-level monitors are alive, including the two strong change detectors. \emph{Bottom block}: false-switch rate---DRFC-Seq is the only one that stays at $0$ under drift, while DecCUSUM, DL-GLR-klUCB, and even the gradual-drift-specialized DL-ADR-bandit all false-alarm, exactly the dichotomy of main Proposition~1. Detection latencies (rounds): DRFC-Seq $\sim5000$--$6800$, DL-GLR $\sim2200$--$4200$, DL-ADR $\sim1500$--$2300$, DecCUSUM $\sim300$---faster detection buys the false alarms. (Main-paper ``Non-vacuity: DRFC-Seq does switch on genuine switches.'')}
\label{tab:adaptivity}
\end{table}

\begin{figure}[t]
\centering
\includegraphics[width=0.62\linewidth]{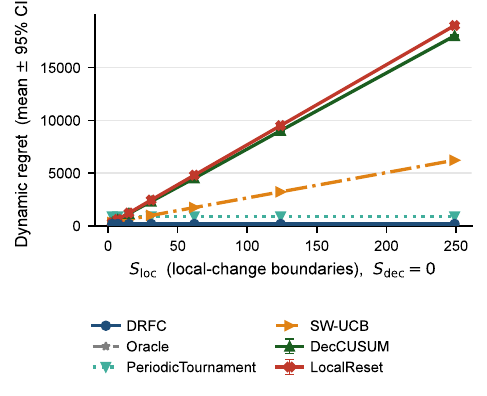}
\caption{Separation stress test ($\Stdec=0$, varying $\Stloc$). LocalReset regret grows linearly in $\Stloc$; DRFC and Oracle pay only monitoring cost. (Main-paper ``Separation under decision-irrelevant drift.'')}
\label{fig:separation}
\end{figure}

\begin{figure}[t]
\centering
\includegraphics[width=0.62\linewidth]{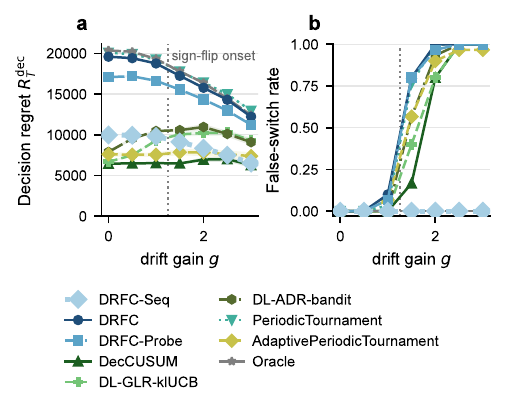}
\caption{Within-regime drift on MovieLens-1M genre dynamics, semi-synthetic (real measured fluctuation, anchored margin) ($\Stdecavg=0$, $\Delta=0.10$, arms $\{$Romance, Thriller, Comedy, Action$\}$, $8/16$ real bins reverse). $x$-axis: gain $g$ scaling the empirical fluctuation (dotted: real sign-flip onset). (a) Decision regret $R_T^{\mathrm{dec}}$: the periodic-scan family blows up to $\sim2\times$ DRFC-Seq; the strong decision-level detectors DL-GLR-klUCB and DL-ADR-bandit are \emph{higher} regret than DRFC-Seq under strong drift while DecCUSUM reaches marginally lower regret only by false-switching. (b) False-switch rate: DRFC-Seq (adopt-capable) stays at $0$; all change-detection competitors---including DL-GLR-klUCB and DL-ADR-bandit---climb to $1$. $N=24$ agents at a live operating point ($2r_n<\Delta$); 30 seeds.}
\label{fig:movielens-drift}
\end{figure}

\begin{table}[t]
\centering
\small
\setlength{\tabcolsep}{3pt}
\begin{tabular}{lcccc}
\toprule
Method & $R_T^{\mathrm{dec}}$ & $R_T$ (inst.) & False-switch & Switches \\
\midrule
\textbf{DRFC-Seq}          & $\mathbf{2710\pm6}$   & $\mathbf{6962}$ & $\mathbf{0.00}$ & $0.0$ \\
DecCUSUM                   & $3053\pm475$          & $7305$ & $1.00$ & $4.7$ \\
AdaptivePeriodicTourn.     & $3075\pm786$          & $7326$ & $1.00$ & $5.1$ \\
DL-GLR-klUCB               & $3459\pm736$          & $7710$ & $0.90$ & $3.4$ \\
DL-ADR-bandit              & $3868\pm500$          & $8120$ & $0.90$ & $4.1$ \\
Oracle (perfect comm.)     & $3851\pm265$          & $8103$ & $0.00$ & $0.0$ \\
DRFC-KGraph                & $4335\pm281$          & $8587$ & $1.00$ & $10.1$ \\
PeriodicTournament         & $4606\pm243$          & $8857$ & $1.00$ & $9.6$ \\
\bottomrule
\end{tabular}
\caption{\emph{Raw, un-anchored} MovieLens-1M within-regime drift ($\Stdecavg=0$, raw time-average margin $\bar\Delta\!\approx\!0.0115$, $8/16$ bins reverse, $N=24$, $30$ seeds, horizon $12000$). No constructed square wave and no anchored margin, the drift being the measured rating fluctuation. DRFC-Seq is the only deployable monitor that is both lowest-regret and never false-switches, while every fast decision-level change detector (DecCUSUM, DL-GLR-klUCB, DL-ADR-bandit, AdaptivePeriodicTournament) false-alarms, consistent with main Proposition~1. The conservative decentralized DRFC variants also false-switch on stale snapshots, and only the perfect-communication Oracle reference matches DRFC-Seq's zero rate---DRFC-Seq attains it under realistic communication. The passive SW-UCB-Dec index ($3356\pm78$, no adopt-event) is omitted from the false-switch comparison.}
\label{tab:natural-drift}
\end{table}

\paragraph{The comparator choice does not manufacture the result.} With one arm played per agent per round, instantaneous dynamic regret and decision regret differ by a single instance constant, $R_T(\pi)-R_T^{\mathrm{dec}}(\pi)=\sum_t N\,[\bar\mu_{a_t^\star}(t)-\bar\mu_{a_r^{\mathrm{avg}}}(t)]=:D\geq0$, because each policy's own action terms cancel. $D$ does not depend on the policy $\pi$, so replacing the time-average comparator $a_r^{\mathrm{avg}}$ by the instantaneous best arm $a_t^\star$ adds the same $D$ to every method and leaves all rankings and gaps unchanged. For this instance $D\approx4252$, computed from the equal-agent means ($8$ of $16$ real bins reverse, minimum instantaneous global gap $-0.080$, time-average margin $0.0115$). The $R_T$ column applies this offset. DRFC-Seq stays the lowest-regret method under the instantaneous comparator as well ($6962$ against DecCUSUM $7305$ and the Oracle $8103$), so the separation is a property of the methods, not of the comparator. The diagnostic is included in the reproducibility package.

\begin{table}[t]
\centering
\small
\setlength{\tabcolsep}{4pt}
\resizebox{\columnwidth}{!}{%
\begin{tabular}{lrrrrr}
\toprule
Cluster & \textbf{Drama} & Romance & Thriller & Comedy & Action\\
\midrule
age25\_occ17 & \textbf{0.680} & 0.627 & 0.582 & 0.592 & 0.548\\
age25\_occ0  & \textbf{0.618} & 0.564 & 0.470 & 0.515 & 0.472\\
age35\_occ17 & \textbf{0.537} & 0.480 & 0.461 & 0.449 & 0.452\\
\midrule
Population avg & \textbf{0.620} & 0.538 & 0.526 & 0.510 & 0.494\\
\bottomrule
\end{tabular}}
\caption{MovieLens-1M cluster-by-genre breakdown (mean binarised rating, $\geq 4$). \emph{Drama} (bold) is the highest-rated genre in \emph{every} cluster, with population-average margin $0.08$ over the runner-up. Within each cluster, however, the ranking of the four non-best genres varies (see e.g.\ Comedy/Thriller/Action interchange across rows), producing high $\Stloc$ on the time-binned replay even though $\Stdec=0$. Three representative clusters and the population average are shown; \emph{Drama} is the top genre in all 12 clusters in the released data table.}
\label{tab:movielens}
\end{table}

\begin{figure}[t]
\centering
\includegraphics[width=0.62\linewidth]{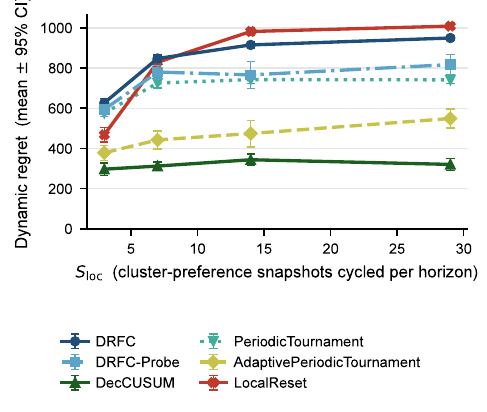}
\caption{MovieLens-1M replay ($\Stdec=0$, $N=12$ clusters, $K=5$ genres). DRFC-Probe slope is $\sim2.4\times$ smaller than LocalReset.}
\label{fig:movielens}
\end{figure}

\begin{figure}[t]
\centering
\includegraphics[width=0.62\linewidth]{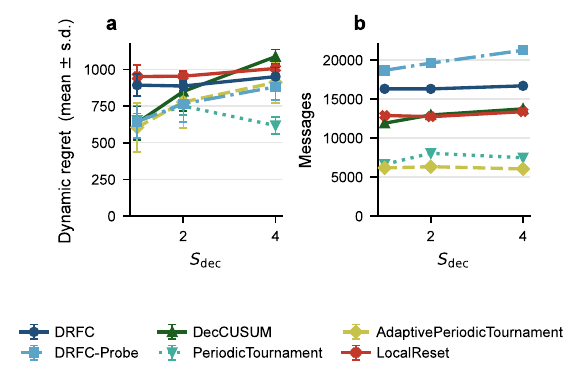}
\caption{K-ary stress check. Probe-triggered DRFC trades regret against source-record count under tight regimes ($\Stdec\in\{1,2,4\}$).}
\label{fig:kary}
\end{figure}

\begin{figure}[t]
\centering
\includegraphics[width=0.62\linewidth]{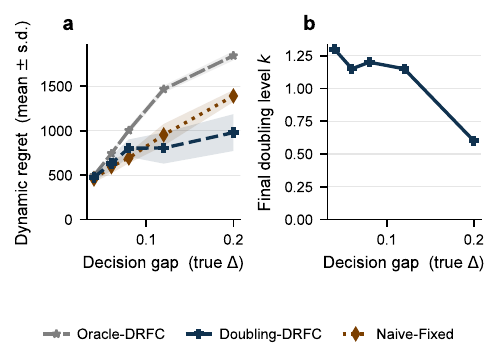}
\caption{Geometric scan-budget walk-up for Gap-Adaptive DRFC (labeled Doubling-DRFC in the plot). Mean walk-up $\leq 1.5$ levels across a $5\times$ range of $\Delta$, matching oracle-tuned regret.}
\label{fig:doubling}
\end{figure}

\begin{table}[t]
\centering
\small
\setlength{\tabcolsep}{4pt}
\begin{tabular}{lrrrr}
\toprule
Method & MV slope & MV final & K-ary $S_4$ & K-ary $f$\\
\midrule
DRFC & 12.4 & 950 & 951 & 0.10\\
DRFC-Probe & 8.6 & 818 & 881 & 0.35\\
SW-UCB-Dec & 5.9 & 852 & 787 & 0.00\\
CUSUM-Reset & 10.5 & 894 & 961 & 0.55\\
LocalReset & 20.8 & 1009 & 1007 & 0.20\\
Oracle & 12.7 & 979 & 925 & 0.15\\
\bottomrule
\end{tabular}
\caption{Headline numbers. \emph{MV slope}: regret-per-$\Stloc$ slope on MovieLens replay ($\Stloc{=}3$ to $29$, $\Stdec{=}0$). \emph{MV final}: regret at $\Stloc{=}29$. \emph{K-ary $S_4$}: regret at $\Stdec{=}4$ on the K-ary stress check (horizon $1200$). \emph{$f$}: false-switch rate at $\Stdec{=}4$. All entries are read directly from the archived aggregate tables. DRFC-Probe's MovieLens slope ($8.6$) is $\sim2.4\times$ smaller than the explicit local-reset baseline LocalReset ($20.8$), with lower absolute regret than LocalReset at every $\Stloc$; the passive SW-UCB index has a flatter slope but a higher regret floor and no adopt-event ($f{=}0$). On the K-ary check, conservative DRFC has the lowest false-switch rate among switching methods ($0.10$).}
\label{tab:headline}
\end{table}

\subsection{I.1 Decision-level baseline tuning (fairness)}
\label{app:baseline-tuning}

The two decision-level change detectors we add---DL-GLR-klUCB (GLR on the global
aggregated gap) and DL-ADR-bandit (multi-scale adaptive restart on the global
mean)---are run at single operating points in the main paper
(GLR threshold $6.0$; ADR threshold scale $1.0$ and minimum window $2$). A natural concern is whether these points were
chosen to make the baselines look weak, i.e.\ whether a more conservative
setting would let them keep their switch-adaptivity while avoiding the
within-regime-drift false alarms. Table~\ref{tab:baseline-tuning} sweeps each
detector's primary sensitivity knob and, at every setting, measures \emph{both}
the genuine-switch adopt rate (a real mid-horizon switch, no drift) and the
drift-only false-switch rate (no switch, live drift $b=0.20$), at the same live
operating point ($N=24$, $\Delta=0.10$, horizon $24000$, $30$ seeds) as
main Table~2 (the adaptivity battery).

The sweep is not a search for a better threshold; it is a direct test of the
dichotomy of main Proposition~1. \emph{No} setting of either detector
reaches the both-good corner. Holding the detector switch-adaptive (adopt rate
$\approx1$) pins its drift false-switch rate at $0.60$--$0.90$; the only way to
drive the false-switch rate down is to make the detector progressively
non-adaptive---GLR at threshold $20$ already misses half the genuine switches
(adopt $0.50$) while \emph{still} false-alarming at rate $0.27$, and ADR reaches
a low false-switch rate only by collapsing to adopt rate $0.30$ (scale $2.5$) or
total inertia (adopt $0$, scale $4.0$). The single points used in the main paper
sit in the \emph{adaptive} regime, which is the setting most favourable to the
baselines (they are compared where they actually function as switch detectors,
not where they are inert). DRFC-Seq is off this trade-off curve entirely: adopt
rate $1.00$ \emph{and} false-switch rate $0.00$ simultaneously, at a fixed level
$\alpha=0.10$ with no threshold to tune.

\begin{table}[t]
\centering
\small
\setlength{\tabcolsep}{5pt}
\begin{tabular}{lcc}
\toprule
Detector / setting & genuine adopt & drift FSR \\
\midrule
\multicolumn{3}{l}{\emph{DL-GLR-klUCB} (threshold sweep)}\\
\quad threshold $2.0$            & $1.00$ & $0.87$ \\
\quad threshold $4.0$            & $1.00$ & $0.90$ \\
\quad threshold $6.0$ (paper)    & $0.83$ & $0.90$ \\
\quad threshold $9.0$            & $0.60$ & $0.83$ \\
\quad threshold $14.0$           & $0.67$ & $0.50$ \\
\quad threshold $20.0$           & $0.50$ & $0.27$ \\
\midrule
\multicolumn{3}{l}{\emph{DL-ADR-bandit} (threshold-scale sweep)}\\
\quad scale $0.50$               & $1.00$ & $0.63$ \\
\quad scale $0.75$               & $1.00$ & $0.60$ \\
\quad scale $1.00$ (paper)       & $1.00$ & $0.63$ \\
\quad scale $1.50$               & $0.80$ & $0.80$ \\
\quad scale $2.50$               & $0.30$ & $0.17$ \\
\quad scale $4.00$               & $0.00$ & $0.00$ \\
\midrule
\textbf{DRFC-Seq} (no threshold, $\alpha{=}0.10$) & $\mathbf{1.00}$ & $\mathbf{0.00}$ \\
\bottomrule
\end{tabular}
\caption{Decision-level baseline tuning sweep. ``genuine adopt'' is the
genuine-switch adopt rate (higher is more switch-adaptive); ``drift FSR'' is the
drift-only false-switch rate (lower is more drift-robust). Across the full sweep
of either detector, no setting achieves both a high adopt rate and a low
false-switch rate---suppressing the drift false alarms requires giving up
switch-adaptivity, exactly the main Proposition~1 dichotomy. The paper
operating points sit in the adaptive (baseline-favourable) regime. DRFC-Seq
attains both simultaneously. Reproduced by the baseline-tuning sweep in the
anonymized package, with archived per-seed outputs. $30$ seeds.}
\label{tab:baseline-tuning}
\end{table}

\subsection{I.2 Per-round communication cost}
\label{app:message-cost}

The source-record complexity bound (Proposition~\ref{prop:app-records}) is here
made concrete and compared to the baselines, answering the question of whether
source-traceable gossip and periodic scanning are bandwidth-heavy in practice.
Table~\ref{tab:message-cost} reports the mean source-record transmissions per
agent per round at the standard K-ary operating point ($N=12$, $K=4$,
$\Stdec=2$, edge probability $0.50$, horizon $1200$, $20$ seeds), read from the
archived per-seed communication logs under the simulator's common gossip
accounting, where every method is charged for each record it floods over the
random graph.

The conservative DRFC certifier transmits about $1.13$ records per agent per
round, within a factor $1.3$ of the decision-level change detectors
(DecCUSUM, CUSUM-Reset, LocalReset, all near $0.9$) and about $2.3$ times the
passive index sharing of SW-UCB-Dec, which only broadcasts a $K$-vector each
gossip cycle. The light-probe and Oracle variants spend more traffic ($1.36$
and $1.30$) to buy lower regret, the regret-traffic trade of
the probe variant (Appendix~C) and the Pareto figure. So the periodic certifier is
comparable to a change detector rather than an order heavier, and its cost is
governed by the monitor period $M$ and block size $h$ through the
$\widetilde O(NK(h+\DG)T/M)$ rate of Proposition~\ref{prop:app-records}, with
the per-scan and per-probe source-record displays
$\widetilde O(\Bcert_r)$ and $O(NKq_p)$ of that proposition giving the explicit
per-event message complexity the deployment can tune. Dividing the total rate by
$NT$, the per-agent per-round cost is $\widetilde O(K(h+\DG)/M)$, which depends on
the network size $N$ only through the flooding time $\DG$, so the per-agent
bandwidth grows with the network only through $\DG$ and is otherwise flat in $N$,
while the only $N$-dependence of the total is the unavoidable linear factor
from having $N$ agents. The component split is therefore a constant certifier
overhead set by $M$ and $h$ plus a graph term that scales with $\DG$, the same
$\DG$ the topology and switching-stress sweeps (Sections~\ref{app:topology},~\ref{app:sdec-stress}) vary directly.

\begin{table}[t]
\centering
\small
\setlength{\tabcolsep}{6pt}
\begin{tabular}{lrr}
\toprule
Method & records/agent/round & total records \\
\midrule
SW-UCB-Dec   & $0.50$ & $7200$ \\
DecCUSUM     & $0.90$ & $12957$ \\
CUSUM-Reset  & $0.87$ & $12492$ \\
LocalReset   & $0.88$ & $12740$ \\
\midrule
DRFC-KGraph  & $1.13$ & $16298$ \\
DRFC-Probe   & $1.36$ & $19604$ \\
Oracle       & $1.30$ & $18732$ \\
\bottomrule
\end{tabular}
\caption{Per-round source-record communication at the standard operating point
($N=12$, $K=4$, $\Stdec=2$, edge probability $0.50$, horizon $1200$, $20$
seeds). Records per agent per round is the total flooded source records divided
by $N$ and the horizon. The conservative DRFC certifier is within a factor
$1.3$ of the decision-level change detectors and the probe and Oracle variants
spend more traffic for lower regret, so source-traceable certification is
bandwidth-comparable to a change detector and tunable through $M$ and $h$ by
Proposition~\ref{prop:app-records}. Read from
the archived communication logs in the reproducibility package.}
\label{tab:message-cost}
\end{table}

\subsection{I.3 Consolidated parameter-tuning guidelines}
\label{app:tuning}

This subsection collects the practical guidance for choosing the parameters
$(M,h,\Hscan,\delta_{\mathrm p},W,\alpha,C)$ in one place, including what each
controls, the reported default, and what to do when the gap $\Delta$ or the
drift period is unknown. Table~\ref{tab:tuning} is the summary; the main text
carries a distilled version. The recurring principle is that the two
margin-dependent knobs, the scan budget $\Hscan$ (conservative DRFC) and the
window $W$ (DRFC-Seq), are the only inputs that need prior knowledge, and each
has a knowledge-free doubling replacement that pays only a $\log\log$ or
$\log J$ overhead.

\paragraph{When $\Delta$ is unknown.} Use Gap-Adaptive DRFC
(Proposition~\ref{prop:app-gap-adaptive}), which walks the scan budget up a
geometric ladder $\Hscan^{(k)}\propto4^k$ and attains the
Theorem~\ref{thm:app-regret} regret at the running-minimum gap with only an
additive $\log\log T$ overhead, so no margin input is needed above a floor.

\paragraph{When the drift period is unknown.} Use Doubling-Window DRFC-Seq
(Theorem~\ref{thm:app-adaptive-window}), which runs a geometric ladder of
sliding windows on one shared probe stream and adopts on the shortest window
that has matched the longest drift-safe window, recovering the period-oracle's
latency within the doubling factor and a $(1+o(1))$ union inflation, with no
period input.

\paragraph{The benign knobs.} The monitor period $M$ trades periodic monitoring
cost ($\widetilde O(NK(h+\DG)T/M)$) against post-switch detection delay
($O(M)$), so $M$ scales with the communication budget and the tolerable switch
latency, not with $\Delta$. The block size $h$ sets the per-block radius and is
a small constant. The probe gap $\delta_{\mathrm p}$ need only satisfy
$\delta_{\mathrm p}\ge Kh+\DG(\alpha)$ so each block fully floods before the
next. The level $\alpha$ is the horizon-uniform false-switch budget and is
robust across a $20\times$ range (Table~\ref{tab:window-sens}). The radius
constant $C$ errs only conservatively as it grows (Table~\ref{tab:confidence-sweep}),
so the practical $C=1.15$ and the worst-case theorem value bracket a range over
which the zero-false-switch conclusion is preserved.

\begin{table}[t]
\centering
\small
\setlength{\tabcolsep}{3.5pt}
\begin{tabular}{llll}
\toprule
Knob & Controls & Default & If $\Delta$/period unknown \\
\midrule
$M$ & monitor cost vs delay & $120$ & comm.\ budget, not $\Delta$ \\
$h$ & per-block radius & $6$/$24$ & small constant \\
$\Hscan$ & scan budget & scan calib. & Gap-Adaptive DRFC (App.~B) \\
$\delta_{\mathrm p}$ & probe cadence & $20$ & $\ge Kh+\DG$ \\
$W$ & drift-period span & $120$ & Doubling-Window (Thm.~H.2) \\
$\alpha$ & false-switch level & $0.10$ & robust $20\times$ \\
$C$ & radius constant & $1.15$ & conservative as $C{\uparrow}$ \\
\bottomrule
\end{tabular}
\caption{Parameter-tuning summary. The only margin-dependent inputs are the
scan budget $\Hscan$ and the window $W$, and each has a knowledge-free doubling
replacement (Proposition~\ref{prop:app-gap-adaptive},
Theorem~\ref{thm:app-adaptive-window}). The remaining knobs are governed by the
communication budget and the tolerable latency rather than by $\Delta$ or the
drift period, and the confidence inputs $\alpha,C$ are robust or err
conservatively (Tables~\ref{tab:window-sens},~\ref{tab:confidence-sweep}).}
\label{tab:tuning}
\end{table}

\subsection{I.4 Beyond Erd\H{o}s--R\'enyi: topology robustness}
\label{app:topology}

The main paper's random-graph sweep uses Erd\H{o}s--R\'enyi activation. To test
whether the result depends on that choice we re-run DRFC at $\Stdec=2$ over five
structurally distinct base graphs, each made time-varying by activating every
base edge independently per gossip round, with $20$ seeds. The families are the
regular low-diameter ring-with-chords (paper default), a Watts--Strogatz
small-world graph, a random geometric graph, a Barab\'asi--Albert scale-free
graph, and the path graph of maximal diameter $N-1$.

Figure~\ref{fig:topology} reports group regret and post-switch detection delay
against the source-record reach rate, the fraction of flooded records that reach
every agent and the empirical proxy for $1$ minus the flooding-failure
probability, hence for the flooding time $\DG$. The reading is that what governs
DRFC is the realized reach rate, not the topology label. At a sparse activation
probability $0.30$ the reach rate ranges from $0.68$ on the small-world graph
down to $0.22$ on the geometric graph, and DRFC regret tracks it monotonically,
from $922$ at the highest reach to $952$ at the lowest, with the four connected
families collapsing onto a single reach-versus-regret curve. As the activation
probability rises to $0.80$ every connected family reaches a reach rate of
$1.00$ and a regret near $900$ within noise, regardless of whether the graph is
a regular ring, a clustered geometric graph, or a hub-dominated scale-free
graph. The path graph is the one persistent stress case. Its single-line base
graph fails to flood ($\mathrm{reach}\approx0$) even at activation $0.80$,
because a single inactive edge severs the line, so DRFC cannot recover switches
and its detection delay stays censored at the horizon. This is exactly the
$\DG\to N$ corner where the $N\Stdec\DG$ term of Theorem~\ref{thm:app-regret}
blows up, the honest boundary of the guarantee rather than a counterexample to
it. The experiment therefore supports the claim that only $\DG$, no finer
topology functional, enters the regret, and it directly answers the topology
generalization question by spanning small-world, clustered, scale-free, and
maximal-diameter graphs.

\begin{figure}[t]
\centering
\includegraphics[width=0.92\linewidth]{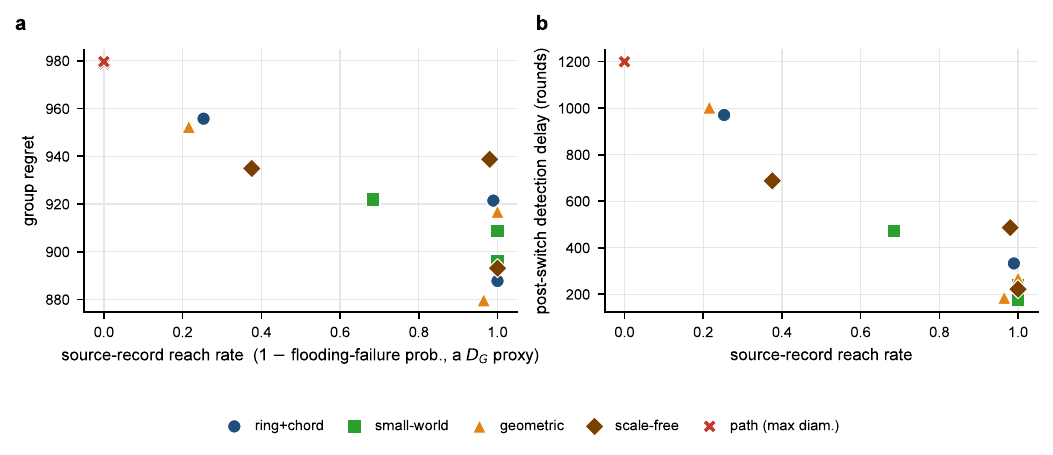}
\caption{Topology robustness beyond Erd\H{o}s--R\'enyi ($\Stdec=2$, $N=12$,
$K=4$, $20$ seeds). Each base graph (ring+chord, small-world, geometric,
scale-free, path) is made time-varying by per-round edge activation. (a) Group
regret against the source-record reach rate. Across the four connected families
the points collapse onto one curve, so regret is governed by the realized reach
rate (the $\DG$ proxy) and not by the topology family. (b) Post-switch detection
delay against reach rate. The path graph (maximal diameter) never floods and
sits at censored delay, the $\DG\to N$ corner of the bound. Reproduced by the
topology-robustness sweep.}
\label{fig:topology}
\end{figure}

\subsection{I.5 Switching-regime stress: validating the $\Delta$, spacing, and $\DG$ dependence}
\label{app:sdec-stress}

To validate the way Theorem~\ref{thm:app-regret} and the main information-theoretic lower bound depend on the per-regime certification $\Ccert_r$, the
connectivity $\DG$, and the per-regime scan count, we sweep three axes in the
genuine-switching regime $\Stdec>0$, with $20$ seeds, each isolating one
predicted dependence (Figure~\ref{fig:sdec-stress}).

\paragraph{Gap $\Delta$.} Holding $\Stdec=2$ and edge probability $0.50$, as the
gap grows over $\{0.08,0.12,0.18,0.26\}$ the post-switch detection delay falls
from $382$ rounds to $117$, because a wider gap is certified with fewer balanced
blocks, the empirical face of the $\Ccert_r\sim(K-1)/\Delta$ certification cost.
Group regret rises over the same range from $1111$ to $3262$, because in this
frequently-monitored horizon the dominant cost is the periodic monitoring term,
whose each wasted probe pull of a suboptimal arm costs about $\Delta$, so the
two branches of the bound appear on the two metrics, certification on the delay
and monitoring on the regret.

\paragraph{Connectivity $\DG$.} Holding $\Stdec=2$ and $\Delta=0.10$, raising the
edge probability over $\{0.25,0.40,0.60,0.85\}$ lifts the reach rate from $0.07$
to $1.00$ and collapses the detection delay from $1667$ rounds to about $200$,
the direct empirical signature of the $N\Stdec\DG$ delay term, since higher
connectivity shrinks the flooding time $\DG$.

\paragraph{Switch count $\Stdec$.} Holding $\Delta=0.10$, edge probability
$0.50$, and the inter-switch spacing fixed at a regime length of $600$ rounds
(horizon scaled as $600(\Stdec+1)$), group regret grows linearly in $\Stdec$
over $\{1,2,3,4\}$, from $845$ to $1388$ to $1877$ to $2414$, with near-constant
increments of about $520$. This isolates the $N\Stdec$ switch term, showing the
per-switch cost is constant once the spacing is held fixed, so the linear growth
is in the switch count and not an artifact of a shrinking regime.

Across all three axes the false-switch rate stays $0$, and the qualitative
scalings match the bound, not its constants, as elsewhere in this supplement.

\begin{figure}[t]
\centering
\includegraphics[width=0.98\linewidth]{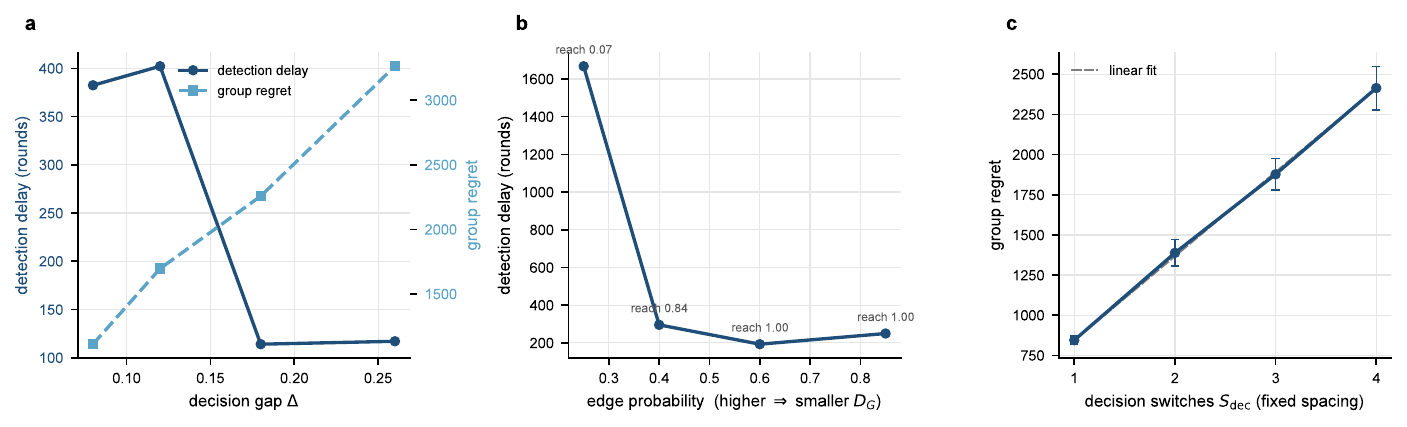}
\caption{Switching-regime stress at $\Stdec>0$ ($N=12$, $K=4$, $20$ seeds),
validating the bound's three dependences. (a) Detection delay falls with the gap
$\Delta$ (easier certification, the $\Ccert_r\sim(K-1)/\Delta$ branch) while
regret rises with $\Delta$ (the monitoring branch, each wasted pull costing
$\Delta$). (b) Detection delay collapses as connectivity rises and the reach
rate climbs, the $N\Stdec\DG$ delay term. (c) Group regret grows linearly in
$\Stdec$ at fixed inter-switch spacing, the $N\Stdec$ switch term. Reproduced by
the switching-regime stress sweep.}
\label{fig:sdec-stress}
\end{figure}

\subsection{I.6 Asynchrony and partial participation}
\label{app:async}

The main algorithm schedules a synchronized balanced block in which
every agent participates. We now show the equal-agent estimator degrades
gracefully when only a random subset of agents participates in each block, the
partial-participation model of asynchronous cooperative bandits where in each
round an unknown time-varying subset of agents is active~\citep{wang2025asynchronous},
and that the natural re-normalization keeps the estimator unbiased, answering
the question of whether equal-agent averaging can absorb dropouts.

\paragraph{Oblivious dropout model.} At each balanced block let
$\Pi\subseteq[N]$ be the participating set, each agent included independently
with probability $p$, drawn independently of the rewards and of the mean path
(the oblivious analogue of the main reward-obliviousness assumption). A participating agent
pulls every active arm on the block's balanced schedule, a non-participating one
contributes no pulls. The \emph{re-normalized} equal-agent block estimator of
arm $a$ averages over the participants,
$\hat\mu_a^{\Pi}=|\Pi|^{-1}\sum_{i\in\Pi}\hat g_{i,a}$, where $\hat g_{i,a}$ is
agent $i$'s within-block empirical mean. The \emph{naive} estimator keeps
dividing by the full count, $\hat\mu_a^{\mathrm{naive}}=N^{-1}\sum_{i\in\Pi}\hat g_{i,a}$.

\begin{proposition}[Re-normalization is unbiased under oblivious dropout]
\label{prop:async-unbiased}
Fix a block with active arms $a,b$ and per-agent means $\mu_{i,a}$. Under the
oblivious dropout model, conditioned on $|\Pi|=m$ for any $m\ge1$,
\[
\E\!\left[\hat\mu_a^{\Pi}-\hat\mu_b^{\Pi}\,\middle|\,|\Pi|=m\right]
=\frac1N\sum_{i=1}^N(\mu_{i,a}-\mu_{i,b})=\Gbar_{a,b},
\]
the all-agent equal-agent gap, so the re-normalized gap estimator is exactly
unbiased for $\Gbar_{a,b}$ for every realized participation level. Its anytime
radius uses the realized pull count $|\Pi|h$, so over a length-$W$ window with
mean participation $p$ the effective per-arm count is $pNWh$ and the radius
$r_n$ inflates by a factor $1/\sqrt p$ over full participation. The naive
estimator instead satisfies
$\E[\hat\mu_a^{\mathrm{naive}}-\hat\mu_b^{\mathrm{naive}}\mid|\Pi|=m]=(m/N)\Gbar_{a,b}$,
so its gap is biased by the participation fraction, shrinking to $p\,\Gbar_{a,b}$
in expectation.
\end{proposition}
\begin{proof}
Conditioned on $|\Pi|=m$, the set $\Pi$ is a uniformly random $m$-subset of
$[N]$ because the inclusions are i.i.d.\ and exchangeable, so
$\Prb[i\in\Pi\mid|\Pi|=m]=m/N$ for every $i$. The block schedule is fixed
independently of the rewards, so
$\E[\hat g_{i,a}-\hat g_{i,b}]=\mu_{i,a}-\mu_{i,b}$. Hence
\[
\E\!\left[\tfrac1m\textstyle\sum_{i\in\Pi}(\hat g_{i,a}-\hat g_{i,b})\,\middle|\,|\Pi|=m\right]
=\tfrac1m\sum_{i=1}^N\tfrac{m}{N}(\mu_{i,a}-\mu_{i,b})
=\Gbar_{a,b},
\]
independent of $m$, which is the first claim. The radius statement is the
concentration radius of Lemma~\ref{lem:app-concentration} with $N$ replaced by
the realized participant count, whose window mean is $pN$. For the naive
estimator the $1/m$ is replaced by $1/N$, multiplying the conditional mean by
$m/N$, and taking expectations over $|\Pi|$ with $\E|\Pi|=pN$ gives the factor
$p$. \qed
\end{proof}

The re-normalized estimator therefore carries the within-regime drift guarantee
of the main DRFC-Seq theorem to partial participation with no bias, only a
$1/\sqrt p$ radius inflation that slows adoption gracefully, while the naive
estimator reads a gap shrunk by $p$ and stalls. The records that do flood are
deduplicated by source exactly as before, so a dropped agent is simply absent
from the window rather than double-counted, and the source-traceable scheme is
unchanged. This is robustness to missed probe blocks and temporary
disconnection at the estimator level, complementary to the lossy-gossip
robustness already modeled by the main random-graph flooding assumption.

\paragraph{Empirical degradation.} Figure~\ref{fig:async} sweeps the
participation probability $p$ from $1.0$ down to $0.30$ on the within-regime
drift instance and on a genuine-switch instance ($N=24$, $K=4$, $\Delta=0.10$,
amplitude $b=0.20$, period-spanning window, $30$ seeds), and confirms
Proposition~\ref{prop:async-unbiased}. On the drift instance the re-normalized
windowed gap stays flat at $0.095$ to $0.097$, within noise of the true margin
$\Delta=0.10$, across the whole participation range, while the naive gap shrinks
linearly from $0.096$ at full participation to $0.028$ at $p=0.30$, exactly the
$p\,\Gbar$ bias the proposition predicts. DRFC-Seq holds a zero false-switch
rate at every participation level under both estimators, so drift safety
survives dropout, whereas the decision-level DecCUSUM false-alarms at rate one
throughout. On the genuine-switch instance the re-normalized monitor adopts the
switch in every run down to $p=0.90$, with the adopt latency growing from $4980$
to $5557$ rounds as the $1/\sqrt p$ radius inflation slows certification, and
still in two thirds of runs at $p=0.75$, stalling only once $p\le0.6$ drives the
inflated radius above the margin, a safe failure that never mis-adopts. The
naive estimator, reading a gap shrunk by $p$, already collapses to a zero adopt
rate by $p=0.75$. Re-normalization therefore extends reliable adoption to far
lower participation, and the only price of dropout is the graceful
$1/\sqrt p$ latency growth the analysis predicts, not a bias or a false switch.

\begin{figure}[t]
\centering
\includegraphics[width=0.98\linewidth]{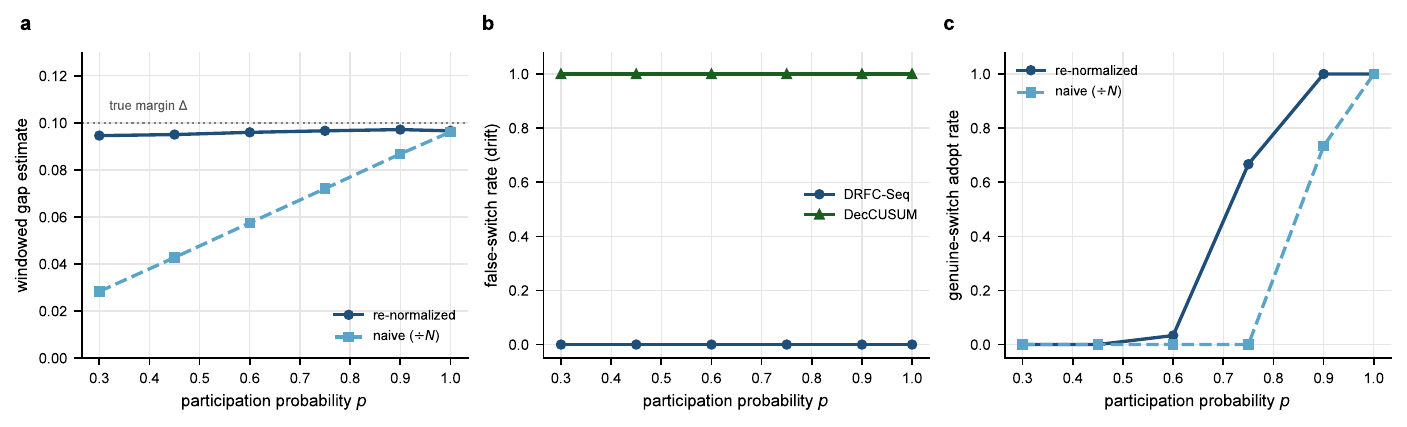}
\caption{Asynchrony and partial participation ($N=24$, $K=4$, $\Delta=0.10$,
$30$ seeds), sweeping the per-block participation probability $p$. (a) On the
drift instance ($\Stdecavg=0$) the re-normalized windowed gap stays at the true
margin $\Delta$ across all $p$ (unbiased), while the naive estimator shrinks to
$p\,\Delta$ (biased), the empirical face of Proposition~\ref{prop:async-unbiased}.
(b) DRFC-Seq holds a zero false-switch rate under dropout while DecCUSUM
false-alarms at rate one. (c) On a genuine-switch instance ($\Stdecavg=1$) the
re-normalized monitor keeps adopting down to about $75\%$ participation with a
gracefully growing latency, while the naive estimator stalls. Reproduced by the
partial-participation sweep.}
\label{fig:async}
\end{figure}

\clearpage
\subsection{I.7 Structured and persistent participation bias}
\label{app:struct}

Section~\ref{app:async} assumes dropout is oblivious, independent of the rewards. The harder question is structured participation, where which agents drop out is correlated with their private arm preferences and the pattern persists across the horizon. In a heterogeneous network the equal-agent global mean of an arm is the average of per-agent means, so if participation correlates with those means the surviving subset is no longer a uniform sample and the re-normalized subset mean is biased. The headline is an exact, operating-point-independent characterization of that bias as a covariance identity. A finite-sample corollary then shows the margin absorbs the bias once it exceeds the confidence radius by a constant factor, a regime reached with a longer window or a larger gap, and a final part shows that knowing the propensities removes the bias outright.

\paragraph{Persistent-propensity model.} Each agent $i$ participates in every block independently with a fixed propensity $p_i\in[p_{\min},p_{\max}]$, reward-independent given the block but correlated with the per-agent means $\mu_{i,a}$. Write $\bar p=N^{-1}\sum_i p_i$. The re-normalized subset-mean estimator of arm $a$ over a window of $W$ blocks pools the surviving agent-block means, $\hat\mu_a=\bigl(\sum_{w,i}\mathbf{1}_{i,w}\hat g_{i,a,w}\bigr)/\bigl(\sum_{w,i}\mathbf{1}_{i,w}\bigr)$, where $\mathbf{1}_{i,w}$ indicates that agent $i$ joined block $w$.

\begin{proposition}[Subset-mean bias under structured participation]
\label{prop:struct}
Under the persistent-propensity model, as the window length $W\to\infty$ the re-normalized subset mean converges almost surely to the propensity-weighted mean, and its bias relative to the all-agent equal-agent mean $\bar\mu_a=N^{-1}\sum_i\mu_{i,a}$ is the exact identity
\[
\mathrm{bias}_a \;=\; \frac{\sum_i p_i\mu_{i,a}}{\sum_i p_i}-\bar\mu_a \;=\; \frac{\mathrm{Cov}_i(p,\mu_{\cdot,a})}{\bar p},
\]
where $\mathrm{Cov}_i$ and $\bar p$ are the empirical covariance and mean over the $N$ agents. Consequently
\[
|\mathrm{bias}_a|\;\le\;\frac{(p_{\max}-p_{\min})\,(\max_i\mu_{i,a}-\min_i\mu_{i,a})}{4\,\bar p},
\]
and the contested-pair gap estimate is biased by $B_{a,b}=[\mathrm{Cov}_i(p,\mu_{\cdot,a})-\mathrm{Cov}_i(p,\mu_{\cdot,b})]/\bar p$. DRFC-Seq's adopt rule preserves the true decision ordering, with neither a false switch nor a missed genuine switch, whenever $\Delta>4r_n$ and
\[
|B_{a,b}| \;<\; \Delta - 4r_n ,
\]
so structured participation bias is absorbed up to a margin-dependent boundary and only beyond it can flip a decision. If the propensities are known, the self-normalized inverse-propensity estimator $\hat\mu_a^{\mathrm{IPW}}=\bigl(\sum_{i\in\Pi}\hat g_{i,a}/p_i\bigr)/\bigl(\sum_{i\in\Pi}1/p_i\bigr)$ is asymptotically unbiased for $\bar\mu_a$, at the price of a variance inflation by the Kish effective-sample-size factor folded into $r_n$.
\end{proposition}
\begin{proof}
Participation is drawn independently across the $W$ blocks, and the within-block means $\hat g_{i,a,w}$ are bounded in $[0,1]$, so by the law of large numbers the block-averaged numerator $W^{-1}\sum_{w,i}\mathbf{1}_{i,w}\hat g_{i,a,w}\to\sum_i p_i\mu_{i,a}$ and the block-averaged denominator $W^{-1}\sum_{w,i}\mathbf{1}_{i,w}\to\sum_i p_i$ almost surely, hence $\hat\mu_a\to(\sum_i p_i\mu_{i,a})/(\sum_i p_i)$. Clearing denominators,
\[
\frac{\sum_i p_i\mu_{i,a}}{\sum_i p_i}-\bar\mu_a
=\frac{N\sum_i p_i\mu_{i,a}-(\sum_i p_i)(\sum_i\mu_{i,a})}{N\sum_i p_i}
=\frac{\mathrm{Cov}_i(p,\mu_{\cdot,a})}{\bar p},
\]
with $\mathrm{Cov}_i(p,\mu_{\cdot,a})=N^{-1}\sum_i p_i\mu_{i,a}-\bar p\,\bar\mu_a$, which is the identity. The range bound follows from $|\mathrm{Cov}_i(p,\mu_{\cdot,a})|\le\sqrt{\mathrm{Var}_i(p)\,\mathrm{Var}_i(\mu_{\cdot,a})}\le\tfrac14(p_{\max}-p_{\min})(\max_i\mu_{i,a}-\min_i\mu_{i,a})$ by Cauchy--Schwarz and Popoviciu's inequality $\mathrm{Var}\le(\mathrm{range}/2)^2$. For the decision boundary, on the anytime-valid event of probability $1-\alpha$ each arm's windowed estimate lies within $r_n$ of its biased limit, so the empirical gap lies within $2r_n$ of $\Gbar_{a,b}+B_{a,b}$, with $\Gbar_{a,b}$ the all-agent gap of magnitude $\Delta$. The adopt rule certifies the pair only when this empirical gap, minus a further $2r_n$, is positive. When a genuine switch makes $\Gbar_{a,b}=+\Delta$ the worst-case empirical gap is $\Delta+B_{a,b}-2r_n$, so adoption is guaranteed once $\Delta+B_{a,b}-2r_n>2r_n$, that is $\Delta+B_{a,b}>4r_n$. When the candidate is truly best the certified lower bound never exceeds the biased limit $-\Delta+B_{a,b}$, so a false switch needs $B_{a,b}>\Delta$. The symmetric sufficient condition $|B_{a,b}|<\Delta-4r_n$, which presumes $\Delta>4r_n$, rules out both a missed genuine switch and a false switch. For IPW the same argument on the reweighted sums gives the limit $(\sum_i p_i\,\mu_{i,a}/p_i)/(\sum_i p_i/p_i)=N^{-1}\sum_i\mu_{i,a}=\bar\mu_a$, and the reweighted window's effective sample size is the Kish quantity $(\sum_{i\in\Pi}1/p_i)^2/\sum_{i\in\Pi}1/p_i^2$. \qed
\end{proof}

\paragraph{Empirical boundary and the IPW repair.} Figure~\ref{fig:struct} is a controlled diagnostic that sweeps the structural skew $s$ of a persistent preference-correlated dropout on the same drift and genuine-switch instances ($N=24$, $K=4$, $\Delta=0.10$, $30$ seeds, a two-group zero-sum tilt on the contested arm). The measured bias is the emergent difference between the subset and IPW windowed gaps, not the closed form. On the drift instance the subset-mean windowed gap falls almost exactly along the covariance prediction of Proposition~\ref{prop:struct}, from $0.090$ at $s=0$ through zero near $s=0.5$ to $-0.097$ at $s=0.9$, the measured bias matching $\mathrm{Cov}(p,\mu)/\bar p$ to within $0.02$ (for instance $0.105$ against the predicted $0.111$ at $s=0.6$), while the IPW estimator holds the gap flat at the margin across all skews. The decision tracks the boundary, the subset monitor holding a zero false-switch rate well past the guaranteed-safe point and breaking only once the bias clears the margin by the radius headroom, to $0.17$ at $s=0.9$ and $1$ at $s=1$, whereas IPW stays at zero throughout. This drift instance is the clean test of the proposition, since its false-switch boundary depends on the bias sign alone, not on the radius headroom. On the genuine-switch instance the headroom is already thin at this partial-participation operating point, where $4r_n$ is comparable to $\Delta$, so adoption is radius-limited even at $s=0$ (subset $0.50$, IPW $0.60$), consistent with the $\Delta>4r_n$ premise not being comfortably met. The structure then suppresses the new best arm, so the biased subset monitor's adopt rate collapses to zero by $s=0.15$, while IPW, which removes the bias but pays a variance inflation, keeps adopting at $0.43$ and $0.23$ for $s=0.15$ and $0.3$ before its inflated radius censors adoption. To show the decision-safety corollary is not vacuous, we rerun the same instances at a wider-margin point where $\Delta>4r_n$ holds ($\Delta=0.25$ at the same window). Genuine-switch adoption then returns to $1$ at $s=0$, and the biased subset monitor adopts in every run up to $s=0.6$ before the bias clears the boundary and adoption falls to zero by $s=0.75$, the predicted transition, while the IPW monitor adopts in every run through $s=0.9$ and falls only at $s=1$. The wider margin also holds the drift-instance false-switch rate at zero throughout, the margin now absorbing the full swept bias. Structured participation bias is therefore real but bounded by the margin, and its mean is removable by inverse-propensity weighting when the propensities are known or estimable from participation logs, at a variance cost that limits finite-sample adoption power.

\begin{figure}[t]
\centering
\includegraphics[width=0.98\linewidth]{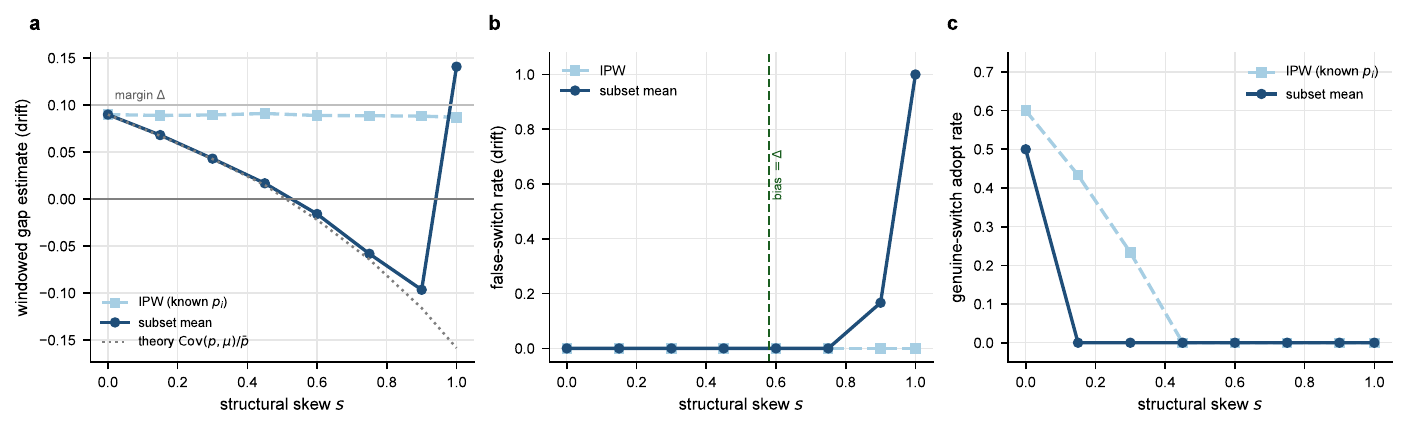}
\caption{Structured and persistent participation bias ($N=24$, $K=4$, $\Delta=0.10$, $30$ seeds), sweeping the structural skew $s$ of a persistent preference-correlated dropout. (a) On the drift instance the subset-mean windowed gap follows the covariance prediction of Proposition~\ref{prop:struct} (dotted) and crosses zero, while IPW with known propensities stays at the margin $\Delta$. The $s=1$ subset point is post-breakdown, measured against the switched candidate. (b) The subset monitor holds false-switch rate $0$ through the guaranteed-safe boundary (vertical line, bias $=\Delta$) and breaks only past it, once the bias clears the margin by the radius headroom, while IPW stays at $0$. (c) On a genuine-switch instance IPW preserves adoption to lower skew than the biased subset mean. Reproduced by the structured-participation sweep.}
\label{fig:struct}
\end{figure}

\subsection{I.8 Scaling to large action spaces}
\label{app:largeK}

A referee asked for the cost of block-based certification at large $K$, where each DRFC scan pulls $Kh$ samples per agent per block. We sweep the arm count $K\in\{4,8,16,32\}$ at $N=12$ agents, with two decision switches ($\Stdec=2$), moderate decision-irrelevant local drift ($\Stloc=26$), decision gap $\Delta=0.20$, and edge probability $0.5$. The monitor period is held fixed at $M=1500$ so that a full $K$-arm scan, which costs $Kh$ pulls per agent and ranges from $32$ to $256$ at $h=8$, completes within one period even at $K=32$, using the disclosed practical radius constant $C=1.15$. Each cell averages $12$ seeds over horizon $15000$ with full source-record reach.

Figure~\ref{fig:largeK} shows two facts. First, DRFC regret tracks the centralized oracle at every arm count, from $4336$ against the oracle $4218$ at $K=4$ to $9896$ against $9605$ at $K=32$, a graceful growth of about $2.3$ times across an eightfold larger action space, and it stays well below every reactive baseline, the margin widening with $K$, so that at $K=16$ DRFC reaches $6612$ against CUSUM-Reset $31171$ and SW-UCB-Dec $31482$. This growth is the governed $Kh$ certification cost the block scan pays for a larger action space, not a failure mode. Second, DRFC per-agent source-record traffic stays nearly flat in $K$, from $0.145$ to $0.193$ records per agent per round, because source-traceable gossip deduplicates by source and there are only $N$ sources regardless of $K$. The index-sharing SW-UCB-Dec instead broadcasts a $K$-vector each cycle, so its per-agent traffic grows linearly in $K$, through $0.333$, $0.667$, $1.333$, and $2.667$. The realized DRFC traffic sits far below the per-agent $\widetilde O(K(h+\DG)/M)$ worst case of Section~\ref{app:message-cost}, since steady-state source dedup leaves it $N$-dominated rather than $K$-dominated, so large-action-space communication is governed by the number of sources and the monitor period rather than by the action-space width.

\paragraph{Subsampled blocks for very large $K$.} The $Kh$ per-block pull cost is the one quantity that grows with the action space. A subsampled active-set block that samples a uniform size-$r$ subset of the active arms each block replaces $Kh$ by $rh$ per block at the price of about $K/r$ more blocks to cover every arm, leaving the elimination guarantee intact because each arm is still certified once it has been covered enough times. This keeps the per-block budget bounded for very large $K$ while the decision-level certification is unchanged. A companion agent-count sweep confirms that the per-agent traffic grows with the network only through the flooding rounds, the $\DG$ dependence of Section~\ref{app:message-cost}, with the per-agent regret order-constant.

\begin{figure}[t]
\centering
\includegraphics[width=0.98\linewidth]{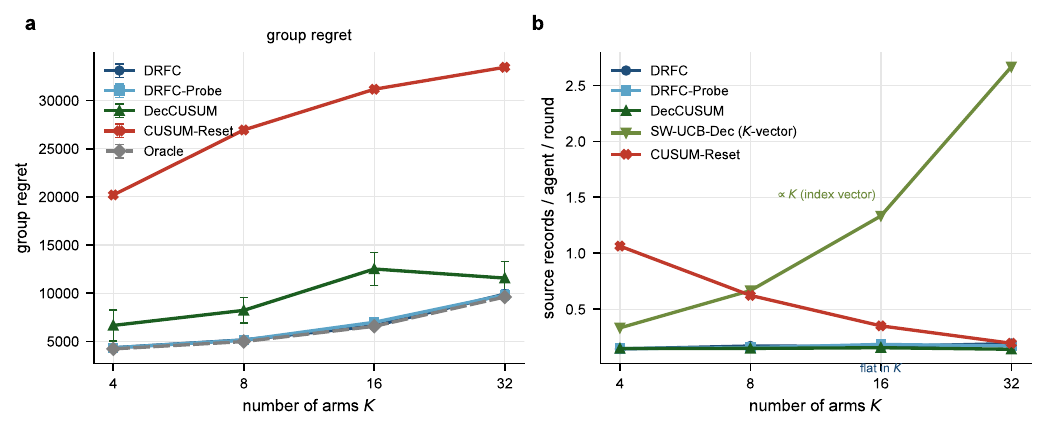}
\caption{Scaling to large action spaces ($N=12$, $\Stdec=2$, $\Delta=0.20$, $12$ seeds, monitor period fixed at $M=1500$). (a) Group regret against the arm count $K$. DRFC tracks the centralized oracle at every $K$ and stays far below the reactive and decision-level change detectors, the margin widening with $K$. (b) Per-agent source-record traffic against $K$. DRFC stays flat near $0.17$ because source-traceable gossip deduplicates by source, while the index-sharing SW-UCB-Dec broadcasts a $K$-vector and grows linearly in $K$. Reproduced by the large-action-space sweep.}
\label{fig:largeK}
\end{figure}

\section{J. Proof of the Personalized Benchmark Theorem}
\label{app:pers}

This section proves the personalized-benchmark theorem for DRFC-Pers, a supplement-only result distinct from the four numbered main-paper theorems. The main paper refers to the personalized benchmark only in its limitations.
The argument is a direct per-cluster instantiation of the global proof
(Sections~A--B of this supplement) plus a union bound over clusters; it does
not introduce new concentration machinery.

\paragraph{Setup.}
Recall the partition $\{I_c\}_{c=1}^{G}$ with $|I_c|$ agents in cluster $c$, the
cluster-personalized best arm
$a_t^{\star,c}=\arg\max_a |I_c|^{-1}\sum_{i\in I_c}\mu_{i,a,t}$, and the
personalized regret
$R_T^{\mathrm{pers}}=\sum_c\sum_{i\in I_c}\sum_t(\mu_{i,a_t^{\star,c},t}
-\mu_{i,A_{i,t},t})$.
DRFC-Pers runs $G$ logically independent active-set scans, one per cluster;
cluster $c$'s scan uses only the agents in $I_c$ and the cluster-mean
$\bar\mu_a^{(c),t}$ as the elimination object.
Define per-cluster decision-switch times $\rho_r^{(c)}$ and per-cluster gap
$\Delta_a^{(r,c)}=\inf_{t\in[\rho_r^{(c)},\rho_{r+1}^{(c)})}
(\bar\mu_{a_r^{\star,c},t}-\bar\mu_{a,t})$ in direct analogy to the main sign-stable decision-regime definition.

\paragraph{Per-cluster sign preservation.}
The block-level estimator inside cluster $c$ is
$\hat G^{(c)}_{a,b}(\mathcal{W})=|I_c|^{-1}\sum_{i\in I_c}\hat g_{i,a,b}(\mathcal{W})$,
a direct restriction of the global object to $I_c$.
Equation~(\ref{lem:app-concentration})-style decomposition (reward noise +
schedule randomization) goes through with $N$ replaced by $|I_c|$ throughout,
so the radius is
$\beta^{(c)}(\mathcal{W})=C\sqrt{\log(K^2T^3/\delta_c)/(|I_c|\,h|\mathcal{W}|)}$
with $\delta_c=\delta/(3G)$ (the budget per cluster from the union bound below).
The sign-preservation lemma (Lemma~\ref{lem:app-sign}) and the
pairwise-correctness lemma (Lemma~\ref{lem:app-pairwise}) hold verbatim with
$\Delta_a^{(r)}$ replaced by $\Delta_a^{(r,c)}$ and the global mean replaced
by the cluster mean.

\begin{lemma}[Per-cluster sign preservation]
\label{lem:app-pers-sign}
If $\mathcal{W}$ is contained in cluster $c$'s decision regime $r$, then for
every $b\neq a_r^{\star,c}$,
$\bar G^{(c)}_{a_r^{\star,c},b}(\mathcal{W})\geq\Delta_b^{(r,c)}$.
\end{lemma}

\begin{proof}
Fix $\mathcal{W}\subseteq[\rho_r^{(c)},\rho_{r+1}^{(c)})$ and $b\neq a_r^{\star,c}$.
\[
\begin{aligned}
\bar G^{(c)}_{a_r^{\star,c},b}(\mathcal{W})
={}& \frac1B\sum_{\ell\in\mathcal{W}}\frac1{q_\ell}\sum_{t\in\mathcal{U}_\ell}\bigl(\bar\mu_{a_r^{\star,c},t}-\bar\mu_{b,t}\bigr)\\
\overset{\text{①}}{\ge}{}& \frac1B\sum_{\ell\in\mathcal{W}}\frac1{q_\ell}\sum_{t\in\mathcal{U}_\ell}\Delta_b^{(r,c)}\\
={}& \Delta_b^{(r,c)},
\end{aligned}
\]
where ① applies the cluster-mean margin $\bar\mu_{a_r^{\star,c},t}-\bar\mu_{b,t}\ge\Delta_b^{(r,c)}$ at every $t\in[\rho_r^{(c)},\rho_{r+1}^{(c)})$ and the closing equality uses the convex weights $\tfrac1B\cdot\tfrac1{q_\ell}\ge0$ summing to $1$.
\end{proof}

\paragraph{Per-cluster active-set scan certification.}
By the same argument as Lemma~\ref{lem:app-active-scan} but with $N$ replaced
by $|I_c|$, the per-cluster scan in regime $r$ certifies $a_r^{\star,c}$ once
its radius is below $\Delta_a^{(r,c)}/4$, requiring at most
$B_a^{(r,c)}=O(\log(K^2T^3/\delta_c)/(|I_c|h(\Delta_a^{(r,c)})^2))$
blocks per active arm.
The one-scan group regret is
$\Ccert_r^{(c)}=\sum_{a\neq a_r^{\star,c}}\Gamma_a^{(r,c)}/(\Delta_a^{(r,c)})^2$.

\paragraph{Per-cluster dynamic regret.}
Applying the five-term decomposition of Theorem~\ref{thm:app-regret} per
cluster gives
\begin{align*}
R_T^{(c)}
&\leq
\sum_{r=0}^{\Stdec^{(c)}}(1+\tfrac{\Lreg_r^{(c)}}{M})\Ccert_r^{(c)} \\
&\quad
+|I_c|\Stdec^{(c)}\bigl(M+\Hscan+\DG(\delta_c/3)\bigr).
\end{align*}
The boundary-crossing scan accounting and synchronization term carry over
verbatim with $N\to |I_c|$.

\paragraph{Union bound across clusters.}
Reserve the concentration budget $\delta/3$ for the per-cluster concentration
events, the flooding budget $\delta/3$ for the propagation events, and the
synchronization budget $\delta/3$ for the per-cluster epoch-tie-breaking
events---exactly as in Lemma~\ref{lem:uniform-good-event}.
Each cluster consumes $1/G$ of each budget; the per-cluster effective
confidence is $\delta_c=\delta/(3G)$.
The $\log(K^2T^3/\delta_c)$ factor inside $\Ccert_r^{(c)}$ absorbs the extra
$\log G$ inside $\widetilde O$.

Summing the per-cluster regrets and applying linearity of expectation gives
\[
R_T^{\mathrm{pers}}
=
\widetilde O\!\left(\sum_{c=1}^{G}
\left[\sum_{r=0}^{\Stdec^{(c)}}\!\!(1+\tfrac{\Lreg_r^{(c)}}{M})\Ccert_r^{(c)}
+|I_c|\Stdec^{(c)}\bigl(M+\Hscan+\DG(\delta/(3G))\bigr)\right]\right),
\]
which is the per-cluster personalized-benchmark bound for DRFC-Pers.
The high-probability statement holds with probability at least $1-\delta$ by
the union bound.
\qed

\paragraph{Per-cluster separation.}
The main class-relative separation theorem also lifts per-cluster:
on an instance family where cluster $c$ has $\Stloc^{(c)}=\Theta(T/L)$ but
$\Stdec^{(c)}=0$, any $(\alpha,\gamma,\eta,\tau)$-reactive protocol triggered
by per-cluster local jumps incurs $\Omega((1-\eta)\Stloc^{(c)}c_{\mathrm{reset}})$
reset regret within cluster $c$, while DRFC-Pers pays no $\Stloc^{(c)}$
adaptation term.
The construction reuses the sign-flip family of supplementary Section~E
restricted to a single cluster.

\paragraph{When the bound is tight.}
For each cluster, the per-cluster bound is tight up to the same $\log$ factors
as the global bound (matched by the homogeneous-family certification lower
bound applied per cluster).
The $G$-fold parallelism cost shows up as a factor $G$ in the additive
$\log\log G$ inside $\widetilde O$ from the union bound, but no multiplicative
$G$ on the leading $\Ccert_r^{(c)}$ term.

\section{K. Reproducibility Checklist}
\label{app:repro}

All experiments are implemented in Python (standard library only for the
simulators; the figure layer uses \texttt{matplotlib} and \texttt{numpy}). Each
run is deterministic given its recorded seed and frozen experiment
configuration, and writes tabular outputs consumed by the figure-generation
routines. The anonymized code-and-data archive submitted as supplementary
material contains the reproduction instructions, per-seed logs, aggregate tables, and
significance outputs needed to reproduce every figure, table, and statistical
test; the archive is released publicly upon publication.

\paragraph{Environment.} CPython 3.13 with \texttt{numpy} 2.2 and
\texttt{matplotlib} 3.10 for the figure layer; the simulators use the standard
library only. Experiments run on a single commodity x86-64 CPU core under
Windows 11, with under 2\,GB RAM and no GPU or external services. The
MovieLens-1M table used by the real-data drift experiment ($\sim$6\,MB) is
auto-downloaded and cached on first use; the cache is not included in the
supplementary archive.

\paragraph{Mapping experiments to results.} The reproducibility package maps
each reported artifact to a documented experiment entry:
\begin{itemize}
\item \emph{Within-regime drift phase transition} (main DRFC-Seq theorem, synthetic
square-wave instance): simulator outputs, aggregate tables, and rendered
figures for the synthetic drift panel.
\item \emph{Real-data within-regime drift} (MovieLens-1M, Drama-excluded genre
set with $8/16$ sign-changing bins): the corresponding replay outputs and
rendered real-data drift figures.
\item Decision-switch, ablation, gossip, graph-probability, user-cluster, and
doubling sweeps: one documented experiment entry per sweep, each paired with its
aggregate table and rendered figure.
\item Headline significance tests of Table~\ref{tab:significance}, the one-sided
Mann-Whitney $U$ on decision regret and Fisher exact on the per-seed
false-switch event with Holm correction: the analysis reuses the exact figure
pipeline and re-aggregates its per-seed values against the archived aggregate
tables as a self-check.
\end{itemize}

\paragraph{Statistics.} Every reported number is a mean over independent runs
with distinct seeds: $20$ runs for the K-ary, MovieLens-replay, user-cluster, and
doubling sweeps (the frozen configuration default), and $30$ runs for the
within-regime-drift sweeps. Shaded bands in figures are 95\% confidence
intervals of the mean ($1.96\,\sigma/\sqrt{\text{runs}}$). Table $\pm$ ranges are
one standard deviation across runs. The false-switch rate is the per-seed event rate, the fraction of
seeds in which the monitor adopts a non-time-average-best arm at least once inside a no-decision-switch
regime. This is the quantity the per-seed Fisher test below operates on, and it is the definition used
in every false-switch table and figure caption.

\paragraph{Statistical significance of the headline separation.}
The headline drift separation is backed by a significance analysis over the same
$30$ seeds, not a visual gap. Each method draws an independent policy seed per
run, so the $30$ per-method values form an independent sample, and we test the
two headline quantities directly. Decision regret is compared with a one-sided
Mann-Whitney $U$ test and the per-seed false-switch event with Fisher's exact
test, both against each adopt-capable decision-level detector. Holm-Bonferroni
controls the family-wise error across all comparisons within each dataset, and a
$10{,}000$-resample bootstrap interval for the median regret difference
accompanies every regret test. A self-check re-aggregates the per-seed values
used here and matches the archived aggregate tables to within $10^{-3}$, so these are
provably the same runs that underlie the synthetic and MovieLens drift
phase-transition figures. Table~\ref{tab:significance} reports, for each
baseline, the range over the post-margin amplitudes together with the most
conservative Holm-adjusted $p$-value in that range.

On the synthetic instance every separation is significant. Past the margin
DRFC-Seq attains $1.4$ to $2.2$ times lower median decision regret than each
detector at Holm-adjusted $p$ below $3\times10^{-10}$, with a zero false-switch
rate against rates of $0.17$ to $0.83$ at Holm-adjusted Fisher $p$ below $0.03$,
and the bootstrap median-difference interval excludes zero in every case. On real
MovieLens drift the false-switch separation is the strong result, with DRFC-Seq
holding rate zero while each detector climbs toward one at Holm-adjusted $p$ down
to $2\times10^{-16}$. The regret picture is reported faithfully rather than
overclaimed. DRFC-Seq has significantly lower regret than the two restart
detectors DL-GLR-klUCB and DL-ADR-bandit, while DecCUSUM reaches a marginally
lower median regret by harvesting decision-irrelevant reward through false
switches, so the one-sided regret test against DecCUSUM is not significant on real
data. This is the honest content of the real-data panel, where the decisive
quantity is the false-switch rate.

\begin{table}[t]
\centering
\small
\setlength{\tabcolsep}{4pt}
\begin{tabular}{llcccc}
\toprule
Data & Baseline & regret ratio & MWU Holm $p$ & FSR (DRFC vs base) & Fisher Holm $p$ \\
\midrule
\multicolumn{6}{l}{Synthetic square-wave drift, $b\in\{0.15,0.20,0.25\}$ (post-margin)}\\
& DecCUSUM      & $1.41$--$2.01$ & $\le 2.6\times10^{-10}$ & $0$ vs $0.17$--$0.83$ & $\le 2.6\times10^{-2}$\\
& DL-GLR-klUCB  & $1.76$--$2.19$ & $\le 2.6\times10^{-10}$ & $0$ vs $0.57$--$0.73$ & $\le 1.5\times10^{-6}$\\
& DL-ADR-bandit & $1.60$--$1.86$ & $\le 2.6\times10^{-10}$ & $0$ vs $0.43$--$0.53$ & $\le 4.6\times10^{-5}$\\
\midrule
\multicolumn{6}{l}{Real MovieLens drift, $g\in\{1.5,2.0,2.5,3.0\}$ (post-onset)}\\
& DecCUSUM      & $0.70$--$0.95$ & n.s.\ & $0$ vs $0.17$--$1.00$ & $\le 1.3\times10^{-1}$\\
& DL-GLR-klUCB  & $1.09$--$1.42$ & $\le 3.1\times10^{-8}$ & $0$ vs $0.40$--$1.00$ & $\le 3.7\times10^{-4}$\\
& DL-ADR-bandit & $1.14$--$1.38$ & $\le 2.6\times10^{-10}$ & $0$ vs $0.57$--$1.00$ & $\le 2.2\times10^{-6}$\\
\bottomrule
\end{tabular}
\caption{Significance of the headline drift separation, $30$ seeds per cell,
one-sided Mann-Whitney $U$ on decision regret and Fisher's exact test on the
per-seed false-switch event, with Holm-Bonferroni family-wise correction within
each dataset. Regret ratio is median base regret over median DRFC-Seq regret, so a
value above one means DRFC-Seq is lower. Each cell reports the range over the
swept amplitudes and the most conservative Holm-adjusted $p$-value in that range.
On synthetic drift every separation is significant. On real MovieLens drift the
false-switch separation is significant against all three detectors at $g\ge2.0$
(down to $p\le2\times10^{-16}$) and marginal only against DecCUSUM at the onset
gain $g=1.5$. DecCUSUM reaches marginally lower real-data regret (ratio below one)
by harvesting decision-irrelevant reward through false switches, so its regret
comparison is not significant while its false-switch rate climbs to one.
Re-aggregated per-seed values match the archived aggregate tables.}
\label{tab:significance}
\end{table}

\paragraph{Hyperparameters.} All hyperparameters are the frozen defaults in the
recorded experiment configurations; runtime overrides are limited to the number
of runs, horizon, and seed. The drift experiments fix the exploitable
decision margin $\Delta=0.10$ and sweep the drift gain over
$\{0.0,0.5,1.0,1.5,2.0,2.5,3.0\}$ with sliding-window confidence level
$\alpha=0.10$.

\paragraph{DRFC-Seq live operating point (non-vacuity).} For the
within-regime-drift experiments DRFC-Seq runs with $N=24$ agents and a
period-spanning sliding window of $120$ probe-blocks (each $6$ balanced
samples/arm, probe gap $20$), so the per-arm window count is $n=720$ and the
stitched anytime radius is $2r_n\approx0.085$, strictly \emph{below} the headline
decision margin $\Delta=0.10$. This is deliberate: a monitor whose radius
exceeds the margin can never trigger, so a zero false-switch rate would be
vacuous. At this live point the monitor \emph{is} capable of switching, which we
verify with an adaptivity battery (genuine $\Stdecavg{=}1$ switch at the horizon
midpoint, with and without concurrent drift): DRFC-Seq adopts a genuine switch
in \emph{every} run at false-switch rate $0$, at both the headline margin
$\Delta=0.10$ and a separated margin $\Delta=0.20$ and with or without concurrent
drift, and never adopts under a drift-only instance ($\Stdecavg{=}0$, $b=0.20$).
The adaptivity battery uses horizon $24000$, so the $\sim5000$--$6800$-round
detection latency is measured rather than censored, and is summarized in the
adaptivity table of Section~I. The window
($\approx 6240$ rounds, $120$ probe-blocks of $52$ rounds) spans $\approx 12.5$
drift periods (period $500$ rounds). The MovieLens-1M clipping
($[0.05,0.95]$) is verified to preserve the time-average-best arm with strictly
positive margin at every gain by an archived preprocessing check.

\paragraph{Robustness to window and confidence misspecification.} The
period-spanning window is not a hand-tuned requirement that presumes oracle
knowledge of the drift period. Holding the true period fixed at $P=500$ rounds and
the post-margin amplitude at $b=0.20$, we sweep the sliding-window length over a
$16\times$ range, from $1.0$ to $16.6$ drift periods, and the confidence level
$\alpha$ over a $20\times$ range (Table~\ref{tab:window-sens}). DRFC-Seq holds a zero false-switch
rate at every setting, with decision regret flat to within $0.1\%$. At the fixed
agent count $N=24$ part of this safety is conservative inertia. Shortening the
window lowers the per-arm count $n=Wh$, which inflates the anytime radius
$r_n\propto1/\sqrt{n}$, so the radius is live ($2r_n<\Delta$) only for windows
above about nine periods and a shorter window is also too conservative to trigger.

To separate genuine drift-robustness from this inertia we repeat the sweep holding
the radius live and constant ($2r_n\approx0.085<\Delta$) by scaling the agent
count inversely with the window, $N\approx2880/W$ (Table~\ref{tab:window-live}). The false-switch rate stays
zero across windows from $1.25$ to $12.5$ periods even at this live radius, so the
drift-robustness is genuine time-average tracking and not the conservative inertia
of a wide radius. All these windows sit at or above one drift period, where the
windowing bias $\varepsilon_W=bP/(2|\mathcal{W}|)$ stays within the margin and
oscillates in sign as the window slides, absorbed by the time-uniform radius.
Holding the radius live at shorter windows by adding agents is not free on the
switch-adaptive side of this agent-scaled sweep. The genuine-switch adopt latency
grows from $5165$ rounds at $12.5$ periods to $8997$ at $9.4$ periods and exceeds
the horizon below about nine periods, but the fixed-agent control next shows that
this growth is the cost of the added agents, not of the shorter window.

To isolate the window length from the agent count we hold $N=24$ and the
communication setup fixed and instead pin the radius live by deepening each probe
block, $h\approx720/W$, so the per-arm window count $n=Wh$ and the radius stay
fixed while only the number of probe-blocks changes (Table~\ref{tab:window-fixedn}). Across the reachable live
band, from about two to $12.5$ drift periods, DRFC-Seq holds a zero false-switch rate and
adopts the genuine switch in every run, and now the adopt latency falls as the
window shortens, from $5193$ rounds at $12.5$ periods to $3073$ at about two periods. The
window length is therefore a wide and forgiving design choice rather than a
hand-picked value. A window of about two drift periods already gives the
drift-robust good corner and adopts fastest at this agent count. At a fixed agent
count the radius-liveness condition couples to the window length, since a deeper
block consumes rounds, so the in-rounds window cannot fall below about two periods
here, which is why the agent-scaled sweep above is needed to reach shorter live
windows.

\begin{table}[t]
\centering
\small
\setlength{\tabcolsep}{6pt}
\begin{tabular}{lrrr}
\toprule
Sweep & setting & FSR & regret \\
\midrule
\multicolumn{4}{l}{Window length (drift periods), $\alpha=0.10$}\\
& $1.04$ & $0.00$ & $9972$ \\
& $1.66$ & $0.00$ & $9972$ \\
& $2.18$ & $0.00$ & $9973$ \\
& $3.12$ & $0.00$ & $9974$ \\
& $4.68$ & $0.00$ & $9972$ \\
& $6.24$ & $0.00$ & $9970$ \\
& $9.36$ & $0.00$ & $9972$ \\
& $12.48$ & $0.00$ & $9971$ \\
& $16.64$ & $0.00$ & $9975$ \\
\midrule
\multicolumn{4}{l}{Confidence level $\alpha$, window $=12.5$ periods}\\
& $0.01$ & $0.00$ & $9974$ \\
& $0.05$ & $0.00$ & $9972$ \\
& $0.10$ & $0.00$ & $9973$ \\
& $0.20$ & $0.00$ & $9975$ \\
\bottomrule
\end{tabular}
\caption{DRFC-Seq under window and confidence misspecification within-regime
drift ($\Stdecavg=0$, $b=0.20>\Delta=0.10$, period $P=500$ rounds, $N=24$, $30$
seeds). The false-switch rate (fraction of seeds with at least one false switch)
stays $0$ and decision regret stays flat across a $16\times$ window range and a
$20\times$ confidence range. Reproduced by the window-sensitivity sweep.}
\label{tab:window-sens}
\end{table}

\begin{table}[t]
\centering
\small
\setlength{\tabcolsep}{5pt}
\begin{tabular}{rrrrr}
\toprule
window & $N$ & $2r_n$ & drift FSR & genuine adopt \\
(periods) & & & ($\Stdecavg{=}0$) & (latency) \\
\midrule
$1.25$ & $240$ & $0.085$ & $0.00$ & --- \\
$1.56$ & $192$ & $0.085$ & $0.00$ & --- \\
$1.87$ & $160$ & $0.085$ & $0.00$ & --- \\
$2.18$ & $138$ & $0.085$ & $0.00$ & censored \\
$3.12$ & $96$  & $0.085$ & $0.00$ & censored \\
$4.68$ & $64$  & $0.085$ & $0.00$ & censored \\
$6.24$ & $48$  & $0.085$ & $0.00$ & censored \\
$9.36$ & $32$  & $0.085$ & $0.00$ & $0.85$ ($8997$) \\
$12.48$ & $24$  & $0.085$ & $0.00$ & $1.00$ ($5165$) \\
\bottomrule
\end{tabular}
\caption{DRFC-Seq at a held-live radius. The anytime radius is pinned at
$2r_n\approx0.085<\Delta=0.10$ across all window lengths by scaling the agent
count inversely with the window ($N\approx2880/W$), so safety cannot come from a
conservatively wide radius. The drift-only false-switch rate ($\Stdecavg=0$,
$b=0.20$, $20$ seeds) stays $0$ across the full range of windows from about one to
twelve drift periods, all on the safe side of the one-period bias boundary. The genuine-switch adopt rate
(separated margin $\Delta=0.20$, latency in rounds in parentheses) shows the
trade-off of holding the radius live at short windows, where the larger agent
count slows coordination and the adopt latency grows until it exceeds the horizon
(censored). The period-spanning window adopts fastest. Reproduced by the
held-live-radius sweep.}
\label{tab:window-live}
\end{table}

\begin{table}[t]
\centering
\small
\setlength{\tabcolsep}{5pt}
\begin{tabular}{rrrrr}
\toprule
window & $h$ & drift FSR & genuine & latency \\
(periods) & & ($\Stdecavg{=}0$) & adopt & (rounds) \\
\midrule
$2.18$ & $34$ & $0.00$ & $1.00$ & $3073$ \\
$3.12$ & $24$ & $0.00$ & $1.00$ & $3243$ \\
$4.68$ & $16$ & $0.00$ & $1.00$ & $3512$ \\
$6.24$ & $12$ & $0.00$ & $1.00$ & $3879$ \\
$9.36$ & $8$  & $0.00$ & $1.00$ & $4486$ \\
$12.48$ & $6$  & $0.00$ & $1.00$ & $5193$ \\
\bottomrule
\end{tabular}
\caption{DRFC-Seq at a held-live radius with the agent count fixed ($N=24$) and
the communication setup unchanged. The radius is pinned live
($2r_n\approx0.085<\Delta=0.10$) by deepening each probe block, $h\approx720/W$,
so the per-arm window count $n=Wh\approx720$ and only the number of probe-blocks
$W$ varies. Across the reachable live band DRFC-Seq holds a zero false-switch rate
(drift-only, $\Stdecavg=0$, $b=0.20$, $20$ seeds) and adopts the genuine switch in
every run (separated margin $\Delta=0.20$), with the adopt latency falling as the
window shortens. This removes the agent-count confound of
Table~\ref{tab:window-live}, leaving window length the only varying factor.
Reproduced by the fixed-agent held-live-radius sweep.}
\label{tab:window-fixedn}
\end{table}

\subsection{Removing the drift-period knob with Doubling-Window DRFC-Seq}
\label{app:adaptive-window-exp}

The window sweeps above hold the sliding window matched to a known drift period. Doubling-Window DRFC-Seq (Theorem~\ref{thm:app-adaptive-window}) removes that knowledge. It runs a geometric ladder of windows $W\in\{8,16,32,64\}$ probe-blocks on one shared probe stream at union level $\alpha/J$, anoints the shortest window whose certificate has matched the longest drift-safe window $W_{\max}$ over a full $W_{\max}$-block horizon, and adopts on the shortest anointed window. The anointment is the empirically checkable certificate that Theorem~\ref{thm:app-adaptive-window} turns into drift-safety, so the monitor adopts only through a window proven to track the time-average best arm, while paying the per-switch latency of the shortest such window rather than of $W_{\max}$.

We stress this against an unknown drift scale. Each of $30$ seeds ($10$ per shape) draws a random drift period $P$ uniform in $[600,4000]$ rounds, a random phase, and a random waveform shape from square, triangle, and sine, all with margin $\Delta=0.10$ and amplitude $b=0.22$ (so $b>2\Delta$, the cancellation regime of Proposition~1). Each draw is run in two cells, a drift-only cell ($\Stdecavg=0$, measuring the false-switch rate) and a drift-plus-genuine-switch cell ($\Stdecavg=1$ with the same live drift, the average-best arm flipping at the horizon midpoint, measuring the adopt rate and detection latency). Four monitors are compared on identical instances. \textbf{DW-Adaptive} is the doubling-window policy with no knowledge of $P$. \textbf{Oracle-$W^\star$} is single-window DRFC-Seq told $P$ and using the shortest grid window with a constant-fraction drift-safety margin, $L(W^\star)\ge 2P$; this cover factor $2$ is close to the $b/\Delta=2.2$ scale at which Theorem~\ref{thm:app-adaptive-window} places its constant-fraction-margin $W^\star$, though the experiment does not tie the two exactly. This is the unfair upper reference. \textbf{Fixed-Short} and \textbf{Fixed-Long} pin the single window to the shortest and the longest grid entry. The probe-block span is $212$ rounds, so the grid windows span $1696$ to $13568$ rounds and the period range exercises a drift-safe window $W^\star$ that ranges across the whole ladder as $P$ varies. The drawn periods are arbitrary reals, not multiples of the block span, so the whole-block period convention of Theorem~\ref{thm:app-adaptive-window} does not hold exactly here. Remark~\ref{rem:off-grid} covers exactly this case, replacing the convention with an explicit off-grid slack $2b/W$ that the fine block grid keeps well below the confidence radius, and the zero false-switch rate below confirms that analysis rather than relying on the alignment.

Table~\ref{tab:adaptive-window} and Figure~\ref{fig:adaptive-window} report the result. DW-Adaptive holds a zero false-switch rate in both cells and adopts every genuine switch, matching Oracle-$W^\star$, and its detection latency tracks the oracle that is told the period at a mean ratio of $1.36$ and a median ratio of $1.19$, well inside the doubling-factor-$2$ slack of Theorem~\ref{thm:app-adaptive-window}, using no period input. The mean is the looser of the two because a few seeds anoint a conservatively long window, while the median ratio $1.19$ sits near the no-overhead floor the theorem predicts, the operating window matching the oracle's up to the $(1+o(1))$ union inflation of the radius. DW-Adaptive is $27$ percent faster in mean latency than the drift-safe Fixed-Long window, which must use $W_{\max}$ at every period. Fixed-Short is the only monitor that violates safety, false-switching on $1$ of the $30$ drift-plus-switch seeds, the one with the longest drift period $P=3889$, because its window is below the period at the long end of the range. The pattern is the Pareto statement of the theorem. No single fixed window is both safe and fast across the randomized range, Fixed-Long is safe but slow and Fixed-Short is fast but drift-unsafe, while DW-Adaptive recovers the period-oracle's operating point online and across all three waveform shapes.

\begin{table}[t]
\centering
\small
\setlength{\tabcolsep}{4.0pt}
\begin{tabular}{lccccc}
\toprule
Monitor & knows $P$ & drift FSR & adopt & latency & switch FSR \\
\midrule
DW-Adaptive   & no  & $0.00$ & $1.00$ & $8031$  & $0.00$ \\
Oracle-$W^\star$ & yes & $0.00$ & $1.00$ & $5894$  & $0.00$ \\
Fixed-Short   & no  & $0.00$ & $1.00$ & $2714$  & $0.03$ \\
Fixed-Long    & no  & $0.00$ & $1.00$ & $11004$ & $0.00$ \\
\bottomrule
\end{tabular}
\caption{Period-agnostic window robustness (Doubling-Window DRFC-Seq, $30$ seeds, randomized period $P\in[600,4000]$, randomized phase, square/triangle/sine drift, $\Delta=0.10$, $b=0.22$, $N=24$). Columns report the drift-only false-switch rate, the genuine-switch adopt rate, the detection latency in rounds, and the false-switch rate in the drift-plus-switch cell. DW-Adaptive matches the period-oracle on false switches and adopt rate and tracks its latency without knowing $P$, is $27$ percent faster than the drift-safe Fixed-Long window, and unlike Fixed-Short never violates safety. Reproduced by the period-agnostic window sweep.}
\label{tab:adaptive-window}
\end{table}

\begin{figure*}[t]
\centering
\includegraphics[width=0.92\textwidth]{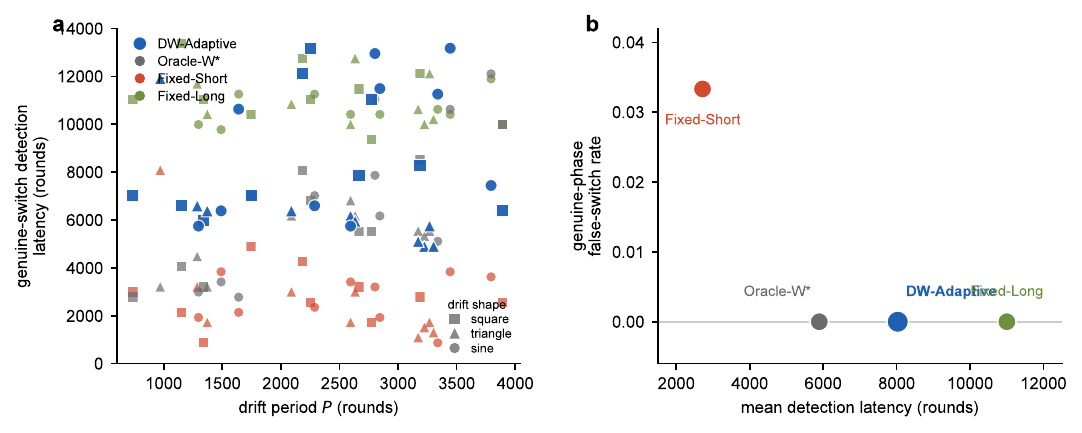}
\caption{Doubling-Window DRFC-Seq is period-agnostic. (a) Per-seed genuine-switch detection latency against the realized drift period $P$, across $30$ seeds and the square, triangle, and sine drift shapes. DW-Adaptive (blue) sits just above the period-oracle Oracle-$W^\star$ (gray) and well below the drift-safe Fixed-Long window (green) at every period, while Fixed-Short (red) is fastest but is the one monitor that false-switches. (b) The false-switch versus latency Pareto frontier, one point per monitor. DW-Adaptive, Oracle-$W^\star$, and Fixed-Long all hold a zero genuine-phase false-switch rate, and DW-Adaptive reaches that safe frontier far to the left of Fixed-Long, while Fixed-Short pays a nonzero false-switch rate for its low latency. The proposed monitor recovers the period-oracle's operating point without the period as input.}
\label{fig:adaptive-window}
\end{figure*}

\subsection{Robustness to a changing drift period}
\label{app:changing-period}

The randomized sweep above keeps the drift period fixed within each run and only unknown across runs. A natural follow-up question is a period that itself changes during a run. We switch the drift half-period once at the horizon midpoint, so every run contains two distinct periods, and we test both directions, a grow case where the full period rises from $700$ to $2800$ rounds and a shrink case where it falls from $2800$ to $700$, with the genuine decision switch placed at $0.7$ of the horizon inside the post-change phase. The waveform phase is carried continuously across the change, so the time-average gap stays $\Delta=0.10$ on both sides and the time-average best arm is unchanged. The four monitors match Table~\ref{tab:adaptive-window}, except that Oracle-$W^\star$ is now told the post-change period and Fixed-Early pins the window matched to the pre-change period.

Table~\ref{tab:changing-period} reports the outcome over $10$ seeds per case. DW-Adaptive holds a zero drift false-switch rate and adopts every genuine switch in both directions, and its mean detection latency tracks the post-change oracle at a ratio of $0.91$ on the grow case and $0.96$ on the shrink case, recovering the post-change oracle operating point online with no period input. This regime is a robustness check rather than a separation. Both fixed-window baselines also stay drift-safe here, and Fixed-Early is in fact faster, because the period change stays inside the anytime radius tolerance of the shortest window, so no fixed window is forced to false-switch. The separating regime, where a too-short fixed window does false-switch, is the randomized sweep of Table~\ref{tab:adaptive-window} in which Fixed-Short is the single monitor that breaks. The narrow and specific point here is that the doubling ladder recovers the post-change oracle's safety and latency without being told either period, not that it dominates a fixed window in this easy setting.

\begin{table}[t]
\centering
\small
\setlength{\tabcolsep}{4.0pt}
\begin{tabular}{llccccc}
\toprule
Case & Monitor & knows $P$ & drift FSR & adopt & latency & ratio \\
\midrule
grow   & DW-Adaptive     & no  & $0.00$ & $1.00$ & $6278$ & $0.91$ \\
grow   & Oracle-$W^\star$ & yes & $0.00$ & $1.00$ & $6872$ & $1.00$ \\
grow   & Fixed-Early     & no  & $0.00$ & $1.00$ & $2229$ & $0.32$ \\
grow   & Fixed-Long      & no  & $0.00$ & $1.00$ & $7148$ & $1.04$ \\
\midrule
shrink & DW-Adaptive     & no  & $0.00$ & $1.00$ & $7173$ & $0.96$ \\
shrink & Oracle-$W^\star$ & yes & $0.00$ & $1.00$ & $7508$ & $1.00$ \\
shrink & Fixed-Early     & no  & $0.00$ & $1.00$ & $6490$ & $0.86$ \\
shrink & Fixed-Long      & no  & $0.00$ & $1.00$ & $6618$ & $0.88$ \\
\bottomrule
\end{tabular}
\caption{Robustness to a mid-run drift-period change (Doubling-Window DRFC-Seq, $10$ seeds per case, $N=24$, $K=4$, $\Delta=0.10$, $b=0.22$, square, triangle, and sine drift). The drift half-period switches once at the horizon midpoint, the grow case from full period $700$ to $2800$ rounds and the shrink case from $2800$ to $700$, with the genuine switch in the post-change phase. Columns report whether the monitor is told the post-change period, the drift-only false-switch rate, the genuine-switch adopt rate, the detection latency in rounds, and the latency ratio to the post-change oracle. All monitors stay drift-safe in this regime, so it is a robustness check rather than a separation, and Fixed-Early is in fact faster. The point is that DW-Adaptive recovers the post-change oracle's safety and latency without knowing either period, while the separating regime where a too-short fixed window breaks is Table~\ref{tab:adaptive-window}. Reproduced by the changing-period sweep.}
\label{tab:changing-period}
\end{table}

\begin{table}[t]
\centering
\small
\setlength{\tabcolsep}{3.0pt}
\begin{tabular}{llrrrr}
\toprule
Regime & $C$ & adopt & false & delay & regret \\
\midrule
default & $1.15$ practical & $0.90$ & $0.03$ & $210$ & $904$ \\
default & $2.00$ & $0.70$ & $0.00$ & $486$ & $927$ \\
default & $3.19$ per-pair & $0.00$ & $0.00$ & $1173$ & $1000$ \\
default & $5.00$ & $0.00$ & $0.00$ & $1200$ & $977$ \\
default & $7.26$ worst-case & $0.00$ & $0.00$ & $1200$ & $977$ \\
\midrule
calibrated & $1.15$ practical & $1.00$ & $0.00$ & $224$ & $751$ \\
calibrated & $3.19$ per-pair & $1.00$ & $0.00$ & $749$ & $2532$ \\
calibrated & $8.41$ worst-case & $0.90$ & $0.00$ & $4855$ & $12076$ \\
\bottomrule
\end{tabular}
\caption{Confidence-scale sweep for the conservative certifier ($30$ seeds,
$\Stdec=2$, $\Delta=0.10$). The default operating point is the reported K-ary
stress configuration (horizon $1200$, monitor period $120$, regime length
$400$). The theorem-calibrated regime stretches the horizon to $24000$ with
monitor period $4000$, so each $8000$-round regime is long enough for the
worst-case constant to finish its elimination scans. Columns report the
fraction of runs adopting every decision switch, the false-switch rate, mean
detection delay in rounds (censored at the horizon when no switch is ever
certified), and mean group regret. Raising $C$ never creates false switches,
and in the calibrated regime even the worst-case constant adopts every switch
in $90$ percent of runs (at least one switch in every run) at false-switch
rate $0$, at a certification cost that grows as $C^2/\Delta^2$ in line with
the $\Ccert_r$ term of the main DRFC theorem.}
\label{tab:confidence-sweep}
\end{table}

\paragraph{Confidence-scale sweep and theorem-calibrated regime.} The
confidence-radius constant $C=1.15$ in the K-ary and user-cluster simulators
is the result of practical tuning, not the worst-case theorem-calibrated
value. Table~\ref{tab:confidence-sweep} quantifies the gap by sweeping the
constant for the conservative certifier across the practical value, the
per-pair union-bound value $\sqrt{2\log(K^2/\delta)}\approx3.19$, and the
worst-case appendix value $\sqrt{2\log(K^2T^3/\delta)}$, which is
$\approx7.26$ at horizon $1200$ and $\approx8.41$ at horizon $24000$. Three
facts emerge. First, raising $C$ errs only in the conservative direction. The
false-switch rate never exceeds the practical point's $0.03$ and is $0$ at
every $C\geq2$. Second, at the default operating point the radius after one
comparison block is $\approx0.43$ for the worst-case constant against a
decision margin of $0.10$, so elimination cannot finish inside the $400$-round
regime and certification stalls, which is why the reported experiments use
the practical value. Third, in a theorem-calibrated regime whose regime length
respects the $c_0K\log T/(N\Delta^2)$ floor of the main information-theoretic lower bound, every swept
constant becomes operational and certifies at zero false-switch rate, with
certification cost scaling as $C^2/\Delta^2$ exactly as in the $\Ccert_r$
term of the main DRFC theorem. DRFC-Seq's radius uses the same functional form as
the certified object, so the zero-false-switch conclusion under drift is
scale-robust in the same conservative direction. The reported implementation uses
this radius in the finite-horizon form $\log(K(K-1)T/\alpha)$ rather than the
horizon-free stitched form $\log(C_0K(K-1)s^2/\alpha)$ of the DRFC-Seq
anytime-validity proof, a practical proxy at a known run length. The two forms
differ only by the start-time peeling factor, and the sweep above shows the
zero-false-switch conclusion is preserved across radius constants from the
practical value up to the worst-case appendix value, so it does not depend on
which valid radius form is used.

\paragraph{Dataset license.} MovieLens-1M~\citep{harper2015movielens} is
distributed by GroupLens Research under the dataset's own license, which
permits research use.

\paragraph{Ethical considerations.} This work studies cooperative
decision-making in decentralized systems with heterogeneous agents. We use
MovieLens-1M as a publicly released, fully anonymized benchmark for cooperative
learning, and no human-subject experiments or personally identifying data are
involved. The decentralized framework can support privacy-preserving
deployments, since no raw observations leave an agent, though deployment in
user-facing recommender systems requires further audit for fairness across
demographic strata.

\paragraph{Additional related references.} Recent work further studies
switch-count non-stationarity, formal definitions of non-stationary bandits,
smooth or heavy-tailed change processes, constrained feedback, linear dynamic
regret, and heterogeneous or robust multi-agent bandits
\citep{abbasi2023newlook,liu2023definition,suk2024adaptive,genalti2025catoni,
li2025constrained,hu2026dynamic,xu2025heterogeneous,mirfakhar2025heterogeneous,
adams2025finite,wang2025heavytailed,hu2025robust}.

\clearpage
\bibliographystyle{plainnat}
\bibliography{refs}

\end{document}